\documentclass[twoside,11pt]{article}

\usepackage{jmlr2e}

\usepackage{cite}
\usepackage{amsmath,amssymb,amsfonts}
\usepackage{pifont}
\usepackage{microtype}

\usepackage{subfigure}
\usepackage{booktabs} 
\usepackage{hyperref}
\usepackage{float}
\usepackage{makecell}

\usepackage{bm}
\usepackage{algorithm}
\usepackage{algorithmic}
\usepackage{graphicx}
\usepackage{textcomp}
\usepackage{xcolor}
\usepackage{comment}
\usepackage{graphicx,psfrag,epsf}
\usepackage{enumerate}
\usepackage{natbib}
\usepackage{xspace}
\usepackage{mathrsfs}
\usepackage{xspace}
\usepackage{tikz}

\newcommand{\tm}[1]{\texttt{Term #1}}

\newcommand{\norm}[1]{\left\lVert#1\right\rVert}
\newcommand{\abs}[1]{\left\lvert#1\right\rvert}

\newcommand{\vecr}[1]{\text{vec}\left(#1\right)}
\newcommand{\tr}[1]{\text{Tr}\left(#1\right)}
\newcommand{\Pj}[1]{\bm{P}_{#1}}
\newcommand{\Pnormal}[2]{\mathcal{P}_{\mathcal{N}_{#1}}\left(#2\right)}
\newcommand{\Ptangent}[2]{\mathcal{P}_{\mathcal{T}_{#1}}\left(#2\right)}

\newcommand{\grof}[2]{\mathcal{GR}_{#1}\left(#2\right)}
\newcommand{\polarof}[2]{\mathcal{GR}^{\text{polar}}_{#1}\left(#2\right)}
\newcommand{\qrof}[2]{\mathcal{GR}^{\text{QR}}_{#1}\left(#2\right)}

\newlength\myindent
\newcommand\bindent{%
  \begingroup
  \setlength{\itemindent}{\myindent}
  \addtolength{\algorithmicindent}{\myindent}
}
\newcommand\eindent{\endgroup}

\newcommand{\innerp}[2]{\left\langle #1,#2 \right\rangle}
\newcommand{\yes}{\textcolor{green}{\ding{51}}}
\newcommand{\no}{\textcolor{red}{\ding{55}}}

\newcommand{\name}{\texttt{PerPCA}\xspace}
\newcommand{\dispca}{\texttt{distPCA}\xspace}

\newcommand{\revise}[1]{ #1}

\newcommand{\vecu}{\bm{u}}
\newcommand{\vecv}{\bm{v}}

\newcommand{\vecvperp}{\bm{v}^{\perp}{}}

\newcommand{\vecw}{\bm{w}}
\newcommand{\matA}{\bm{A}}
\newcommand{\matB}{\bm{B}}

\newcommand{\matF}{\bm{F}}
\newcommand{\matG}{\mathbf{G}}
\newcommand{\hmatF}{\hat{\bm{F}}}
\newcommand{\matI}{\bm{I}}

\newcommand{\matR}{\bm{R}}

\newcommand{\matU}{\bm{U}}
\newcommand{\matV}{\bm{V}}
\newcommand{\matW}{\bm{W}}
\newcommand{\matX}{\bm{X}}
\newcommand{\matY}{\bm{Y}}

\newcommand{\matUt}{\bm{U}_{\tau}}
\newcommand{\matVit}{\bm{V}_{(i),\tau}}
\newcommand{\dU}{\Delta\bm{U}}
\newcommand{\dV}{\Delta\bm{V}}
\newcommand{\dUt}{\Delta\bm{U}_{\tau}}
\newcommand{\dVit}{\Delta\bm{V}_{(i),\tau}}

\newcommand{\hmatU}{\hat{\bm{U}}}
\newcommand{\hmatV}{\hat{\bm{V}}}
\newcommand{\hmatUt}{\hat{\bm{U}}_{\tau}}
\newcommand{\hmatVit}{\hat{\bm{V}}_{(i),\tau}}
\newcommand{\matPi}{\bm{\Pi}}
\newcommand{\hmatPi}{\hat{\bm{\Pi}}}

\newcommand{\matS}{\bm{S}}
\newcommand{\matSigma}{\bm{\Sigma}}
\newcommand{\dS}{\delta\bm{S}}

\newcommand{\hr}{\hat{r}}
\newcommand{\hmu}{\hat{\mu}}
\newcommand{\hvecu}{\hat{\bm{u}}}
\newcommand{\hvecv}{\hat{\bm{v}}}

\newcommand{\wz}{\widetilde{\zeta}}
\newcommand{\wzit}{\widetilde{\zeta}_{(i),\tau}}
\newcommand{\zit}{\zeta_{(i),\tau}}
\newcommand{\zzt}{\zeta_{(0),\tau}}

\newcommand{\htheta}{\hat{\theta}}
\newcommand{\gmop}{G_{max,op}}

\newcommand{\polar}{\texttt{Polar}}

\newtheorem{assumption}{Assumption}[section]

\firstpageno{1}

\usepackage{lastpage}
\jmlrheading{25}{2024}{1-\pageref{LastPage}}{7/22; Revised
8/23}{2/24}{22-0810}{Naichen Shi and Raed Al Kontar}
\ShortHeadings{Personalized PCA}{Shi and Kontar}

\begin{document}

\title{Personalized PCA: Decoupling Shared and Unique Features}

\author{\name Naichen Shi \email naichens@umich.edu  \\
 \name Raed Al Kontar \email alkontar@umich.edu\\
\addr Department of Industrial \& Operations Engineering \\ University of Michigan\\
Ann Arbor, MI 48109-2117, USA\\}

\editor{Martin Jaggi}
\maketitle

\begin{abstract}
In this paper, we tackle a significant challenge in PCA: heterogeneity. When data are collected from different sources with heterogeneous trends while still sharing some congruency, it is critical to extract shared knowledge while retaining the unique features of each source. To this end, we propose personalized PCA  (\name), which uses mutually orthogonal global and local principal components to encode both unique and shared features. We show that, under mild conditions, both unique and shared features can be identified and recovered by a constrained optimization problem, even if the covariance matrices are immensely different. Also, we design a fully federated algorithm inspired by distributed Stiefel gradient descent to solve the problem. The algorithm introduces a new group of operations called generalized retractions to handle orthogonality constraints, and only requires global PCs to be shared across sources. We prove the linear convergence of the algorithm under suitable assumptions. Comprehensive numerical experiments highlight \name's superior performance in feature extraction and prediction from heterogeneous datasets. As a systematic approach to decouple shared and unique features from heterogeneous datasets, \name finds applications in several tasks, including video segmentation, topic extraction, and feature clustering. 
\end{abstract}


\section{Introduction} 
Principal component analysis (PCA) \citep{pearsonpca,hotellingpca} unravels data features by finding a few principal components (PCs) from high dimensional data that explain the largest portion of the variance. Due to its effective feature learning and dimension reduction capability, PCA has seen immense success across various domains, including image processing \citep{pcaimage1,pcaimage2}, time series modeling \citep{pcatimeseries1,pcatimeseries2}, bio-information \citep{pcagenetics1,pcagenetics2}, condition monitoring \citep{pcacm1,pcacm2}, and many more. 

However, since all data are equally weighted in standard PCA, an underlying assumption is that these data come from homogeneous distributions. This assumption, however, is often challenged in various scenarios, including the Internet of Things (IoT), where data do not come from a single source but a large number of distinct edge devices (or clients). The edge devices, from smartphones to connected vehicles, usually operate in different environments and conditions \citep{kontar2017nonparametric,kontar2018nonparametric}. The data collected by edge devices are also subject to changes in external conditions \citep{ioft} or user preferences \citep{surveyonpersonalization}. Thus, it is common for the datasets to contain significant heterogeneity and even conflicting trends while sharing some congruity. 

Standard PCA often does not work well when data homogeneity is not guaranteed \citep{hca,hcaoptimalweight}. Few works have endeavored to extend the PCA philosophy to incorporate data heterogeneity. For example, Heterogeneous PCA \citep{hca} considers the case where data from different sources have different noise levels. They propose a reweighting technique to alleviate noise heteroscedasticity. Such an approach is shown to be useful in identifying PCs from heteroscedastic noises. However, simply treating the discrepancy among datasets as different levels of noise might be inadequate to understand the intrinsic features within the data and insufficient to encode both unique and shared features across devices and clients. As such, personalized solutions are needed. 

To transmute the heterogeneity from a bane into a blessing, in this work, we propose personalized PCA (\name) that fits personalized features on each client in addition to common features shared by all clients. In our model, data are driven by several mutually orthogonal global (shared) and local (personalized) PCs. The global PCs model the common patterns among different datasets, while the local PCs model the idiosyncratic features of one specific dataset. Global and local PCs work together to fit the observations. Figure \ref{fig:toyexampleintro} is an illustration of using homogeneous PCA and personalized PCA to fit two heterogeneous datasets. As shown in the figure, simply pooling together all data across datasets using homogeneous PCA will fail to encode the unique features within each dataset, and the horizontal PC is a misleading one that is not representative of any source. In contrast, personalized PCA aims at decoupling unique and shared features so that heterogeneity across data sources is accounted for.



\begin{figure}[htbp]
\centering
\begin{minipage}[t]{0.48\textwidth}
\centering
\includegraphics[width=8cm]{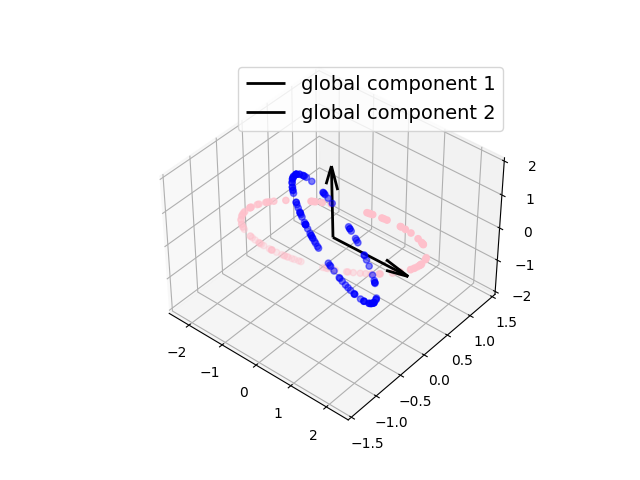}
\text{Homogeneous PCA}
\end{minipage}
\begin{minipage}[t]{0.48\textwidth}
\centering
\includegraphics[width=8cm]{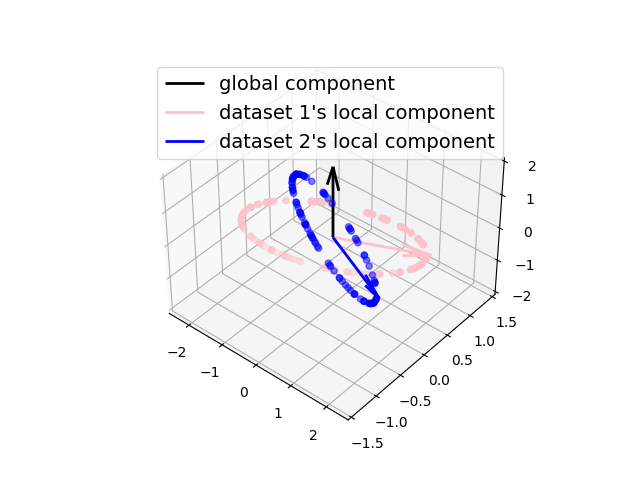}
\text{Personalized PCA}
\end{minipage}
\caption{Comparison between homogeneous PCA (standard PCA) and personalized PCA (\name). There are two datasets, one colored blue and the other pink. Dots represent the observations. Observations from one dataset are on a 2-dimensional plane. 
The black arrows represent the global PCs learned, and the colored arrows represent learned local PCs. Homogeneous PCA is a standard PCA on the pooled dataset. We will revisit the example in Section \ref{sec:experiment}.}
\label{fig:toyexampleintro}
\end{figure}

There are several benefits to personalization. Firstly, employing several local PCs to fit individual data patterns enables us to describe immensely heterogeneous trends in datasets accurately. Also, global PCs shared by all data can be estimated more precisely without being affected by disagreeing drifts from different sources. What's more, disentangling local features from global ones provides a systematic and interpretable approach to analyzing the heterogeneity structure of datasets and leveraging this knowledge for better analytics. These include: (i) \textit{Improving classification and clustering}: instead of using raw data, operating on unique features may yield better performance as differences become more explicit when removing shared features, (ii) \textit{Transforming personalized, predictive analytics}: Through selectively transferring common knowledge from one data source to another, we can reduce the negative transfer of knowledge and enhance personalized predictive and prescriptive models, (iii) \textit{Anomaly Detection}: Through monitoring changes in the unique features, we envision that anomalies can be better and faster detected.

To enable personalized PCA, we propose an optimization framework to provably recover both global and local PCs from noisy observations. The objective is to minimize the empirical reconstruction error under orthogonality constraints. The formulation stands on solid theoretical ground: We prove that, under an identifiability condition, the optimal solution can recover the true global and local PCs. 

Not only can the PCs be solved, but they can also be solved efficiently. We design an algorithm based on Stiefel manifold gradient descent that can be proved to converge linearly into the global optimum under mild conditions. The algorithm relies on a new operation called generalized retraction to handle the orthogonality constraints. It is worth noting that our algorithm is designed in a \emph{federated manner}, as the need to share raw data or place all data in a central location is circumvented, and only the updates of global PCs need to be shared across clients. Compared with centralized PCA, where all datasets are uploaded to a central server where PCA is learned on the aggregated dataset, our algorithm reaps the benefits of distributed and federated analytics. Those include communication, cost, storage, and privacy benefits \citep{ioft}. We will show the advantages of \name over existing distributed PCA methods in Section \ref{sec:relatedwork}. 


Furthermore, \name proposes a novel provable paradigm of decoupling shared and unique features. Its applications go beyond simple data dimension reduction. We show that \name has remarkable performance in video segmentation and topic extraction tasks. Hence \name opens up new possibilities for broader applications.

Moving forward, we will use client, edge device, data source, and local dataset interchangeably to represent the entities of interest. Here, entities are broadly defined, encompassing various levels of granularity. For instance, we can extract shared and unique features across dispersed datasets, output classes within a dataset, or even among observations (such as images) within a single dataset.

\subsection{Main contributions}
We summarize our contributions in the following:
\begin{itemize}
    \item \textbf{Modeling}: We propose a personalized PCA model that learns both global and local features from \textit{distributed} datasets. These features can be recovered from observations by solving a nonconvex optimization problem designed to minimize reconstruction error. 
    
    \item \textbf{Consistency}: We find that there exists a simple sufficient condition based on the ``misalignment'' of local PCs to ensure the identifiability of the global and local PCs: the maximum eigenvalue of the average of projections into local subspaces should be smaller than 1. We show that, under the identifiability condition, both global and local PCs can be estimated from noisy observations with an error that is upper bounded by $O(\frac{1}{n})$, where $n$ is the number of observations on each client. As the error decreases to $0$ when $n$ approaches infinity, the error bound essentially implies the consistency of \name. The analysis extends conventional matrix perturbation bounds \citep{matrixbook,fantope} into personalized settings where the change in one client's covariance matrix can affect the PC estimates on all clients. \revise{We also use a minimax statistical lower bound to show that the statistical error upper bound is almost tight in terms of the eigengap and misalignment parameter.}
    
    \item \textbf{Algorithm}: We design an algorithm based on Stiefel manifold gradient descent (St-GD hereon) to obtain global and local PC estimates. The major difficulty for the algorithm is handling the orthogonality constraints. To tackle it, we introduce a correction step that relies on a group of operations called \emph{generalized retractions}. A generalized retraction extends retraction in literature \citep{stiefelgeometry} as it is defined on the entire $\mathbb{R}^{d\times r}$ rather than the tangent bundle of the Stiefel manifold $St(d,r)$. In our algorithm, clients only need to share iterates of global PCs, thus preserving privacy and minimizing communication costs. 
    
    \item \textbf{Convergence}: The proposed algorithm has a local linear convergence rate. To our best knowledge, this is the first theoretical guarantee for an algorithm that simultaneously learns global and local PCs. Interestingly, the convergence is faster when local PCs are more heterogeneous, a result that lies in sharp contrast to conventional predictive federated or transfer learning \citep{transferlearning} theory as it highlights that heterogeneity can be a blessing in disguise. On the technical side, we introduce a novel Lyapunov function to study distributed St-GD with generalized retractions. 
    
    \item \textbf{Numerical results}: Empirical evidence on both synthetic and real-life datasets confirms \name's ability to decouple shared and unique features. Also, \name has exciting applications in video segmentation and topic extraction. For instance, on video segmentation tasks, \name has significant advantages over the popular \texttt{Robust PCA} \citep{robustpca} when heterogeneity patterns are not sparse.
\end{itemize}

\subsection{Organization}
The paper is organized as follows: We review related work and introduce notations in Section \ref{sec:preliminary}. In Section \ref{sec:formulation}, we propose the formulation of \name and link it with constrained optimization. Section \ref{subsec:statisticalefficiency} includes the theoretical analysis on identifiability and consistency. A federated algorithm to solve \name is developed in Section \ref{sec:algorithm}, and its convergence guarantee is established in Section \ref{sec:convergence}. Numerical experimentation results are demonstrated in Section \ref{sec:experiment}. Finally, Section \ref{sec:con} concludes the paper with a brief discussion. 
Readers mainly interested in the implementation and applications of \name can focus on Sections \ref{sec:formulation}, \ref{sec:algorithm}, and \ref{sec:experiment}. An implementation of the proposed method is in the \href{https://github.com/UMDataScienceLab/Personalized_PCA}{linked Github repository}. 

\section{Preliminaries}
\label{sec:preliminary}
In this section, we will review related work in the literature and introduce needed notations.
\subsection{Related work}
\label{sec:relatedwork}
\textbf{Structural PCA}
Structural PCA attempts to build structural models for data and noise. Research on structural PCA abounds. A seminal algorithm along this line is \texttt{Robust PCA} \citep{robustpca}. The authors point out that traditional PCA is sensitive to noise in the observations and tackle this issue by decomposing an observation matrix $\bm{Y}$ into a low-rank part $\bm{L}$ and a sparse noise part $\bm{S}$: $\bm{Y}=\bm{L}+\bm{S}$. The low-rank matrix $\bm{L}$ corresponds to the signal, and $\bm{S}$ represents the noise. It turns out that the two parts can be exactly identified under regularity conditions with carefully designed algorithms. \texttt{Robust PCA} has become a useful technique in image denoising and video processing \citep{robustpcaimage}, collaborative filtering \citep{robustpcaoutlier}, and many more. Sparse PCA \citep{sparsepca} adds sparse constraints on the PCs, encouraging each PC to depend on a minimal number of variables. While these methods are powerful in handling large noise or high dimensional data, they mainly analyze homogeneous data. 

Several algorithms have also been invented to leverage variance heterogeneity in different samples. Heterogeneous component analysis (\texttt{HCA}) \citep{hca} assumes data come from different sources with different levels of noise. To better learn the PCs with heteroscedastic variance, \texttt{HCA} reweights the empirical loss of each observation according to the inverse of its variance so that noisier samples contribute less to the total loss. \citet{hcaoptimalweight} calculates the optimal weights in the asymptotic case by considering the signal-to-noise ratio. Though these methods have superior performance compared to uniform weighting PCA, heterogeneity among different sources is only modeled by the noise magnitude. A few heuristic methods also attempt to use low-rank features to characterize heterogeneity, including joint and individual variance explained (\texttt{JIVE}) \citep{jive}, common and individual feature extraction (\texttt{CIFE}) \citep{cife}. However, it is difficult to distribute these methods for federated learning and provide theoretical guarantees for their outputs.

\textbf{Distributed PCA}
There has been a recent push to calculate PCs on distributed devices. Oftentimes, the clients/edge devices use their local data to estimate PCs and communicate with a central server to update their estimates. One round of information exchange between clients and the central server is referred to as a communication round. Based on the number of communication rounds between edge devices and the server, research can be roughly divided into two categories (1) those that require only one round of communication and (2) those that require multiple rounds of communication.

For one-round PCA algorithms, clients estimate PCs from local datasets and send summary statistics to the server. The server then analyzes the aggregated statistics to calculate the PCs of the entire dataset. There are several ways for the server to calculate PCs. \citet{dispca2} proposes a method to reconstruct the aggregated covariance matrix by averaging the clients' covariance matrices approximated by a few top PCs. Global PCs can be obtained by learning the top eigenvalues of the averaged covariance matrix. \dispca \citep{oneshotdpca} provides an alternative approach, where the server stacks locally calculated PCs into a large matrix and runs another PCA on the stacked matrix. \citet{stacksingularvectorsvd} uses a similar
method, where clients calculate a singular value decomposition (SVD) of the local observation matrix, and then send the singular values and singular vectors to the server. The server stacks the scaled singular vectors and runs SVD on the stacked matrix. Federated PCA \citep{federatedpca} considers streaming data applications where edge devices have limited memory budgets. In their work, locally estimated subspaces are hierarchically merged to form the global subspace. \citet{dispca1} also focuses on streaming data and reduces large datasets into smaller ones. In spite of the reductions in communication or memory cost, these algorithms are often not guaranteed to recover true PCs exactly. Also, they are built upon homogeneity assumptions and neglect statistical heterogeneity among the distributed datasets.

To obtain more refined estimates of PCs from distributed datasets, a series of works propose to use multiple rounds of communication \citep{shiftinverse,shiftinverse2,huangpcariemann,quantizedriemann}. Among them, \citet{shiftinverse} and \citet{shiftinverse2} design PC updates by shift-and-invert iterations. The shift-and-inverse method \citep{shiftandinversecentral} reformulates inverse power iteration as an unconstrained convex optimization problem and uses gradient-based iterative algorithms to solve it. With a similar rationale, \citet{shiftinverse} applies the shift-and-invert formulation to distributed settings and applies distributed Newton methods to solve for the top eigenvector of the covariance matrix. Then, the covariance matrix is deflated to calculate the subsequent eigenvectors. Besides shift-and-invert iterations, manifold optimization is also employed for PCA. \citet{huangpcariemann} uses distributed Riemann optimization to find top PCs from homogeneous datasets. To further reduce communication costs, \citet{quantizedriemann} combines quantized distributed optimization and Riemannian gradient descent with an exponential map to calculate the leading eigenvectors of the covariance matrix. These methods usually treat the difference among clients' covariance matrices as errors. Thus, when datasets are heterogeneous, the errors are large, and these algorithms fail to retrieve true PCs. 

\textbf{Gradient descent on manifolds}
The centralized version of gradient descent on manifolds, or Riemaniann gradient descent, has been well-studied \citep{optimizationonmatrixmanifold,intromanifolds}. Algorithms based on exponential mappings \citep{stiefelgeometry} can achieve convergence rates comparable to their Euclidean counterparts. Since exponential mappings are expensive to compute, there are algorithms that replace them with retractions. \citet{proofonkpca} presents an elegant framework for analyzing kPCA by Riemannian gradient descent with Cayley retraction. This work proves the local linear convergence of Stiefel gradient descent and also shows that the algorithm can exactly recover the top eigenspaces. 

Recent years have also seen advances in distributed manifold optimization. \citet{mingyiconsensus,mingyistiefel} introduces a simple distributed St-GD algorithm that minimizes a general objective on the manifold. In each round, clients use St-GD on the local objectives and send the updated variables to the server, then the server averages the received update and applies a retraction. The algorithm is guaranteed to converge into stationary points with a sublinear rate. 

\name also exploits St-GD to solve PCs. However, our algorithm enhances simple manifold optimization by simultaneously optimizing local and global PCs, while also incorporating orthogonality constraints between the global and local PCs. \name thus introduces a special correction step to handle such constraints. This is done by defining a new retraction measure we name as a \emph{generalized retraction} defined on the entire $\mathbb{R}^{d\times r}$ rather than the tangent bundle of Stiefel manifold $St(d,r)$.

We should note that among all the distributed algorithms discussed, only \name models different or distributed datasets by global and local PCs. Thus, it brings unique advantages in decoupling local and global features from highly heterogeneous datasets. Besides, there are several additional benefits of \name in convergence and computation compared with typical existing models. In terms of convergence, \name converges into stationary points of the empirical reconstruction error and is guaranteed to recover true PCs exactly with proper initialization. The algorithm does not involve a computationally intensive exponential map and can solve $k$ PCs at one time. More importantly, \name is fully federated, and different clients can collaborate by only sharing a few global PCs that encode shared and not unique features. The comparisons of \name and several typical PCA algorithm is summarized in Table \ref{tab:relatedwork}. 

\begin{table}[h!]
    \centering
\begin{tabular}{cccccc}
\hline
    Method & Source &\thead{Exact\\ convergence} &        kPCA & Federated & Personalized \\

\hline
\texttt{Robust PCA} &\citep{robustpca}  & \yes  & \yes & \no & \no\\

\texttt{JIVE} &\citep{jive} &        \no  &        \yes &         \no &         \yes \\

\dispca &\citep{oneshotdpca} &          \no  &         \yes &        \yes &         \no \\

\texttt{Distri-Eigen} & \citep{shiftinverse} &        \yes  &         \no &        \yes &         \no \\

\texttt{CEDRE} &\citep{huangpcariemann} &        \yes &         \no &        \yes &         \no \\

\texttt{PCA by St-GD} &\citep{proofonkpca} &        \yes  &        \yes &         \no &         \no \\

\textbf{\name} &\textbf{ ours} &        \yes  &        \yes &        \yes &        \yes \\
\hline
\end{tabular}  
    \caption{Comparison of related work. Metrics included and their definitions are: (i) Exact convergence: the algorithm can recover top subspaces of sample covariance matrix exactly, 
    (ii) kPCA: the algorithm can calculate the subspace spanned by top $k$ PCs instead of one single component, (iii) Federated: the algorithm can be done in a distributed fashion where raw data remains where it is generated on the edge and only focused updates need to be shared across clients, (iv) Personalized: the algorithm encodes both shared and unique features across all datasets. }
    \label{tab:relatedwork}
\end{table}


\subsection{Notations}
We first introduce needed notations in this subsection. For a $d$-dimensional vector $\bm{x}$, we use $\norm{\bm{x}}$ to denote its 2-norm. The inner product of two vectors is defined as a standard inner product in Euclidean space: $\langle \bm{x},\bm{y}\rangle=\bm{x}^T\bm{y}$. We use $\bm{I}_d$ to denote the identity matrix in $\mathbb{R}^d$. We sometimes omit the subscript $d$ if the dimension is clear from the context. For a real matrix $\bm{A}\in \mathbb{R}^{m \times n}$, we use $\norm{\bm{A}}_F$ to denote its Frobenius norm $\norm{\bm{A}}_F=\sqrt{\sum_{i=1}^n\sum_{j=1}^m\bm{A}_{ij}^2}$ and $\norm{\bm{A}}_{op}$ to denote its operator norm $\norm{\bm{A}}_{op}=\max_{\bm{v}\in\mathbb{R}^n,\norm{\bm{v}}=1}\norm{\bm{A}\bm{v}}$. For two matrices $A,B\in \mathbb{R}^{m \times n}$, we define their inner product as $
\left\langle \bm{A},\bm{B}\right\rangle = \sum_{i=1}^n\sum_{j=1}^mA_{ij}B_{ij}=\tr{\bm{A}^T\bm{B}}
$.

If $\bm{A}\in \mathbb{R}^{n \times n}$ is symmetric positive definite (PSD), it has an eigendecomposition $\bm{A} = \bm{U}\bm{D}\bm{U}^T$, where $\bm{D}$ is $n$ by $n$ diagonal matrix whose diagonal entries are all positive, $\bm{U}$ is a $n$ by $n$ unitary matrix. Then for $p\in\mathbb{R}$, the $p$-th power of $\bm{A}$ is defined as $\bm{A}^{p}=\bm{U}\bm{D}^p\bm{U}^T$. For a square matrix $\bm{A}\in \mathbb{R}^{n \times n}$, we use $\lambda_{min}(\bm{A})$ and $\lambda_{max}(\bm{A})$ to denote the minimum and maximum eigenvalue of $\bm{A}$. Similarly, we use $\lambda_1(\bm{A})$, $\lambda_2(\bm{A})$, ... $\lambda_n(\bm{A})$ to denote the $n$ eigenvalues of $\bm{A}$ in descending order. We use $\norm{\bm{A}}_{op}$ and $\lambda_{max}(\bm{A})$ interchangeably when $\bm{A}$ is symmetric PSD. 

For a matrix $\bm{A}\in \mathbb{R}^{m \times n}$, we use $\vecr{\bm{A}}\in \mathbb{R}^{mn}$ to denote its vectorization, i.e., the vector formed by concatenating all the column vectors in $\bm{A}$. $col(\bm{A})$ is the linear subspace spanned by all column vectors of $\bm{A}$. 
We use $\bm{A}_{i_1:i_2,j_1:j_2}$ to denote the submatrix of $A$ formed by picking the $i_1,i_1+1...i_2$-th row and $j_1,j_1+1...j_2$-th column of $\bm{A}$. For two matrices $\bm{A}\in \mathbb{R}^{m \times n_1}$ and $\bm{B}\in \mathbb{R}^{m \times n_2}$, $[\bm{A},\bm{B}]\in \mathbb{R}^{m \times (n_1+n_2)}$ is defined as the concatenation of $\bm{A}$ and $\bm{B}$ by column.

Finally, we use the standard $O\left(\cdot\right)$, $\Omega\left(\cdot\right)$, and $o\left(\cdot\right)$ notations throughout the paper.

\section{What is \name?}
\label{sec:formulation}
We will establish the formulation of \name in this section.
\subsection{Motivation}
Suppose we have $N$ clients (i.e. data sources), each with a dataset $\{\bm{Y}_{(i)}\}_{i=1}^N$, where $\bm{Y}_{(i)}$ is a $d$ by $n_i$ matrix. $d$ is the dimension of data, and $n_i$ is the the number of datapoints on client $i$. 
The datasets $\{\bm{Y}_{(i)}\}_{i=1}^N$ have commonalities but also possess client-level distinctive features. The task is to find a few low-dimensional common and unique features that best characterize the observations from the high dimensional data $\{\bm{Y}_{(i)}\}_{i=1}^N$.

Standard PCA uses a small number of principal components (PCs) to explain the variations in $\{\bm{Y}_{(i)}\}_{i=1}^N$. Such treatment ignores the client-to-client difference in the observations. The present IoT system usually consists of distributed edge devices (clients) that operate in extremely heterogeneous environments. It is thus important to consider the different features of different clients. As a more capacious description of the data, we consider the model where local observations are driven by $r_1$ global PCs and $r_{2,(i)}$ local PCs. More specifically, from data source $i$, observation $\bm{y}_{(i)}$ is generated from
\begin{equation}
\label{eqn:model}
    \bm{y}_{(i)} \sim \sum_{q=1}^{r_1}\phi_{(i),q}\bm{u}_q+\sum_{q=1}^{r_{2,(i)}}\varphi_{(i),q}\bm{v}_{(i),q}+\bm{\epsilon}_{(i)}
\end{equation}
where $\phi_{(i),q}$'s and $\varphi_{(i),q}$'s are coefficients, or scores in PCA terminology. $\bm{u}_q$'s are global PCs, $\bm{v}_{(i),q}$'s are local PCs, and $\bm{\epsilon}_{(i)}$ are i.i.d. noise vectors. $r_1$ is the number of global PCs, and $r_{2,(i)}$ is the number of local PCs on client $i$. We allow $\bm{v}_{(i),q}$'s to be client-dependent while enforcing $\bm{u}_q$'s to remain the same across all clients. Naturally, $\bm{u}_q$'s encode the information shared by all participants, while $\bm{v}_{(i),q}$'s can describe distinctive patterns on each client.

Similar to standard PCA, different principal components need to be orthonormal:
\begin{equation}
\label{eqn:orthonormalconstraints}
\left\{\begin{aligned}
&\bm{u}_{q_1}^T\bm{u}_{q_2}=\delta_{q_1,q_2}\\
&(\bm{v}_{(i),q_1})^T\bm{v}_{(i),q_2}=\delta_{q_1,q_2},\, \forall q_1,q_2,\,\forall i=1,\cdots,N
\end{aligned}\right.
\end{equation}
where $\delta_{q_1,q_2}$ is the Kronecker delta. In addition to \eqref{eqn:orthonormalconstraints}, we further require that the global and local features are orthogonal:
\begin{equation}
    \bm{u}_{q_1}^T\bm{v}_{(i),q_2}=0,\,\forall i=1,\cdots,N
\end{equation}
The orthogonality of PCs implies that the shared and unique features span different
subspaces, thus describing independent and decoupled patterns in the data sources.

\eqref{eqn:model} is an interpretable linear model that naturally incorporates both common and individual features of different clients. It is useful in applications where disentangling global and local features is important. The development of IoT and recent advancements in federated and distributed analytics present numerous such applications, including time series data, image and video data, and language data. We will show the efficacy of \eqref{eqn:model} on several examples.

\subsection{Method}
The task of \name is to recover global and local PCs from observations $\{\bm{Y}_{(i)}\}_{i=1}^N$. We can write global and local PCs into matrix form:
\begin{equation}
\left\{\begin{aligned}
&\bm{U}=[\bm{u}_1,\cdots,\bm{u}_{r_1}]\\
&\bm{V}_{(i)}=[\bm{v}_{(i),1},\cdots,\bm{v}_{(i),r_{2,(i)}}]
\end{aligned}\right.
\end{equation}
and solve for $\bm{U}$ and $\bm{V}_{(i)}$'s by minimizing the empirical reconstruction loss:
\begin{equation}
\label{eqn:naiveproblemformulation}
\begin{aligned}
\min_{\bm{U},\{\bm{V}_{(i)}\}_{i=1,\cdots,N}}& \frac{1}{2}\sum_{i=1}^N\frac{1}{n_i}\norm{\bm{Y}_{(i)}-\hat{\bm{Y}}_{(i)}}_F^2\\
\text{subject to  }&  
\bm{U}^T\bm{U}=\bm{I},\, \bm{V}_{(i)}^T\bm{V}_{(i)}=\bm{I},\,  \bm{V}_{(i)}^T\bm{U}=\bm{0},\, \forall i \\
\end{aligned}
\end{equation}
where $\hat{\bm{Y}}_{(i)}$ is the statistical fit for client $i$'s data given PCs $\bm{U}$ and $\bm{V}_{(i)}$:
\begin{equation}
\label{eqn:yhatofuvi}
    \hat{\bm{Y}}_{(i)} = \bm{U} \bm{U}^T\bm{Y}_{(i)} +  \bm{V}_{(i)}\bm{V}_{(i)}^T\bm{Y}_{(i)}
\end{equation}

Intuitively, in \eqref{eqn:naiveproblemformulation}, we look for the PCs so that the predicted $\hat{\bm{Y}}_{(i)}$ can best fit the distributed datasets. The objective \eqref{eqn:naiveproblemformulation} has another interpretation: by some algebra, we can transform the objective \eqref{eqn:naiveproblemformulation} into:
\begin{equation}
\begin{aligned}
\label{eqn:problem}
\max_{\bm{U},\{\bm{V}_{(i)}\}_{i=1,\cdots,N}}& \frac{1}{2}\sum_{i=1}^N\left[\tr{\bm{U}^T\bm{S}_{(i)}\bm{U}}+\tr{\bm{V}_{(i)}^T\bm{S}_{(i)}\bm{V}_{(i)}}\right]\\
\text{subject to  }&  
\bm{U}^T\bm{U}=\bm{I},\, \bm{V}_{(i)}^T\bm{V}_{(i)}=\bm{I},\,  \bm{V}_{(i)}^T\bm{U}=\bm{0},\, \forall i \\
\end{aligned}
\end{equation}
where $\bm{S}_{(i)}$ is defined as the data covariance matrix:
$$
\bm{S}_{(i)}=\frac{1}{n_i}\bm{Y}_{(i)}\bm{Y}_{(i)}^T
$$

From \eqref{eqn:problem}, it is clear that \name attempts to find global and local low dimensional subspaces that best align with the data covariance matrix. We will study objective \eqref{eqn:problem} from here on.

For simplicity, we introduce 
\begin{equation}
\label{eqn:defoffi}
f_i(\bm{U},\bm{V}_{(i)})=\frac{1}{2}\tr{\bm{U}^T\bm{S}_{(i)}\bm{U}}+\frac{1}{2}\tr{\bm{V}_{(i)}^T\bm{S}_{(i)}\bm{V}_{(i)}}
\end{equation}
and
\begin{equation}
\label{eqn:defoff}
f(\bm{U},\{\bm{V}_{(i)}\})=\sum_{i=1}^Nf_i(\bm{U},\bm{V}_{(i)})
\end{equation}
Then \eqref{eqn:problem} transforms to maximizing $f$ under orthonormality constraints. Notice that though $f$ and $f_i$'s are convex, the constraint in \eqref{eqn:problem} is nonconvex. Thus, the problem is nonconvex. 

The nonconvex formulation \eqref{eqn:problem} appears difficult to analyze and solve. In the following sections, we will delve into the identifiability and optimization of \eqref{eqn:problem}. Fortunately, our results show that under minimal conditions, \eqref{eqn:problem}
can be solved efficiently, and the optimal solution can recover the true PCs.

\section{Are Global and Local PCs Identifiable?}
\label{subsec:statisticalefficiency}

Given the formulation \eqref{eqn:problem}, one may ask whether it is possible to identify the true local and global PCs by solving \eqref{eqn:problem}. 

Apparently, global and local PCs cannot be decoupled in every case. As a simple counterexample, if all local PCs are the same, then distinguishing local from global PCs is impossible, as there are infinite combinations of them that all can maximize the explained variance in \eqref{eqn:problem}. The edifying counterexample poses the fundamental question of model identifiability. Therefore, we need to find out which data instances are identifiable. In the following, we will introduce an identifiability condition, then establish the relationship between the estimated and true PCs.


We restrict our analysis to recovering the subspace spanned by top PCs \citep{matrixbook}. Therefore we introduce the projection matrix notation $\bm{P}_{\bm{U}}$: if $\bm{U}$ is a matrix with orthonormal columns, i.e. $\bm{U}^T\bm{U}=\bm{I}$, then $\bm{P}_{\bm{U}}$ is defined as $\bm{P}_{\bm{U}}=\bm{U}\bm{U}^T$. We use $\matPi_g$ to denote the projection matrix to the true global eigenspace, i.e., $\matPi_g=\Pj{\bm{U}_{\text{true}}}$, where $\bm{U}_{\text{true}}$ are the true top global PCs. Also, we use $\matPi_{(i)}$ to denote the projection matrix to the true local eigenspace, $\matPi_{(i)}=\Pj{\bm{V}_{(i),\text{true}}}$, where $\bm{V}_{(i),\text{true}}$ are the true top local PCs on client $i$. 

Remember, we model global and local PCs as mutually vertical features; such property can be formally characterized by the following assumption. 
\begin{assumption}
\label{ass:covariancedecompose}
(Orthogonality of global and local PCs) Let $\matPi_g$ be the global projection matrix, $\matPi_{(i)}$'s the local projection matrices. We assume that 
\begin{equation}
\matPi_g\matPi_{(i)} = 0
\end{equation}
In addition, we consider the case where the subspace corresponding to the projection $\matPi_g+\matPi_{(i)}$ is indeed an invariant subspace of the population covariance matrix on client $i$, $\bm{\Sigma}_{(i)}$ 
, i.e. $\left(\matPi_g+\matPi_{(i)}\right)\bm{\Sigma}_{(i)}=\bm{\Sigma}_{(i)}\left(\matPi_g+\matPi_{(i)}\right)$.
\end{assumption}
In Assumption \ref{ass:covariancedecompose}, the requirement $\left(\matPi_g+\matPi_{(i)}\right)\bm{\Sigma}_{(i)}=\bm{\Sigma}_{(i)}\left(\matPi_g+\matPi_{(i)}\right)$ essentially assumes that $\matU_{true}$ and $\matV_{(i),true}$ are indeed the eigenvectors of the population covariance matrix $\matSigma_{(i)}$.

As the counterexample suggests, assumption \ref{ass:covariancedecompose} alone is insufficient to guarantee the identifiability of global and local PCs. To distinguish them, we need another identifiability condition. To rule out the counterexample, local PCs and accordingly $\matPi_{(i)}$, should differ from each other. To this end, we introduce the notion of ``\textit{misalignment}''. Misalignment is quantified by the parameter $\theta$, which represents the maximum eigenvalue of the average of the local projection matrices. Assumption \ref{ass:identifiability} is a formal statement of the identifiability condition.
\begin{assumption}
\label{ass:identifiability}
(Misalignment) Let $\matPi_{(i)}$'s be the local projection matrices. We assume there exists a positive constant $\theta\in(0,1)$ such that:
\begin{equation}
\lambda_{\max}\left(\frac{1}{N}\sum_{i=1}^N\matPi_{(i)}\right)\le 1-\theta
\end{equation}
\end{assumption}
The constant $\theta$ characterizes the misalignment between local principal spaces. When $\theta$ is larger, the local eigenspaces are more heterogeneous. When $\theta$ is smaller, the local eigenspaces are more similar. As an extreme case, if all $\matPi_{(i)}$'s are identical, $\frac{1}{N}\sum_{i=1}^N\matPi_{(i)}$ is still a projection, thus its maximum eigenvalue is $1$ and $\theta$ becomes zero. 

\subsection{Statistical error}
It turns out that the identifiability Assumption \ref{ass:identifiability} is sufficient to ensure identifiability. The following perturbation bound shows that when the sample covariance matrix is close to the population covariance matrix, we can obtain relatively accurate estimates of global and local eigenspaces through solving \eqref{eqn:problem}.

\begin{theorem}
\label{thm:statisticalerror}
Under assumption \ref{ass:covariancedecompose} and \ref{ass:identifiability}, and if there exists a constant $\delta>0$, such that $\lambda_{r_1+r_{2,(i)}}\left(\left(\matPi_g+\matPi_{(i)}\right)\matSigma_{(i)}\right)-\lambda_{1}\left(\left(\matI-\matPi_g-\matPi_{(i)}\right)\matSigma_{(i)}\right)\ge \delta$ for all $i$, we have:
\begin{equation}
\label{eqn:statisticalerror}
\norm{\bm{P}_{\hat{\bm{U}}}-\matPi_g}_F^2+\frac{1}{N}\sum_{i=1}^N\norm{\bm{P}_{\hat{\bm{V}}_{(i)}}-\matPi_{(i)}}_F^2\le \frac{8}{\theta\delta^2} \frac{1}{N}\sum_{i=1}^N\norm{ \bm{\Sigma}_{(i)}-\bm{S}_{(i)}}_{F}^2
\end{equation}
where $\hat{\bm{U}}$, and $\hat{\bm{V}}_{(i)}$'s are the optimal solutions to the objective in \eqref{eqn:problem}. 
\end{theorem}
$\delta$ is usually called eigengap in literature \citep{fantope,huangpcariemann}. The $\delta^{-2}$ factor on the right-hand side of \eqref{eqn:statisticalerror} is standard for matrix perturbation analysis. 

Theorem \ref{thm:statisticalerror} confirms the intuition on identifiability. Specifically, as $\theta$ increases, the right-hand side of equation \eqref{eqn:statisticalerror} decreases, resulting in a smaller estimation error. Consequently, finding local and global PCs becomes easier. \textit{This result critically highlights that heterogeneity can be a blessing}. For the counterexample, $\theta\to 0$, the right-hand side approaches infinity. Hence, one cannot accurately recover the PCs. 

In addition, Theorem \ref{thm:statisticalerror}\textit{ highlights the benefits of collaborative learning} across multiple related clients. The right-hand side of \eqref{eqn:statisticalerror} is the average difference between the sample and population covariance matrix on all clients. For clients with a larger dataset, the distance is lower, and for clients with a smaller dataset, the distance can be higher. Through jointly optimizing objective \eqref{eqn:problem}, clients learn from each other and obtain PC estimates with statistical error depending on the average distance. 

\subsection{Minimax statistical lower bound}
\revise{Though the statistical error bound provided in Theorem \ref{thm:statisticalerror} is intuitive, it is not apparent whether the bound is sharp. To fully understand the statistical difficulty in recovering shared and unique components from $\{\matS_{(i)}\}$, we will establish a lower bound on the minimax
risk of estimators under the subspace error.}

\revise{For simplicity, we define the subspace error between $\{\hmatU,\{\hmatV_{(i)}\}\}$ and $\{\matU,\{\matV_{(i)}\}\}$ as
\begin{align}
\label{eqn:deflsubspace}
L_{\texttt{subspace}}\left(\{\hmatU,\{\hmatV_{(i)}\}\},\{\matU,\{\matV_{(i)}\}\}\right)=\norm{\Pj{\hmatU}-\Pj{\matU}}_F^2+\frac{1}{N}\sum_{i=1}^N\norm{\Pj{\hmatV_{(i)}}-\Pj{\matV_{(i)}}}_F^2
\end{align}
}

\revise{Additionally, we use $\Theta$ to denote the parameter space specified by Assumption \ref{ass:covariancedecompose},
\begin{align}
\Theta=\left\{\matU,\{\matV_{(i)}\}|\matU^T\matU=\matI, \matV_{(i)}^T\matV_{(i)}=\matI,\matU^T\matV_{(i)}=0\right\}
\end{align}
}
The following theorem provides a lower bound for the statistical error.

\begin{theorem}
\label{thm:statisticallowerbound}
If the data generation process satisfies Assumptions \ref{ass:covariancedecompose} and \ref{ass:identifiability}, the eigengap introduced in Theorem \ref{thm:statisticalerror} is at least $\delta$, and $\sum_{i=1}^N\norm{\matS_{(i)}-\matSigma_{(i)}}_F^2=o(1)$, then among data generated by all possible $\{\matU_{true},\{\matV_{(i),true}\}\}\in \Theta$, the supremum of the subspace error between the optimal solution to \eqref{eqn:problem}, $\{\hmatU,\{\hmatV_{(i)}\}\}$, and the ground truth, $\{\matU_{true},\{\matV_{(i),true}\}\}$, is at least
\begin{equation}
\sup_{\{\matU_{true},\{\matV_{(i),true}\}\}\in\Theta} \frac{L_{\texttt{subspace}}\left(\{\hmatU,\{\hmatV_{(i)}\}\},\{\matU_{true},\{\matV_{(i),true}\}\}\right)}{\frac{1}{N}\sum_{i=1}^N\norm{\matSigma_{(i)}-\matS_{(i)}}_F^2}=\Omega \left(\frac{1}{\theta}+\frac{1}{\delta^2}\right)
\end{equation}
\end{theorem}

\revise{Theorem \ref{thm:statisticallowerbound} measures the subspace error minimax lower bound in terms of misalignment parameter $\theta$ and eigengap $\delta$. Roughly speaking, the lower bound is greater than $\Omega \left(\left(\frac{1}{\theta}+\frac{1}{\delta^2}\right)\frac{1}{N}\sum_{i=1}^N\norm{\matSigma_{(i)}-\matS_{(i)}}_F^2\right)$. This almost matches the upper bound provided in Theorem \ref{thm:statisticalerror} as the error scales with $\frac{1}{\theta}$ and $\frac{1}{\delta^2}$. Theorem \ref{thm:statisticallowerbound} also demonstrates the intrinsic statistical difficulty of separating global and local PCs. When the local PCs are more aligned and noise components grow larger, $\theta$ and $\delta$ become smaller, and the statistical error of the subspace estimate becomes larger accordingly. }

\revise{The proof of Theorem \ref{thm:statisticallowerbound} is based on a variant of the ``spiked population model'' \citep{sparsepcalower}. We use perturbation analysis to calculate the leading order of the subspace error when the sample covariance is close to the population covariance. The full proof is in Appendix \ref{ap:proofstatisticallowerbound}. There is also a comparison between the theoretical statistical error estimate and the statistical error obtained from numerical simulations in Appendix \ref{ap:proofstatisticallowerbound}. }

\subsection{Sample complexity}
In this section, we estimate the statistical error when data are generated by a sub-Gaussian distribution. A random vector $\bm{y} \in \mathbb{R}^{d}$ admits a sub-Gaussian distribution with parameter $\sigma$ if for each fixed vector $\bm{v}\in \mathbb{S}^{d-1}$, $\mathbb{E}\left[e^{\lambda\innerp{\bm{v}}{y}}\right]\le e^{\frac{\lambda^2\sigma^2}{2}}$ for all $\lambda\in\mathbb{R}$. $\sigma$ is a parameter that denotes the variance level: when $\sigma$ is larger the data are noisier. As a special case, if $\bm{y}$ admits a Gaussian distribution with mean zero and covariance $\bm{\Sigma}_y$, then $\sigma^2=\norm{\bm{\Sigma}_y}_{op}$ \citep{wainwrightbook}. The following corollary gives an upper bound of the estimation error.
\begin{corollary}
\label{cor:statisticalerrorbydn}
If the dataset on each client $i$ $\{\bm{Y}_{(i)}\}_{i=1}^N$ admits an i.i.d. sub-Gaussian distribution with parameter $\sigma$, and the assumptions in Theorem \ref{thm:statisticalerror} are satisfied, then with probability at least $1-\widetilde{\delta}$ (over the randomness of the data generation process), we have:
\begin{equation}
\label{eqn:statisticalerrorbydn}
\norm{\bm{P}_{\hat{\bm{U}}}-\matPi_g}_F^2+\frac{1}{N}\sum_{i=1}^N\norm{\bm{P}_{\hat{\bm{V}}_{(i)}}-\matPi_{(i)}}_F^2\le \frac{1}{\theta\delta^2} \sigma^4 C^2\frac{d}{N}\sum_{i=1}^N\max\left\{\left(\frac{d+\log\frac{2N}{\widetilde{\delta}}}{n_i}\right)^2,\frac{d+\log\frac{2N}{\widetilde{\delta}}}{n_i}\right\}
\end{equation}
where $C$ is a constant.
\end{corollary}
The inequality \eqref{eqn:statisticalerrorbydn} essentially shows the consistency of the solutions $\hat{\bm{U}}$ and $\hat{\bm{V}}$. When the data dimension $d$ is fixed and sample size $n_i$ is relatively large, the right hand side of \eqref{eqn:statisticalerrorbydn} decreases with  $O\left(\sum_{i=1}^N\frac{1}{N\theta\delta^2n_i}\right)$. As $n_i$'s approach infinity, the subspace error also decreases to $0$, and the estimated eigenspaces approach the true values accordingly.

Equation \eqref{eqn:statisticalerrorbydn} also highlights the benefits of knowledge sharing. If each client only uses their own data to estimate the PCs, the estimation error would be $O\left(\frac{1}{n_i}\right)$. The error can be high for clients with few observations (i.e., small $n_i$). However, when $N$ clients collaborate in learning global and local PCs, the estimation error becomes the average of individual statistical errors $O\left(\sum_{i=1}^N\frac{1}{N\theta\delta^2n_i}\right)$. Data-poor clients can thus borrow strength from other clients to improve the estimates of their PCs. 

The statistical consistency and knowledge-sharing effect will also be examined by numerical experiments in Section \ref{sec:experiment}.

Here, we note that statistical consistency can not be achieved by existing estimates without personalized modeling. For example, the statistical error of \dispca \citep{oneshotdpca} depends on $O\left(\frac{1}{N}\sum_{i=1}^N\norm{\bm{\Sigma}_{(i),l}}_{op}\right)$, which does not decrease with number of observations $n_i$ as long as $\norm{\bm{\Sigma}_{(i),l}}_{op}>0$. The comparison highlights the advantages of personalization through our formulation in \eqref{eqn:problem}.

Now we present the proof of Corollary \ref{cor:statisticalerrorbydn}.
\begin{proof}
We will adopt the covariance concentration bound in  \citet{wainwrightbook} and \citet{covarianceconcentrationopnorm}. Since data on client $i$ admit independent sub-Gaussian distributions, theorem 13.3 in \citet{covarianceconcentrationopnorm} states that, with probability at least $1-\delta_1$, there exists a constant $C$ such that:
$$
\norm{\bm{\Sigma}_{(i)}-\bm{S}_{(i)}}_{op}\le \sigma^2C \max\left\{\sqrt{\frac{d+\log\frac{2}{\delta_1}}{n_i}},\frac{d+\log\frac{2}{\delta_1}}{n_i}\right\}
$$
We can choose $\delta_1=\frac{\widetilde{\delta}}{N}$. Then by a union bound, we know that with probability at least $1-\widetilde{\delta}$:
$$
\norm{\bm{\Sigma}_{(i)}-\bm{S}_{(i)}}_{op}\le \sigma^2C \max\left\{\sqrt{\frac{d+\log\frac{2N}{\widetilde{\delta}}}{n_i}},\frac{d+\log\frac{2N}{\widetilde{\delta}}}{n_i}\right\}
$$
holds for all $i$.

Combining this and Theorem \ref{thm:statisticalerror}, we can prove the bound in \eqref{eqn:statisticalerrorbydn}.
\end{proof}

Equation \eqref{eqn:statisticalerrorbydn} also gives a simple estimate of the sample complexity.
\begin{corollary}
Under the assumptions of Theorem \ref{thm:statisticalerror}, and assuming that data on client $i$ admits an i.i.d sub-Gaussian with parameter $\sigma$, if each client has at least  $O\left(\frac{1}{\epsilon}\frac{\sigma^4d^2}{\theta \delta^2}\right)$ observations, then with high probability, the estimation error is smaller than $\epsilon$.
\end{corollary}
\begin{proof}
The proof is quite straightforward. Notice that when $n_i\ge d$, the right-hand side of \eqref{eqn:statisticalerrorbydn} is dominated by $\frac{d}{n_i}$. Thus if we neglect the logarithm factors on the right-hand side of \eqref{eqn:statisticalerrorbydn} and set $
\frac{4}{\theta\delta^2}\sigma^4C^2d\frac{1}{N}\sum_{i=1}^N\frac{d}{n_i}\le \epsilon$, the statistical error will also be upper bounded by $\epsilon$.

It is natural to see that the inequality holds when each client has observations no less than $O\left(\frac{1}{\epsilon}\frac{\sigma^4d^2}{\theta \delta^2}\right)$. 
\end{proof}

\section{Recovering Local and Global PCs}
\label{sec:algorithm}
The statistical consistency proved in Section \ref{subsec:statisticalefficiency} dwells on the premise that the objective in \eqref{eqn:problem} can be solved to optimality. An efficient algorithm to solve the problem is not apparent as the constraints in \eqref{eqn:problem} are nonconvex. In this section, we develop a class of algorithms to solve \eqref{eqn:problem}. 

The major difficulty in optimizing \eqref{eqn:problem} lies in the nonconvex constraints: in addition to the orthonormal constraints on $\bm{U}$ and $\bm{V}_{(i)}$'s, the constraints $\matU^T\matV_{(i)}=0$ require global and local PCs to be mutually orthogonal. The later constraints introduce interaction between local and global variables, which deems simple distributed Stiefel manifold descent \citep{mingyistiefel} incompetent.

To handle the orthogonality constraints, we propose a class of algorithms that we call Personalized PCA (\name). \name adopts Stiefel manifold gradient descent to ensure that all constraints are satisfied during the algorithm. It is worth noting that \name is naturally \textit{federated} as the computation is distributed over clients, and only updates of the global PCs need to be shared. 

In the following of this section, we will build the \name algorithm step by step. But before delving into the technical details of parallel gradients and retractions, to illustrate the essence of \name, we will first present a simple instance of \name that exploits polar projections to maintain the orthonormality of the updates.

\subsection{An instance of \name}
The polar projection of a general full-column-rank matrix $\matW\in\mathbb{R}^{n_1\times n_2}$ where $n_1\ge n_2$ returns an orthonormal matrix defined as
\begin{align}
\label{eqn:defofpolar}
\polar\left(\matW\right)=\matW\left(\matW^T\matW\right)^{-\frac{1}{2}}
\end{align}
Polar projection can be efficiently implemented via SVD \citep{mmalgorithm}. It is shown that among all the orthonormal matrices, $\polar\left(\matW\right)$ is closest to $\matW$ \citep{berkeleyproof}. Therefore, we can combine gradient descent with polar projection to solve problem \eqref{eqn:problem}. The pseudocode is summarized in Algorithm \ref{alg:polarprojection}.

\begin{algorithm}
   \caption{ An instance of \name using Polar Projection}
   \label{alg:polarprojection}
\begin{algorithmic}
   \STATE Input client covariance matrices $\{\bm{S}_{(i)}\}_{i=1}^N$, stepsize $\eta_\tau$
\STATE Initialize $\bm{U}_1$, and $\bm{V}_{(1),\frac{1}{2}},\cdots,\bm{V}_{(N),\frac{1}{2}}$.
\FOR{Communication rounds $\tau=1,...,R$}
\FOR{Client $i=1,\cdots,N$}
\STATE $\bm{V}_{(i),\tau}=\polar\left({\bm{V}_{(i),\tau-\frac{1}{2}}-\bm{U}_{\tau}\bm{U}_{\tau}^T\bm{V}_{(i),\tau-\frac{1}{2}}}\right)$
\STATE $[\bm{U}_{(i),\tau+1},\bm{V}_{(i),\tau+\frac{1}{2}} ]=\polar\left(\left[\bm{U}_{\tau},\bm{V}_{(i),\tau}\right]+\eta_{\tau}\bm{S}_{(i)}\left[\bm{U}_{\tau},\bm{V}_{(i),\tau}\right]\right)$
\STATE Uploads $\bm{U}_{(i),\tau+1}$ to server.
\ENDFOR
   \STATE Server calculates $\bm{U}_{\tau+1}=\polar\left(\frac{1}{N}\sum_{i=1}^N\bm{U}_{(i),\tau+1}\right)$
\STATE Server broadcasts $\bm{U}_{\tau+1}$   
   \ENDFOR
\end{algorithmic}
\end{algorithm}

In Algorithm \ref{alg:polarprojection}, at each iteration, client $i$ first deflates $\bm{V}_{(i),\tau-\frac{1}{2}}$ to make it orthogonal to $\matUt$. This ensures that the updates are feasible as $\matUt^T\matVit=0$, $\matUt^T\matUt=\matI$, and $\matVit^T\matVit=\matI$. Then client $i$ uses gradient ascent and polar projection to update $\matU_{(i),\tau+1}$ and $\matV_{(i),\tau+\frac{1}{2}}$. This step increases the objective while respecting the orthonormal constraints on $\matU$ and $\matV_{(i)}$. After the updates, client $i$ sends the updated $\matU_{(i),\tau+1}$ to the server. The server takes the average of all received $\matU_{(i),\tau+1}$, orthonormalize it, then broadcast the updated $\matU_{\tau+1}$.

It is intuitively understandable how the iterations in Algorithm \ref{alg:polarprojection} maximize the objective while keeping the updates feasible. In the rest of this section, we will study a broader class of algorithms through the lens of manifold optimization and show that Algorithm \ref{alg:polarprojection} is actually a special case of such algorithm class. We will begin by reviewing a few concepts from manifold optimization and then provide our definition for a class of operations called ``\emph{generalized retraction}''. Then, we will use the techniques from Stiefel gradient descent to design a class of algorithms that solves \eqref{eqn:problem}. 

\subsection{Generalized retractions}
We begin by introducing the Stiefel manifold commonly used in matrix analysis \citep{stiefelgeometry}.

The Stiefel manifold $St(d,r)$ is the set of all $d$ by $r$ orthonormal matrices:
\begin{equation}
\label{eqn:stiefelmanifolddef}
St(d,r)=\{\bm{U}\in\mathbb{R}^{d\times r}|\bm{U}^T\bm{U}=\bm{I}\}
\end{equation}
It is embedded in a $d \times r$ dimensional Euclidean space. One can verify that $St(d,r)$ is not convex in general \citep{stiefelgeometry}. 

For $\bm{U}\in St(d,r)$, the tangent space of $ St(d,r)$ at $\bm{U}$ is defined as:
$$
\mathcal{T}_{\bm{U}}=\{\bm{\xi}\in\mathbb{R}^{d\times r}|\bm{\xi}^T\bm{U}+\bm{U}^T\bm{\xi}=0\}
$$
It can be derived by differentiating $\bm{U}^T\bm{U}=\bm{I}$. The normal space $\mathcal{N}_{\bm{U}}$ is defined as the orthogonal space of the tangent space at $\bm{U}$. 

Both $\mathcal{T}_{\bm{U}}$ and $\mathcal{N}_{\bm{U}}$ are linear subspaces of $\mathbb{R}^{d\times r}$. Therefore, we can define the projection onto them. $\mathcal{P}_{\mathcal{N}_{\bm{U}}}$ denotes the projection onto the normal space:
$$
\mathcal{P}_{\mathcal{N}_{\bm{U}}}(\bm{V})= \frac{1}{2}\bm{U}\left(\bm{U}^T\bm{V}+\bm{V}^T\bm{U}\right)
$$
Similarly, $\mathcal{P}_{\mathcal{T}_{\bm{U}}}$ denotes the projection onto the tangent space:
$$
\mathcal{P}_{\mathcal{T}_{\bm{U}}}(\bm{V})= \bm{V}-\mathcal{P}_{\mathcal{N}_{\bm{U}}}(\bm{V})
$$
One can verify that for any matrix $\matV\in \mathbb{R}^{d\times r}$, $\Ptangent{\bm{U}}{\bm{V}}^T\bm{U}+\bm{U}^T\Ptangent{\bm{U}}{\bm{V}}=0$

Next, we introduce the notion of a generalized retraction. The motivation for a generalized retraction is rather straightforward. For an orthogonal matrix $\bm{U}$ and a general update matrix $\bm{\xi}$, the matrix $\bm{U}+\bm{\xi}$ can probably violate the orthonormal constraint: $\left(\bm{U}+\bm{\xi}\right)^T\left(\bm{U}+\bm{\xi}\right)\neq \bm{I}$. The generalized retraction finds an approximation $\bm{U}+\bm{\xi}$ that strictly satisfies the orthonormal constraint. Ideally, the best approximation can be found via projection. However, the projection onto a general nonlinear manifold is hard to analyze. Therefore, one can relax this projection to a generalized retraction. More formally, a generalized retraction can be defined as:
\begin{definition}
\label{def:generalizedretraction}
We call a mapping
$$
\grof{\bm{U}}{\cdot}:\mathbb{R}^{d\times r}\to St(d,r)
$$
a generalized retraction if
\begin{enumerate}
    \item (Property 1): \label{cons:preservecolumnspace} $col(\grof{\bm{U}}{\bm{\xi}})=col(\bm{U}+\bm{\xi}),\ \forall\bm{U}\in St(d,r),\ \forall \bm{\xi} \in \mathbb{R}^{d\times r}$ 
    \item (Property 2): \label{cons:approximate} There exist constants $M_1,M_2\ge 0$ and $M_3>0$ such that: 
    $$
    \begin{aligned}
    \norm{\grof{\bm{U}}{\bm{\xi}}-(\bm{U}+\mathcal{P}_{\mathcal{T}_{\bm{U}}}(\bm{\xi}))}_F &\le M_1\norm{\mathcal{P}_{\mathcal{T}_{\bm{U}}}(\bm{\xi})}^2_F+M_2\norm{\bm{\xi}-\mathcal{P}_{\mathcal{T}_{\bm{U}}}(\bm{\xi})}_F,\\ \forall\bm{U}\in St(d,r),&\ \forall \bm{\xi} \in \mathbb{R}^{d\times r}, \norm{\bm{\xi}}_F\le M_3\\
    \end{aligned}
    $$
\end{enumerate}
\end{definition}
Figure \ref{fig:grillus} is an illustration of the Stiefel manifold, tangent space, and generalized retraction.

\begin{figure}
\centering
  \includegraphics[width=0.8\linewidth]{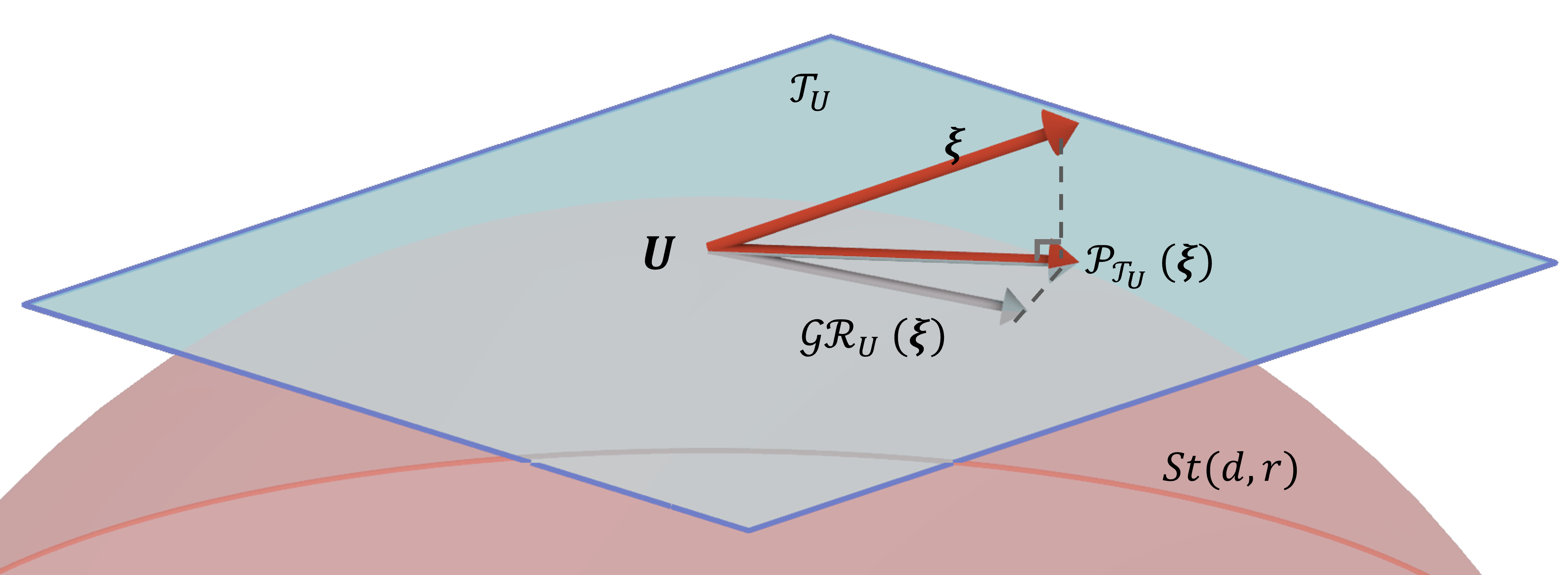}
  \caption{An illustration of the Stiefel manifold, tangent space, and generalized retraction. The red surface represents the Stiefel manifold. The blue plane represents the tangent space at $\bm{U}\in St(d,r)$. $\bm{\xi}$ is a general $d$ by $r$ matrix that represents the update direction. $\Ptangent{\bm{U}}{\bm{\xi}}$ projects $\bm{\xi}$ to the tangent space on $\bm{U}$. Generalized retraction $\grof{\bm{U}}{\bm{\xi}}$ maps $\bm{U}+\bm{\xi}$ back to the Stiefel manifold.}
  \label{fig:grillus}
\end{figure}

Notice that the definition of a generalized retraction extends the definition of retraction in literature \citep{optimizationonmatrixmanifold}. Retraction is usually defined as a mapping from the tangent bundle $\mathcal{T}_{\bm{U}}$ to the Stiefel manifold $St(d,r)$ \citep{stiefelgeometry}. However, a generalized retraction is a mapping from a general $\mathbb{R}^{d\times r}$ to the Stiefel manifold $St(d,r)$. This extension allows us to directly apply the generalized retraction to any matrix, eliminating the need to project it to the tangent space beforehand.

Property \ref{cons:preservecolumnspace} requires that a generalized retraction preserves column spaces. This property is indispensable in our algorithm development as we use it to ensure the orthogonality of global and local PCs. The second property requires that $\grof{\bm{U}}{\bm{\xi}}$ be close to the projection to the tangent space $\bm{U}+\mathcal{P}_{\mathcal{T}_{\bm{U}}}(\bm{\xi})$. In the special case of $\bm{\xi}\in \mathcal{T}_{\bm{U}}$, property \ref{cons:approximate} reduces to $\norm{\grof{\bm{U}}{\bm{\xi}}-(\bm{U}+\bm{\xi})}_F\le M_1\norm{\bm{\xi}}^2_F$, which coincides with the definition of retraction in literature \citep{mingyiconsensus}. When the norm of $\bm{\xi}$ is small, the requirement essentially implies that the difference between a generalized retraction and the projection to a tangent space is a higher-order term.

Though Definition \ref{def:generalizedretraction} looks demanding, we can show that there are several available choices for a generalized retraction. 
\begin{proposition}
Polar projection is defined as:
\begin{equation}
\label{eqn:polarprojectiondef}
\polarof{\bm{U}}{\bm{\xi}}=(\bm{U}+\bm{\xi})\left(\bm{I}+\bm{\xi}^T\bm{U}+\bm{U}^T\bm{\xi}+\bm{\xi}^T\bm{\xi}\right)^{-\frac{1}{2}}
\end{equation}
is a generalized retraction. The computation complexity is $O(dr^2+r^3)$.
\end{proposition}
\eqref{eqn:polarprojectiondef} is consistent with the definition \eqref{eqn:defofpolar}, though the notations are slightly different. Notice that polar projection can be equivalently calculated by the SVD of $\bm{U}+\bm{\xi}$ \citep{mmalgorithm}. We relegate the proof and the implementation details to Appendix \ref{ap:polarisgr}. As discussed, an interesting property of the polar projection is that it is equivalent to the projection of $\bm{U}+\bm{\xi}$ onto the Stiefel manifold:
\begin{equation}
\label{eqn:polarisprojection}
\polarof{\bm{U}}{\bm{\xi}}=\arg\min_{\bm{V}\in St(d,r)}\norm{\bm{U}+\bm{\xi}-\bm{V}}_F
\end{equation}
The proof of \eqref{eqn:polarisprojection} can be found in \citet{berkeleyproof}.

QR decomposition is another influential algorithm in numerical linear algebra. It also satisfies the requirements of a generalized retraction.
\begin{proposition}
For a matrix $\bm{U}+\bm{\xi}\in \mathbb{R}^{d\times r}$, QR decomposition finds an orthogonal matrix $\bm{Q}\in St(d,r) $ and a upper triangular matrix $\bm{R}\in\mathbb{R}^{r\times r}$, such that $\bm{Q}\bm{R}=\bm{U}+\bm{\xi}$. As such, a QR retraction is defined as: 
$$
\qrof{\bm{U}}{\bm{\xi}}=\bm{Q}
$$
is a generalized retraction. The computation complexity is $O(dr^2)$.
\end{proposition}
We relegate the proof to Appendix \ref{ap:qrisgr}.  

In all of our experiments, we choose the generalized retraction as a polar decomposition.

\subsection{\name: The algorithm}
Now, we are ready to introduce the personalized PCA algorithm, \name. Recall that our algorithm is designed to be federated and requires multiple communication rounds between a client and some central server/entity that orchestrates the collaborative learning process. Suppose at communication round $\tau$, each client has feasible global components $\bm{U}_{\tau}$ and local components $\bm{V}_{(i),\tau}$, i.e., $[\bm{U}_{\tau},\bm{V}_{(i),\tau}]\in St(d,r_1+r_{2,(i)})$. Then client $i$ calculates the gradient of objective $f_i$ defined in \eqref{eqn:defoffi}:
$$
\left\{\begin{aligned}
&\nabla_{\bm{U}}f_i(\bm{U}_{\tau},\bm{V}_{(i),\tau})=\bm{S}_{(i)}\bm{U}_{\tau}\\
&\nabla_{\bm{V}_{(i)}}f_i(\bm{U}_{\tau},\bm{V}_{(i),\tau})=\bm{S}_{(i)}\bm{V}_{(i),\tau}
\end{aligned}\right.
$$

Since the gradient direction generally does not align with the tangent space of $\mathcal{T}_{[\bm{U}_{\tau},\bm{V}_{(i),\tau}]}$, simple gradient ascent will move $[\bm{U}_{\tau},\bm{V}_{(i),\tau}]$ out of $St(d,r_1+r_{2,(i)})$. To ensure the iterates move along the manifold, Stiefel optimization first projects the gradient to the tangent space:
\begin{equation}
\label{eqn:parallelgradient}
\matG_{(i),\tau}=\mathcal{P}_{\mathcal{T}_{\left[\bm{U}_{\tau},\bm{V}_{(i),\tau}\right]}}\left(\bm{S}_{(i)}\left[\bm{U}_{\tau},\bm{V}_{(i),\tau}\right]\right)    
\end{equation}

In literature, $\matG_{(i),\tau}$ is usually referred to as the parallel gradient on the manifold \citep{stiefelgeometry}. We shall note that $\matG_{(i),\tau}$ defined above is a $d$ by $r_1+r_{2,(i)}$ matrix.

Clients then update global and local PCs in the direction of the parallel gradient $\matG_{(i),\tau}$. As there is a small difference between the Stiefel manifold and the tangent space, the updated PCs are still not orthonormalized. Therefore, we use a generalized retraction to retract the updated local components to the Stiefel manifold. We use $\bm{V}_{(i),\tau+\frac{1}{2}}$ to denote the retracted matrix. For the global components, clients first send them to a server. The server then takes the average and uses a generalized retraction to map the average to $St(d,r)$. The updated global PC matrix is denoted as $\bm{U}_{\tau+1}$.

A major challenge then arises: after the server averages the global PCs, $\bm{U}_{\tau+1}$ is not orthogonal to $\bm{V}_{(i),\tau+\frac{1}{2}}$ anymore, i.e., $\bm{U}_{\tau+1}^T\bm{V}_{(i),\tau+\frac{1}{2}}\neq 0$ in general. Thus $\bm{U}_{\tau+1}$ and $\bm{V}_{(i),\tau+\frac{1}{2}}$'s become infeasible, and the algorithm based on St-GD cannot proceed. One can verify that $\bm{U}_{\tau}^T\bm{V}_{(i),\tau+\frac{1}{2}}=O(\eta_{\tau})$, which has the same order as the parallel gradient update. Thus, we cannot resolve the infeasibility issue by decreasing stepsize. This is a fundamental limitation of a simple route that uses distributed St-GD.

Can we resolve the challenge by enforcing the orthogonality between global and local PC estimates? Inspired by Gram-Schmit orthonormalization, we introduce a correction step on the local PCs. We calculate the projection of $\bm{V}_{(i),\tau+\frac{1}{2}}$ onto the column space of $\bm{U}_{\tau+1}$, and subtract the projected matrix from $\bm{V}_{(i),\tau+\frac{1}{2}}$. The resulting (deflated) matrix is orthogonal to $\bm{U}_{\tau+1}$. Then, we use a generalized retraction to map the subtracted matrix to the Stiefel manifold. Remember that one key property of a generalized retraction is that it preserves the column space; the retracted matrix is thus still orthogonal to $\bm{U}_{\tau+1}$. We use $\bm{V}_{(i),\tau+1}$ to denote the retracted matrix. Now $\bm{U}_{\tau+1}$ and  $\bm{V}_{(i),\tau+1}$ are feasible, and the updates can repeat over multiple communication rounds until convergence. The pseudocode is summarized in Algorithm \ref{alg:consensus}.
\begin{algorithm}
\caption{\name by St-GD}
\label{alg:consensus}
\begin{algorithmic}
\STATE Input client covariance matrices $\{\bm{S}_{(i)}\}_{i=1}^N$, stepsize $\eta_\tau$
\STATE Initialize $\bm{U}_1$, and $\bm{V}_{(1),\frac{1}{2}},\cdots,\bm{V}_{(N),\frac{1}{2}}$.
\FOR{Communication rounds $\tau=1,...,R$}
\FOR{Client $i=1,\cdots,N$}
\STATE $\bm{V}_{(i),\tau}=\grof{\bm{V}_{(i),\tau-\frac{1}{2}}}{-\bm{U}_{\tau}\bm{U}_{\tau}^T\bm{V}_{(i),\tau-\frac{1}{2}}}$ $\hfill\color{cyan} \texttt{// Deflate then retract}$
\STATE Choice $1$:
\bindent
\STATE Calculate $\matG_{(i),\tau}=\Ptangent{\left[\bm{U}_{\tau},\bm{V}_{(i),\tau}\right]} {\bm{S}_{(i)}\left[\bm{U}_{\tau},\bm{V}_{(i),\tau}\right]}$ $\hfill\color{cyan} \texttt{// Tangent projection}$
\STATE Update $ \bm{U}_{(i),\tau+1}=\bm{U}_{\tau}+\eta_{\tau}(\matG_{(i),\tau})_{1:d,1:r_1}$ $\hfill\color{cyan} \texttt{// Gradient ascent}$
\STATE Update $ \bm{V}_{(i),\tau+\frac{1}{2}}=\grof{\bm{V}_{(i),\tau}}{ \eta_{\tau}(\matG_{(i),\tau})_{1:d,(r_1+1):(r_1+r_{2,(i)})}}$ $\hfill\color{cyan} \texttt{// Retract}$
\eindent
\STATE Choice $2$:
\bindent
\STATE Update $[\bm{U}_{(i),\tau+1},\bm{V}_{(i),\tau+\frac{1}{2}} ]=\polarof{\left[\bm{U}_{\tau},\bm{V}_{(i),\tau}\right]}{\eta_{\tau}\bm{S}_{(i)}\left[\bm{U}_{\tau},\bm{V}_{(i),\tau}\right]}$ \\$ \hfill\color{cyan} \texttt{// Retract after gradient ascent}$
\eindent
\STATE Send $\bm{U}_{(i),\tau+1}$ to the server. $ \hfill\color{cyan} \texttt{// Share global PCs}$
\ENDFOR
\STATE Server calculates $\bm{U}_{\tau+1}=\grof{\bm{U}_{\tau}}{\frac{1}{N}\sum_{i=1}^N\bm{U}_{(i),\tau+1}-\bm{U}_{\tau}}$ $\hfill\color{cyan} \texttt{// Average then retract}$
\STATE Server broadcasts $\bm{U}_{\tau+1}$ 
\ENDFOR
\STATE Return principal components $\bm{U}_{R}$ and $\bm{V}_{(i),R}$'s.
\end{algorithmic}
\end{algorithm}

The first line in the client loop $\bm{V}_{(i),\tau}=\grof{\bm{V}_{(i),\tau-\frac{1}{2}}}{-\bm{U}_{\tau}\bm{U}_{\tau}^T\bm{V}_{(i),\tau-\frac{1}{2}}}$ represents the correction on the local PC matrix. Regardless of whether $\bm{U}_{\tau}$ and $\bm{V}_{(i),\tau-\frac{1}{2}}$ are orthogonal,  $\bm{U}_{\tau}$ and $\bm{V}_{(i),\tau}$ are always feasible: $[\bm{U}_{\tau}, \bm{V}_{(i),\tau}]\in St(d,r_1+r_{2,(i)})$. Then each client applies standard St-GD (choice 1) or a variant of St-GD (choice 2) to update $\bm{U}_{(i),\tau+1}$ and $\bm{V}_{(i),\tau+\frac{1}{2}}$ simultaneously. The updated global PCs are sent to the server. The server takes the simple average of all received global PCs and retracts the average to $St(d,r_1)$. The obtained $\bm{U}_{\tau+1}$ is then broadcasted back to the clients and becomes the starting point of the next iteration. The algorithm repeats for a certain number of communication rounds.

In Algorithm \ref{alg:consensus}, we introduce two algorithmic choices on the client side. For choice $1$, clients perform standard St-GD: first project the updates to the tangent space, then retract them to the Stiefel manifold. For choice $2$, clients use polar projection to replace the St-GD. This update rule is inspired by the Minorization-Maximization algorithm \citep{mmalgorithm}. Remember that by \eqref{eqn:polarisprojection}, polar projection acts as a projection into the nonlinear Stiefel manifold. Hence, it is close to the composition of the projection onto the tangent space and the retraction from the tangent space onto the nonlinear manifold. We propose two choices to enrich practitioners' toolkits as they have similar performances in most of our case studies. We focus on choice $1$ in our theoretical analysis. However, it is observed in the video segmentation task that choice $2$ allows us to use larger stepsizes, thus converging faster. Hence, we leave it to practitioners' discretion to make specific algorithmic choices. 

\revise{In general, the computation complexity per communication round at one client is $O(d^2)$. To see that, we can analyze the update of Algorithm \ref{alg:consensus}. One iteration only involves matrix multiplication and generalized retractions. The computation complexity of matrix multiplication $\matS_{(i)}\matU_{\tau}$ is $O(d^2r_1)$. The complexity of the tangent projection step is similar. When the rank $r_1$ and $r_{2,(i)}$ is far smaller than data dimension $d$, the computation complexity for generalized retractions is only $O(d)$. Thus, the per-iteration computation complexity is $O(d^2)$. It is worth noting that the complexity can be further reduced to $O(d)$ if the covariance matrix $\matS_{(i)}$ is known to be low rank. More specifically, when $\matS_{(i)}$ has a low-rank Cholesky decomposition $\matS_{(i)}=\matY_{(i)}\matY_{(i)}^T$, where $\matY_{(i)}\in \mathbb{R}^{d\times n_{(i)}}$ is a low-rank matrix $n_{(i)}\ll d$, the computation cost of matrix multiplication $\matS_{(i)}\matU_{\tau}=\matY_{(i)}\matY_{(i)}^T\matU_{\tau}$ is reduced to $O(d n_{(i)}r_1)$. As $n_{(i)}$ and $r_1$ is far smaller than $d$, this becomes $O(d)$. Hence the per-iteration computation complexity is only $O(d)$.}

\section{Does Algorithm \ref{alg:consensus} Recover the Local and Global Truth?}
\label{sec:convergence}
Though the development of Algorithm \ref{alg:consensus} is intuitive, it is important to understand whether it converges and, if so, what kind of solution it can recover. In this section, we will analyze the convergence of Algorithm  \ref{alg:consensus} and show that, in general, Algorithm \ref{alg:consensus} converges into stationary points of the objective. In addition, when the local and global components are initialized properly, Algorithm \ref{alg:consensus} will converge into the global optimal solutions linearly, and the result exactly recovers the true local and global PCs. 

\subsection{Global convergence}
\label{sec:sublinearconvergence}
To analyze the convergence, we make an additional assumption that the largest eigenvalues of the sample covariance matrices $\bm{S}_{(i)}$'s are upper bounded:
\begin{assumption}
\label{ass:snormupperbound}
We assume that the operator norms of $\bm{S}_{(i)}$'s are upper bounded by constants $G_{(i),op}$:
\begin{equation}
\label{eqn:opnormupperbound}
\norm{\bm{S}_{(i)}}_{op}\le G_{(i),op}
\end{equation}
and the Frobenius norms of $\bm{S}_{(i)}$'s are upper bounded by constants $G_{(i),F}$:
\begin{equation}
\label{eqn:fnormupperbound}
\norm{\bm{S}_{(i)}}_{F}\le G_{(i),F}
\end{equation}
We use $G_{max,op}$ to denote $\max_iG_{(i),op}$, and $G_{max,F}$ to denote $\max_iG_{(i),F}$.
\end{assumption}
Assumption \ref{ass:snormupperbound} is a common assumption in optimization literature, as it essentially assumes the objective is Lipschitz continuous. Also, if we assume the data are independently generated and follow a sub-Gaussian distribution, Assumption \ref{ass:snormupperbound} will hold with high probability \citep{wainwrightbook}.

The first order condition (KKT condition) to problem \eqref{eqn:model} is that for the parallel gradients defined in \eqref{eqn:parallelgradient}, the local parts are zero on each client, and the average of the global parts is zero:
\begin{equation}
\label{eqn:kktcondition}
\left\{\begin{aligned}
&\left(\matG_{(i)}\right)_{1:d,(r_1+1):(r_1+r_{2,(i)})}=0,\quad \forall i\in\{1,2,,\cdots,N\}\\
&\frac{1}{N}\sum_{i=1}^N\left(\matG_{(i)}\right)_{1:d,1:(r_1+1)}=0\\
\end{aligned}\right.
\end{equation}
The proof of KKT conditions \eqref{eqn:kktcondition} is in Appendix \ref{ap:proofforkkt}. It is clear from Algorithm \ref{alg:consensus} that when \eqref{eqn:kktcondition} is satisfied, the global and local PC updates will be stationary. Thus, \eqref{eqn:kktcondition} essentially describes the stationary points of \eqref{eqn:problem}.

On non-stationary points, \eqref{eqn:kktcondition} generally does not hold. The below theorem provides an upper bound on the magnitude of the violations to conditions \eqref{eqn:kktcondition}. As the violations decrease to zero when the number of communication approaches infinity, the theorem shows that Algorithm \ref{alg:consensus} will converge into the KKT points. We use $r$ to denote maximum rank $r=\max\{r_1,r_{2,(1)},\cdots,r_{2,(N)}\}$.
\begin{theorem}
\label{thm:sublinearconvergence}
Under Assumption \ref{ass:snormupperbound}, if we choose a constant stepsize $\eta_{\tau}=\eta_1= O(\frac{1}{G_{max,op}\sqrt{r}})$, then Algorithm \ref{alg:consensus} with choice $1$ will converge into stationary points:
$$
\begin{aligned}
&\min_{\tau\in \{1,...R\}}\left[\norm{\sum_{i=1}^N\left(\bm{I}-\bm{P}_{\bm{U}_{\tau}}-\bm{P}_{\bm{V}_{(i), \tau}}\right)\bm{S}_{(i)}\bm{U}_{\tau}}^2+\sum_{i=1}^N\norm{\left(\bm{I}-\bm{P}_{\bm{U}_{\tau}}-\bm{P}_{\bm{V}_{(i), \tau}}\right)\sum_{i=1}^N\bm{S}_{(i)}\bm{V}_{(i), \tau}}^2\right]\\
&= O\left(\frac{1}{R}\right)
\end{aligned}
$$
\end{theorem}

Despite the nonconvex constraints in \eqref{eqn:problem}, Algorithm \ref{alg:consensus} provably converges to stationary points, regardless of initial conditions. The $\frac{1}{R}$ convergence rate is comparable to the rate in literature \citep{mingyiconsensus}. 

Our algorithm handles global and local PCs at the same time and attains stationary points of both components. In the following section, we will show the proof sketch of Theorem \ref{thm:sublinearconvergence}. The complete proof is relegated to Appendix \ref{ap:proofforsublinearconvergence}.

\subsubsection{Proof sketch for Theorem \ref{thm:sublinearconvergence} and key lemmas}
\label{sec:proofsketchsublinearconvergence}
As discussed before, one major difficulty in analyzing Algorithm \ref{alg:consensus} lies in the correction step. The correction step changes local PCs by $O(\eta_{\tau})$, which is comparable to that in the gradient ascent step. Therefore, a na\"{\i}ve treatment to the correction step will generate a large error term that cannot be bounded.

To bypass the issue, we exploit one nice structure in objective \eqref{eqn:defoffi}: $f_i(\bm{U},\bm{V}_{(i)})$ is dependent only on the subspace spanned by the concatenated matrix $[\bm{U},\bm{V}_{(i)}]$. Therefore one can make adjustments on $col(\bm{U})$ and $col(\bm{V}_{(i)})$ without changing the objective value, as long as $col([\bm{U},\bm{V}_{(i)}])$ are the same.

One major technical novelty of our work is to introduce Lyapunov functions that take this key property into consideration. We define the two following Lyapunov functions:
\begin{equation}
\label{eqn:defofln1}
\mathcal{L}_{(i),1}(\bm{U},\bm{V}) = -\frac{1}{2}\tr{\bm{U}^T\left(\bm{I}-\bm{P}_{\bm{V}}\right)\bm{S}_{(i)}\left(\bm{I}-\bm{P}_{\bm{V}}\right)\bm{U}}
\end{equation}
and, 
\begin{equation}
\label{eqn:defofln2}
\mathcal{L}_{(i),2}(\bm{U},\bm{V}) = -\frac{1}{2}\tr{\bm{V}^T\bm{S}_{(i)}\bm{V}}
\end{equation}
It's easy to see that when $\bm{V}^T\bm{U}=0$, we have:
$$
\mathcal{L}_{(i),1}(\bm{U},\bm{V})+\mathcal{L}_{(i),2}(\bm{U},\bm{V})=-\frac{1}{2}\tr{\bm{U}^T\bm{S}_{(i)}\bm{U}}-\frac{1}{2}\tr{\bm{V}^T\bm{S}_{(i)}\bm{V}}=-f_n(\bm{U},\bm{V})
$$
At each communication step $\tau$, global and local components are indeed orthogonal $\bm{U}_{\tau}^T\bm{V}_{(i),\tau}=0$, thus $\mathcal{L}_{(i),1}(\bm{U}_{\tau},\bm{V}_{(i),\tau})+\mathcal{L}_{(i),2}(\bm{U}_{\tau},\bm{V}_{(i),\tau})=-f_n(\bm{U}_{\tau},\bm{V}_{(i),\tau})$. 

$\mathcal{L}_{(i),1}$ explicitly encodes the orthogonality constraint into the objective. Such design enables convenient handling of the correction step: we can prove that the correction step on $\bm{V}$ changes  $\mathcal{L}_{(i),1}+\mathcal{L}_{(i),2}$ only by $O(\eta_{\tau}^2)$. Therefore only the descent step can change $\mathcal{L}_{(i),1}+\mathcal{L}_{(i),2}$ by $O(\eta_{\tau})$. Thus, the change of Lyapunov functions is dominated by the update from the parallel gradient. By calculating the update of $\bm{U}$ and $\{\bm{V}_{(i)}\}$ in each communication round, we can have the following informal version of the sufficient descent lemma:

\begin{lemma}
\label{lm:informalsuffcientdecrease}
(Informal) When we choose a constant stepsize $\eta_{\tau}=\eta= O\left(\frac{1}{G_{max,op}\sqrt{r}}\right)$, and  $\bm{U}_{\tau}$ and $\bm{V}_{(i),\tau}$ satisfy the orthogonality condition $\bm{U}_{\tau}^T\bm{V}_{(i),\tau}=0$, we have:
\begin{equation}
\begin{aligned}
&\left\langle \sum_{i=1}^N\nabla_{\bm{U}}\mathcal{L}_{(i),1}(\bm{U}_{\tau},\bm{V}_{(i),\tau}) ,\bm{U}_{\tau+1}-\bm{U}_{\tau}\right\rangle\\
&+\sum_{i=1}^N\left\langle \nabla_{\bm{V}_{(i)}}\mathcal{L}_{(i),1}(\bm{U}_{\tau},\bm{V}_{(i),\tau})+\nabla_{\bm{V}_{(i)}}\mathcal{L}_{(i),2}(\bm{U}_{\tau},\bm{V}_{(i),\tau}) ,\bm{V}_{(i),\tau+1}-\bm{V}_{(i),\tau}\right\rangle\\
&\le -\eta \left(\frac{1}{N}\norm{\sum_{i=1}^N\left(\bm{I}-\bm{P}_{\bm{U}_{\tau}}-\bm{P}_{\bm{V}_{(i), \tau}}\right)\bm{S}_{(i)}\bm{U}_{\tau}}_F^2+\sum_{i=1}^N\norm{\left(\bm{I}-\bm{P}_{\bm{U}_{\tau}}-\bm{P}_{\bm{V}_{(i), \tau}}\right)\sum_{i=1}^N\bm{S}_{(i)}\bm{V}_{(i), \tau}}_F^2\right)\\
&+O(\eta^2)
\end{aligned}   
\end{equation}
\end{lemma}
When $\eta$ is small, the $O(\eta)$ terms will dominate $O(\eta^2)$ terms. Thus, Lemma \ref{lm:informalsuffcientdecrease} essentially shows that in Algorithm \ref{alg:consensus}, the change of Lyapunov functions is negative semidefinite in one communication round. With the sufficient decrease property, standard analysis on first-order optimization yields a  $O\left(\frac{1}{R}\right)$ convergence rate.

Formal proofs of Theorem \ref{thm:sublinearconvergence} and Lemma \ref{lm:informalsuffcientdecrease} can be found in Appendix \ref{ap:proofforsublinearconvergence}.

\subsection{Local convergence}
Theorem \ref{thm:sublinearconvergence} only shows that Algorithm \ref{alg:consensus} converges into stationary points but does not provide further information about the property of the final solution. In problems like feature extraction, we want to know whether the stationary point is a globally optimal solution or whether it corresponds to the true PCs.

To this end, we analyze the convergence of global and local PCs. The convergence depends on a Polyak-Lojasiewicz style condition. Similar to Section \ref{sec:sublinearconvergence}, we will introduce another assumption about the eigenvalue distribution of the sample covariance matrix. Without loss of generality, in this section, we assume $r_1=r_{2,(1)}=\cdots=r_{2,(N)}=r$.
\begin{assumption}
\label{ass:noiseless}
(Covariance matrix eigenvalue lower bound) We further assume that the population covariance $\matSigma_{(i)}$ can be entirely explained by $\matPi_g+\matPi_{(i)}$, i.e., $\matSigma_{(i)}\left(\matPi_g+\matPi_{(i)}\right)=\matSigma_{(i)}$, and that the minimum nonzero eigenvalues of $\matSigma_{(i)}$ is lower bounded by a constant $\mu>0$:
\begin{equation}
\mu\left(\matPi_g+\matPi_{(i)}\right)\preceq\matSigma_{(i)}
\end{equation}
where $\matPi_g$ and $\matPi_{(i)}$ are rank-$r$ projection matrices.
\end{assumption}

Assumption \ref{ass:noiseless} assumes that data covariance can be decomposed as noiseless global and local parts with rank $r$. The noiseless assumption of the population covariance matrices is the standard assumption in the local convergence analysis of many PCA algorithms (e.g., \citep{proofonkpca}). 

The following theorem shows that if Algorithm \ref{alg:consensus} is initialized within the attractive basin of the global optimum, the iterates will converge to the global optimal solution linearly.
\begin{theorem}
\label{thm:linearconvergence}(Informal) Under assumptions \ref{ass:identifiability}, \ref{ass:snormupperbound}, and \ref{ass:noiseless}, if the difference between the population and sample covariance is small, when we initialize close to the global optimum, and choose a constant stepsize $\eta_{\tau} = \eta= O\left(\frac{1}{G_{op,max}\sqrt{r}}\right)$, then Algorithm \ref{alg:consensus} with choice 1 will converge into the global optimum:
$$
f(\hmatU,\{\hmatV_{(i)}\}) - f(\bm{U}_{R},\{\bm{V}_{(i),R}\})  = O\left(\left(1-\eta \frac{\mu\theta}{32}\right)^{R}\right) 
$$
where $\{\hmatU,\{\hmatV_{(i)}\} \}$ is one set of optimal solutions to problem \eqref{eqn:problem}.

Furthermore, we can recover the exact global optimal solutions:
$$
\norm{\Pj{\bm{U}_{R}}-\Pj{\hmatU_g}}_F^2+\frac{1}{N}\sum_{i=1}^N\norm{\Pj{\bm{V}_{(i),R}}-\Pj{\hmatV_{(i)}}}_F^2= O\left(\left(1-\eta \frac{\mu\theta}{32}\right)^{R}\right) 
$$
\end{theorem}
It is worthwhile to point out that in Theorem \ref{thm:linearconvergence}, the convergence is faster for a larger misalignment parameter $\theta$. This is intuitively understandable since when local eigenspaces are more heterogeneous, it is easier to identify different eigenspaces. On the other hand, if all the local eigenspaces are similar, it is difficult to distinguish local PCs from global PCs; thus, the convergence is slower. \emph{This result is in striking contrast to standard federated learning (e.g., \citet{convergencefedavg,fedprox}), where data heterogeneity leads to slower convergence}. We will verify this finding in Section \ref{sec:experiment}. A formal version of Theorem \ref{thm:linearconvergence} and its proof is relegated to the Appendix \ref{ap:proofforlinearconvergence}.

\revise{With the statistical error bound provided by Theorem \ref{thm:statisticalerror} and the convergence guarantee from Theorem \ref{thm:linearconvergence}, we can derive the following corollary.}
\revise{
\begin{corollary}
Under the same assumptions as Theorem \ref{thm:statisticalerror} and Theorem \ref{thm:linearconvergence}, after $t=\Omega\left(\frac{\sqrt{r}\gmop}{\mu\theta}\log\frac{1}{\varepsilon_{stats}}\right)$ communication rounds, we can obtain estimates of global and local PCs that satisfy,
\begin{align*}
 \norm{\Pj{\bm{U}_{t}}-\matPi_g}_F^2+\frac{1}{N}\sum_{i=1}^N\norm{\Pj{\bm{V}_{(i),t}}-\matPi_{(i)}}_F^2= O\left(\varepsilon_{stats}\right)    
\end{align*}
where $\varepsilon_{stats}$ is the statistical error $\varepsilon_{stats}=\frac{1}{\theta\delta^2} \sigma^4 C^2\frac{d}{N}\sum_{i=1}^N\frac{1}{n_i}$.
\end{corollary}
\begin{proof}
By the triangle inequality, we know
\begin{align*}
&\norm{\Pj{\bm{U}_{t}}-\matPi_g}_F^2+\frac{1}{N}\sum_{i=1}^N\norm{\Pj{\bm{V}_{(i),t}}-\matPi_{(i)}}_F^2\\
&\le 2\norm{\Pj{\bm{U}_{t}}-\Pj{\hmatU}}_F^2+\frac{2}{N}\sum_{i=1}^N\norm{\Pj{\bm{V}_{(i),t}}-\Pj{\hmatVit}}_F^2\\
&+2\norm{\matPi_g-\Pj{\hmatU}}_F^2+\frac{2}{N}\sum_{i=1}^N\norm{\matPi_{(i)}-\Pj{\hmatVit}}_F^2
\end{align*}
The first term is bounded by Theorem \ref{thm:linearconvergence}, and the second term is bounded by Theorem \ref{thm:statisticalerror}
\end{proof}
}
\subsubsection{Proof sketch of Theorem \ref{thm:linearconvergence} and key lemmas}
To prove the exponential convergence in Theorem \ref{thm:linearconvergence}, we need a stronger version of the sufficient decrease inequality than Lemma \ref{lm:informalsuffcientdecrease}. We should show that, in each communication round, the change in the Lyapunov functions is negative definite. This requires a careful analysis of the geometry of objective \eqref{eqn:problem} around the global optimum $\Pj{\hmatU}$ and $\{\Pj{\hmatVit}\}$.  

The key result is the Polyak-Lojasiewicz (PL) inequality.
\begin{lemma}
\label{lm:mainpaperplinequality}
(Polyak-Lojasiewicz inequality)
Under the same conditions as Theorem \ref{thm:linearconvergence}, we have
\begin{align*}
&\frac{1}{N}\norm{\sum_{i=1}^N\left(\bm{I}-\bm{P}_{\bm{U}_{\tau}}-\bm{P}_{\bm{V}_{(i), \tau}}\right)\bm{S}_{(i)}\bm{U}_{\tau}}_F^2+\sum_{i=1}^N\norm{\left(\bm{I}-\bm{P}_{\bm{U}_{\tau}}-\bm{P}_{\bm{V}_{(i), \tau}}\right)\sum_{i=1}^N\bm{S}_{(i)}\bm{V}_{(i), \tau}}_F^2\\
    &\ge \frac{\theta\mu}{32}\left(f(\hmatU,\{\hmatV_{(i)}\}) - f(\bm{U}_{R},\{\bm{V}_{(i),R}\})\right)
\end{align*}
\end{lemma}
 The PL inequality shows that the norm of the parallel gradient is lower bounded, a constant fraction of the optimality gap. It certifies a nice geometric property in objective \eqref{eqn:problem} so that each step of gradient descent can make significant progress. By combining the PL inequality with Lemma \ref{lm:suffcientdecrease}, we can easily prove Theorem \ref{thm:linearconvergence}.

One of our major technical contributions is to establish the PL inequality for the nonconvex problem \eqref{eqn:problem}.  We analyze the local geometry of the problem with the help of one special set of optimal solutions $\{\hmatUt,\{\hmatVit\}\}$. We show that this set of optimal solutions is close to the current iterate $\{\matUt,\{\matVit\}\}$. Also, the difference $\{\matUt-\hmatUt,\{\matVit-\hmatVit\}\}$ is aligned with the parallel gradient. As a result, the parallel gradient can direct the updates to the optimal solutions.

The full proof of Lemma \ref{lm:mainpaperplinequality} and Theorem \ref{thm:linearconvergence} is relegated to Appendix \ref{ap:proofforlinearconvergence}. 

\section{Numerical Experiments}
\label{sec:experiment}

This section tests our model on a set of datasets across different applications. We start in Section \ref{sec:proofofconcepts} with a proof of concept study using a synthetic dataset to verify theoretical findings in Sections \ref{subsec:statisticalefficiency} and \ref{sec:convergence}. 

We also \revise{discuss the effects of overparametrization} and show an interesting application of \name in federated client clustering using local PCs. In Section \ref{sec:expcomparisonwithrpca}, we provide an illustrative example in comparison with \texttt{Robust PCA} to shed light on the end goal of our model. Next, we apply \name to a real-life heterogeneous distributed dataset FEMNIST and CIFAR10 to show \name's advantages in finding better features in Section \ref{sec:expfemnist}. Finally, we demonstrate how \name can separate shared and unique features in video and language data in Section \ref{sec:applicationsbeyondfederation}.

We note that from Theorem \ref{thm:linearconvergence}, a suitable initialization is needed for the best performance of \name. We thus employ the standard one-communication round distributed PCA algorithm proposed in \citet{dispca2} as the initialization of global PCs in Algorithm \ref{alg:consensus}, unless specified otherwise. Local PCs are always randomly initialized. In this section we set $r_{2,(1)}=r_{2,(2)}=\cdots=r_{2,(N)}=r_2$.

\subsection{Proof of concept on synthetic datasets}
\label{sec:proofofconcepts}
We generate data from model \eqref{eqn:model}. The  $\bm{u}_q$'s and $\bm{v}_q$'s are set to be orthogonal components. After obtaining $\bm{u}_q$'s and $\bm{v}_q$'s, we sample the score coefficients $\phi_{(i),q}$'s and $\varphi_{(i),q}$'s from i.i.d. Gaussian distributions. 
Noise $\bm{\epsilon}_{(i)}$ are also sampled from i.i.d. Gaussian distributions.

Under this setting, multiple aspects are tested: in Section \ref{sec:expaboutconvergence}, we revisit the example in Figure \ref{fig:toyexampleintro} and examine the convergence behavior of \name numerically. In Sections \ref{sec:experroronn}, \ref{sec:experrorond}, and \ref{sec:experroronN}, we demonstrate how the statistical errors change with the (i) number of observations $n$, (ii) data dimension $d$, and (iii) number of clients $N$, and compare the results with our theory. In Section \ref{sec:expsharedknowledge}, we show that in \name, clients benefit from knowledge sharing to improve their PC estimates. \revise{Then we investigate the numerical performance of \name when $r_1$ and $r_2$ are overparametrized in Section \ref{exp:overparametrize}.} Finally, in Section \ref{exp:expclientclustering}, we describe a method that exploits the estimated local PCs for client clustering. 

\subsubsection{Convergence of \name}
\label{sec:expaboutconvergence}
We first analyze the convergence of \name. Theorem \ref{thm:linearconvergence} predicts that (i) \name has local linear convergence, and (ii) a larger $\theta$ can expedite convergence. To verify the two theoretical results, we run \name on a group of synthetic data. We set $N=2$, $d=3$ and $n_{(i)}=1000$. Each client has exactly one global $\bm{u}_1$ and one local component $\bm{v}_{(i),1}$. After setting global PC $\bm{u}_1$ and local PCs $\bm{v}_{(1),1}$ and $\bm{v}_{(2),1}$, we generate the data according to the model \eqref{eqn:model} where coefficients $\phi_{(i),q}$ and $\varphi_{(i),q}$ are randomly sampled from Gaussian distributions. By changing the direction of local PCs $\bm{v}_{(1),1}$ and $\bm{v}_{(2),1}$, we can modify $\theta$:
$$
\theta = \sin^2 \left(\frac{1}{2}\arccos (\bm{v}_{(1),1}^T\bm{v}_{(2),1})\right)
$$

Figure \ref{fig:toyexampleintro}, shown in the introduction, is an instance of this analysis where $\theta=0.127$.

To see the $\theta$'s effect on convergence, we generate the data with $\theta$ ranging from $0$ to $0.3$. In this experiment, we initialize global and local PCs to be random Gaussian vectors. We run each experiment with the same stepsize $\eta=0.1$ but from $10$ different random initializations and collect the reconstruction error in each communication round $\tau$. The reconstruction error is defined as the objective in \eqref{eqn:naiveproblemformulation} divided by the number of observations $n_{(i)}$:
\begin{equation}
\label{eqn:reconstructionloss}
\text{Reconstruction error} = \frac{1}{N}\sum_{i-1}^N \frac{1}{n_{(i)}}\norm{\bm{Y}_{(i)}-\left(\Pj{\bm{U}}+\Pj{\bm{V}_{(i)}}\right)\bm{Y}_{(i)}}_F^2
\end{equation}
Results are shown in Figure \ref{fig:logerrortotheta}. From Figure \ref{fig:logerrortotheta}(left), we can see that \name indeed enjoys linear convergence. Furthermore, bluer curves have a larger slope, which indicates that a larger $\theta$ leads to faster convergence. Such a finding is corroborated by Figure \ref{fig:logerrortotheta}(right), which plots the log error at the $100$-th communication round with respect to $\theta$. It is clear that the log error decreases linearly with the increase in $\theta$. These results thus confirm insights from Theorem \ref{thm:linearconvergence}.

\begin{figure}[htpb!]
\centering
\subfigure[Log reconstruction error vs communication round]{
\includegraphics[width=0.45\linewidth]{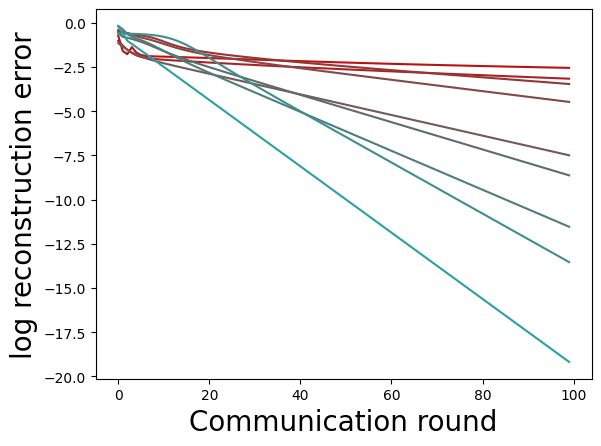}
}
\subfigure[Final log reconstruction error vs $\theta$]{
\includegraphics[width=.45\textwidth]{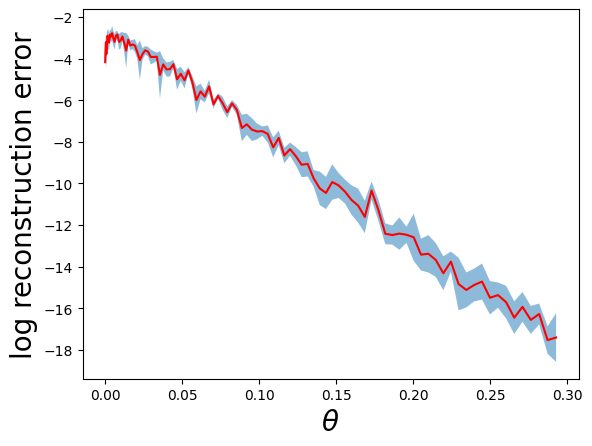}
}
\caption{Left: the learning curve of the reconstruction error. Each curve represents one set of experiments with one $\theta$. The bluer the curve is, the larger $\theta$ is. Right: log reconstruction error after 100 communication rounds for datasets with different misalignment parameters $\theta$. We run each experiment 10 times, each with a different random initialization. The red line represents the mean log error after 100 communication rounds for the ten experiments, and the blue-shaded region shows the confidence interval. }
\label{fig:logerrortotheta}
\end{figure}

\subsubsection{Dependence of Statistical Error on $n$}
\label{sec:experroronn}
Knowing that \name converges rather quickly, we can use the final iterates of Algorithm \ref{alg:consensus} as an estimate of the optimal solution to problem \eqref{eqn:problem}. To show that the estimate can indeed recover the true local and global PCs, we calculate the subspace error between eigenspace estimates and true values defined in \eqref{eqn:deflsubspace},
\begin{equation}
\label{eqn:subspaceerror}
\text{Subspace error} = \norm{\Pj{\bm{U}_{\tau}}-\matPi_g}_F^2+\frac{1}{N}\sum_{i=1}^N\norm{\Pj{\bm{V}_{(i),\tau}}-\matPi_{(i)}}_F^2
\end{equation}
Remember that Theorem \ref{thm:statisticalerror} shows that such error should decrease to $0$ as the number of observations on each client approaches infinity. Additionally, Corollary \ref{cor:statisticalerrorbydn} gives a finite-sample error bound of the subspace error.

Here, we benchmark with a one-shot approach \dispca \citep{oneshotdpca}. However, we provide a simple variant of  \dispca to make it amenable for personalization. 
For standard \dispca, each client first calculates the top $r_1+r_2$ principal components and sends them to the server. The server then concatenates all the received PCs into a  $d\times N(r_1+r_2)$ matrix and calculates the top $r_1$ principal components of the matrix. To enable personalization in  \dispca, we take the following route: we use the obtained top $r_1$ principal components $\bm{U}_{\dispca}$ as estimates of the global principal components. Then, we estimate local PCs with the help of the global ones. Specifically, the global PCs $\bm{U}_{\dispca}$ are sent back to clients. Each client then deflates the sample covariance matrix $\bm{S}_{(i),deflate}=\left(\bm{I}-\Pj{\bm{U}_{\dispca}}\right)\bm{S}_{(i)}\left(\bm{I}-\Pj{\bm{U}_{\dispca}}\right)$, and calculates the top $r_2$ principal components of  $\bm{S}_{(i),deflate}$ as local PCs.


To analyze the statistical consistency, we run \name on datasets with varying numbers of observations $n_{(i)}$ and compare with the benchmark algorithm \dispca. We set $n_{(1)}=n_{(2)}=\cdots=n$. We fix data dimension $d=15$ and generate data from $2$ global PCs and $10$ local PCs. On each client, the variances contributed by local PCs are set to be $100$ times larger than those contributed by global PCs to simulate large heterogeneity. This is achieved by setting the standard deviations of $\phi_{(i),q}$ to be 10 times smaller than $\varphi_{(i),q}$ in data-generating model \eqref{eqn:model}. We use $100$ clients. Among them $50$ clients have $n$ observations, and the rest $50$ clients only have $\frac{1}{10}n$ observations. We run both algorithms and estimate the subspace error \eqref{eqn:subspaceerror} from $5$ different random seeds.

\begin{figure}[htbp!]
\centering
  \includegraphics[width=0.45\linewidth]{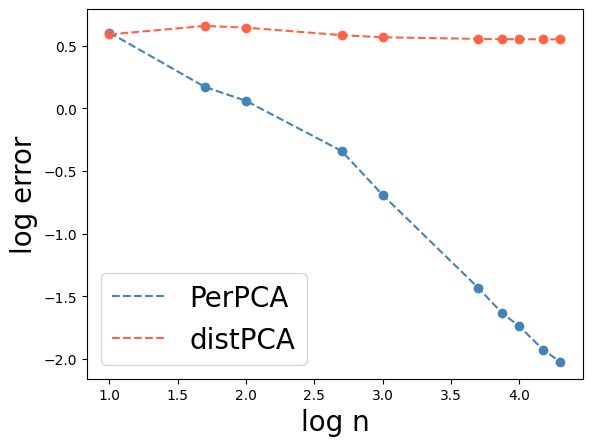}
  \caption{Log subspace error vs local observations $n$. \name is consistent while \dispca is not.}
  \label{fig:statisticalerroronn}
\end{figure}

Results in Figure \ref{fig:statisticalerroronn} show that \name achieves smaller statistical error for almost all $n$, and more importantly, the error decreases with $n$, which indicates that \name gives consistent estimates of global and local PCs. When the error is small, the slope of the curve is approximately $-1$, which matches the theoretical error upper bound $ O\left(\frac{1}{n}\right)$ in Corollary \ref{cor:statisticalerrorbydn}.

In comparison, the statistical error of \dispca does not decrease even when $n$ is very large, implying that the method is not consistent for heterogeneous datasets. This result also sheds light on an important insight. \textbf{Simply learning global components and using them for personalization in a train-then-personalize philosophy is not optimal, as global components from aggregated data may not contain useful information required for personalization}.

\subsubsection{Dependence of Statistical Error on $d$}
\label{sec:experrorond}
We also examine the performance of \name on data with different dimensions $d$. We fix $n=10000$ and generate data with different $d$. Other settings are the same as Section \ref{sec:experroronn}. We calculate the subspace error of estimates given by \name and \dispca. Results are plotted in Figure \ref{fig:statisticalerrorond}. 

\begin{figure}[h!]
\centering
  \includegraphics[width=0.45\linewidth]{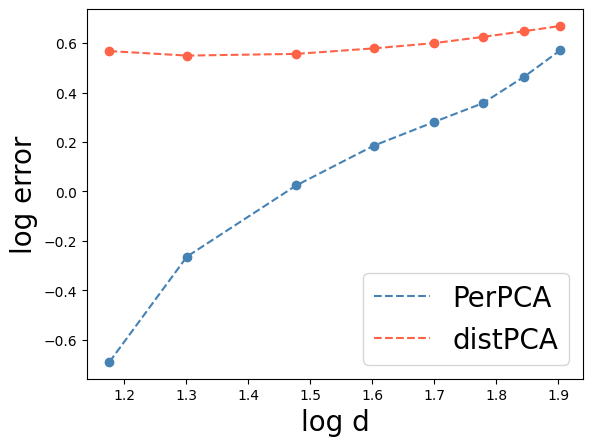}
  \caption{Log error vs data dimension $d$. }
  \label{fig:statisticalerrorond}
\end{figure}

From Figure \ref{fig:statisticalerrorond}, \name still achieves smaller statistical error for all $d$. Also, the error grows almost quadratically with $d$, which again matches the upper bound in Corollary \ref{cor:statisticalerrorbydn}.

\subsubsection{Dependence of Statistical Error on $N$}
\label{sec:experroronN}
Now, we explore whether the number of clients $N$ affects the statistical error. We fix $d=15$, $n=10000$, and change $N$ from $10$ to $1000$. The other settings are also the same as in Section \ref{sec:experroronn}. After obtaining global and local PCs, we calculate the subspace error of both global and local PCs \eqref{eqn:subspaceerror} and the subspace error of only global PCs $\norm{\Pj{\bm{U}}-\matPi_g}_F^2$. Results are plotted in Figure \ref{fig:logerrortoN}.

\begin{figure}[h!]
\centering
\subfigure[Average of subspace error of global and local PC estimates]{
\includegraphics[width=0.4\linewidth]{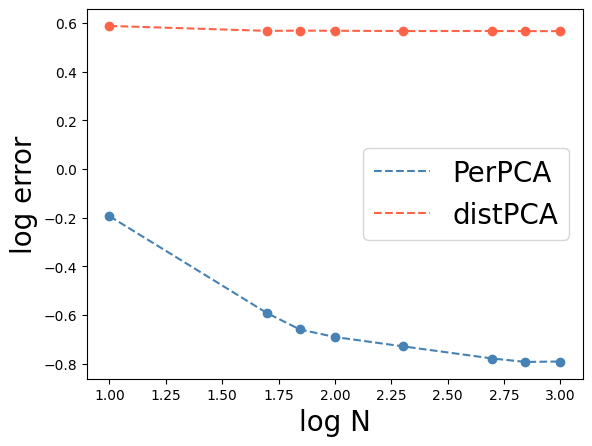}
\label{subfig:averagesubspaceerror}
}
\subfigure[Subspace error of global PC estimates]{
\includegraphics[width=.4\textwidth]{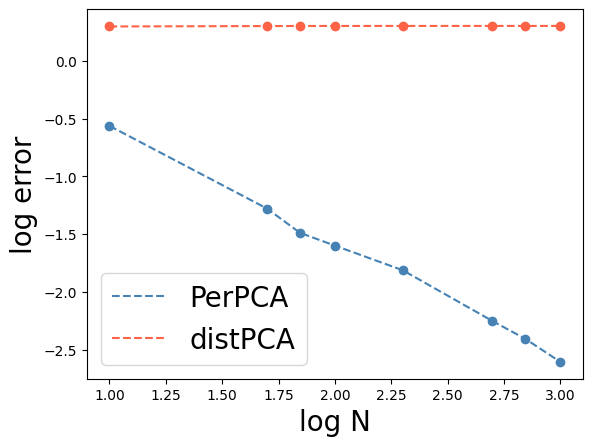}
\label{subfig:globalsubspaceerror}
}
\caption{Left: the average of local and global PCs' subspace error. Right: Global PCs' subspace error. }
\label{fig:logerrortoN}
\end{figure}

Figure \ref{subfig:averagesubspaceerror} shows that when $N$ increases, the average subspace error decreases slowly. The decreasing trend is more conspicuous for the subspace error of global PCs shown in Figure \ref{subfig:globalsubspaceerror}. This is understandable as when more clients participate in \name, more observations are available. Thus, global PCs can be better estimated. 

\subsubsection{Shared knowledge}
\label{sec:expsharedknowledge}
When the PCs on different clients are extremely heterogeneous, it is natural to ask whether clients are sharing knowledge and learning from each other in \name. Corollary \ref{cor:statisticalerrorbydn} indicates that clients can benefit from participating in the collaborative learning process from a theoretical perspective. In this section, we show numerical results on how the learned global components improve client-level predictions.

The dataset on client $i$ is split into a training set $\bm{Y}_{(i),train}$ and testing set $\bm{Y}_{(i),test}$. We use the training set $\bm{Y}_{(i),train}$ to find estimates for global and local PCs and the testing set to calculate the testing error. We focus on the reconstruction error defined in \eqref{eqn:reconstructionloss}. As in Section \ref{sec:experroronn}, we simulate two groups of clients with highly unbalanced dataset sizes. One group of clients has $n$ observations. We call them data-rich clients. The other group of clients have only $\frac{1}{10}n$ observations. We call them data-sparse clients. We set $N=100$ and $n=100$. 

In this experiment, we compare \name with $3$ benchmarks: \texttt{indivPCA}, \texttt{CPCA}, and \dispca. For \texttt{indivPCA}, each client uses their own data to calculate PCs independently without any knowledge sharing. \texttt{CPCA} represents PCA on the pooled data from all clients, i.e., all data are uploaded to a central server, and PCA is learned on the aggregated dataset. For fair comparisons, we allow \texttt{indivPCA} and \texttt{CPCA} to retain $r_1+r_2$ principal components. The results of testing reconstruction error averaged over the groups are shown in Table \ref{tab:borrowpower}. \texttt{Ground Truth} corresponds to the testing loss by the true PCs. 

\begin{table}[h!]
    \centering
    \begin{tabular}{ccccc|c}
\hline
   Client Group &  \texttt{indivPCA} &  \texttt{CPCA} &  \texttt{disPCA} &  \name   &      \texttt{Ground Truth}     \\
\hline
Data sparse & $1.87\pm 0.01$ & $2.07 \pm 0.01$ & $1.91 \pm 0.01$ & $\textbf{1.68} \pm 0.02$  & $1.50\pm 0.01$\\
Data rich & $1.80\pm 0.01$  &  $2.10 \pm 0.01$ & $1.88 \pm 0.01$   &    $\textbf{1.52}\pm 0.01$      &  $1.50\pm 0.01$   \\
\hline
\end{tabular}
    \caption{Testing reconstruction error averaged on each group}
    \label{tab:borrowpower}
\end{table}

From Table \ref{tab:borrowpower}, it is clear that \name achieves the smallest testing error in both the data-sparse and the data-rich group, thus having the best predictive performance. As \name outperforms \texttt{indivPCA}, we can conclude that \name learns useful shared knowledge. The results highlight \name's ability to extract common features from heterogeneous datasets. Also, \texttt{CPCA} exhibits the worst performance. This again highlights the need for personalized learning when data comes from heterogeneous sources.

\subsubsection{Overpametrization}
\label{exp:overparametrize}
In practice, when the true rank $r_1$ and $r_{2,(i)}$'s are unknown, practitioners may choose the rank of $\matU$ and $\matV_{(i)}$ to be larger than the ground truth. This is called an overparametrized regime \citep{overparametrize}. Overparametrization is a common technique for PCA and matrix factorization. Here we investigate the numerical performance of \name in an overparametrized regime. 

We use synthetic data to analyze the convergence behavior of \name. We set $d=30$ and $r_1=1$, $r_{2,(i)}=r_2=1$. Then we randomly generate data $\{\matY_{(i)}\}$ for $N=20$ clients and calculate the corresponding covariance matrix $\{\matS_{(i)}\}$. The data are generated without noise to better understand the convergence. 

We run overparametrized \name on the generated data. More specifically, we choose orthonormal matrices $\matU\in \mathbb{R}^{d\times \hr_1}$ and $\matV_{(i)}\in \mathbb{R}^{d\times \hr_2}$ with rank $\hr_1\ge r_1$ and $\hr_2\ge r_2$ in Algorithm \ref{alg:consensus}. Since, in practice, people may over-parametrize the rank of both global and local PCs differently, we study both cases separately.

\underline{\textit{Case 1: $\hr_2>r_2$}.} We choose the rank of local PCs $\hr_2$ to be higher than the ground truth $r_2$, while keeping $\hr_1=r_1$. Then we run \name starting from random initializations to obtain iterates $\matU_{\tau}$ and $\{\matV_{(i),\tau}\}$ for different $\tau$. We analyze three metrics: 

\begin{itemize}
\item Global error: $\frac{1}{N}\sum_{i=1}^N\frac{1}{n_{(i)}}\norm{\Pj{\matU_{\tau}}\matY_{(i)}-\bm{\Pi}_g\matY_{(i)}}_F^2$
\item Local error: $\frac{1}{N}\sum_{i=1}^N\frac{1}{n_{(i)}}\norm{\Pj{\matV_{(i),\tau}}\matY_{(i)}-\bm{\Pi}_{(i)}\matY_{(i)}}_F^2$
\item Reconstruction Error: $\frac{1}{N}\sum_{i=1}^N\frac{1}{n_{(i)}}\norm{\matY_{(i)}-\left(\bm{\Pi}_{(i)}+\bm{\Pi}_{g}\right)\matY_{(i)}}_F^2$
\end{itemize}


Apparently, the reconstruction error is upper bounded by the sum of the local error and global error. We plot these metrics for different $\hr_2$ in Figure \ref{fig:overparamr2}.

\begin{figure}
    \centering
    \includegraphics[width=0.5\linewidth]{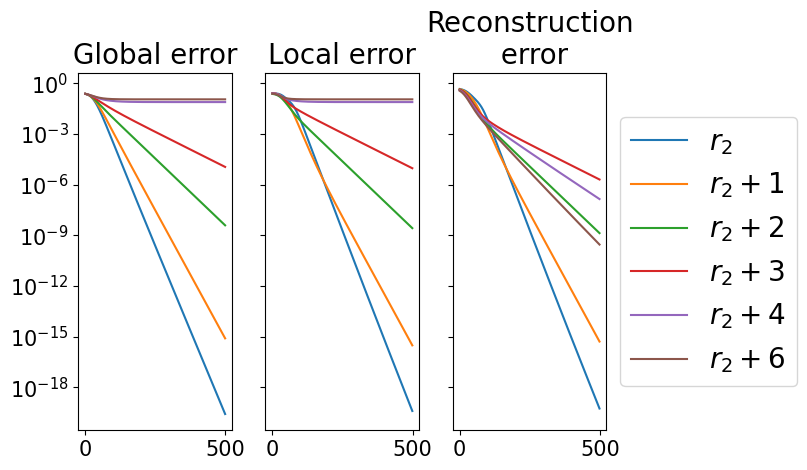}
    \caption{Simulations where we choose $\hr_2\ge r_2$.}
    \label{fig:overparamr2}
\end{figure}

There are a few interesting observations in Figure \ref{fig:overparamr2}. Firstly, for all $\hr_2$, the reconstruction errors decrease linearly. This is understandable as using a larger rank in local features $\hr_2$ can add more representation power to the model, thus helping model fitting. As the covariance matrices are noiseless, the linear decrease of reconstruction error is also consistent with the standard matrix factorization results in \citet{overparametrize}. Secondly, when the local features $\{\matV_{(i)}\}$ are slightly parametrized $r_2\le \hr_2\le r_2+3$, the global error and local error also decrease linearly. Such results show that with slightly overparametrized $\{\matV_{(i)}\}$, one can still recover the global features and the local ones. Thirdly, when $\hr_2$ is very large, $\hr_2\ge r_2+4$, the global error and local error decrease sublinearly. In this highly overparametrized regime, though the reconstruction error decreases to zero linearly, the learned global and local features do not converge to the ground truth equally fast.

When the ground truth $\bm{\Pi}_g$ and $\bm{\Pi}_{(i)}$ are unknown, one cannot evaluate the local and global error. Therefore, we propose the estimated misalignment $\theta_{est}$ value as a statistic indicative of the global-local separation:
$$
\theta_{est}= 1-\lambda_{\max}\left(\frac{1}{N}\sum_{i=1}^N\Pj{\hmatV_{(i)}}\right)
$$
where $\hmatV_{(i)}$ is the recovered local PCs on client $i$. $\theta_{est}$ measures how different the local features are. 

We calculate $\theta_{est}$ for different ranks of $\matV_{(i)}$ and show the results in Table \ref{tab:misalignmentvsrw}.

\begin{table}[H]
    \centering
    
    \begin{tabular}{ccccccc}
    \hline
        $\hr_2$ & $r_2$ & $r_2+1$ & $r_2+2$ & $r_2+3$ & $r_2+4$ & $r_2+6$\\\hline
    $\theta_{est}$ & $0.90$ & $0.69$ & $0.36$ & $0.19$ & $1.8\times 10^{-6}$ & $2.5\times 10^{-9}$\\
    \hline
    \end{tabular}
       \caption{Misalignment value $\theta$ for different ranks of matrix $\matV_{(i)}$}
    \label{tab:misalignmentvsrw}
\end{table}

From Table \ref{tab:misalignmentvsrw}, when $\hr_2$ increases from $r_2+3$ to $r_2+4$, $\theta_{est}$ decreases rapidly from $0.19$ to almost $0$. Such change indicates that the local features are very aligned when $\hr_2=r_2+4$. Thus, local features are not ``distinguishable''. The abrupt changes $\theta_{est}$ echo the results in Figure \ref{fig:overparamr2}: when $\theta_{est}$ is small, local PCs are similar, and the separation between local and global PCs is not clear. 

\underline{\textit{Case 2: $\hr_1>r_1$}.} Similarly, we choose the rank of global PCs $\hr_1$ to be higher than the ground truth $r_1$, while keeping $\hr_2=r_2$. The results are shown in Figure \ref{fig:overparamr1}.

\begin{figure}
    \centering
    \includegraphics[width=0.5\linewidth]{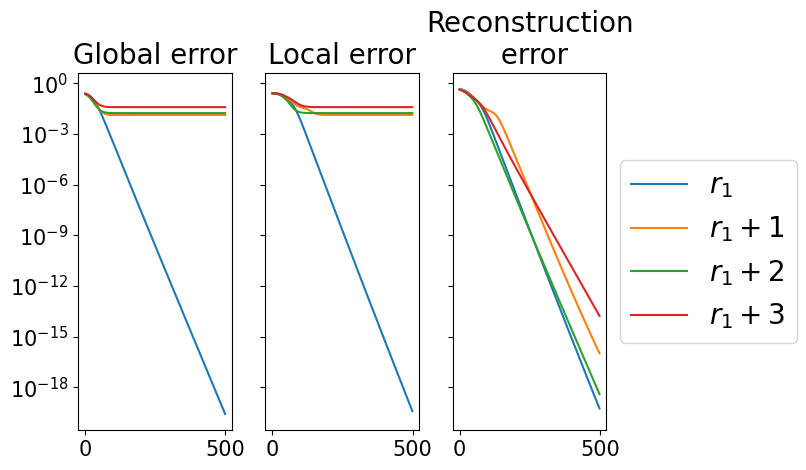}
    \caption{Simulations where we choose $\hr_1\ge r_1$.}
    \label{fig:overparamr1}
\end{figure}

Figure \ref{fig:overparamr1} demonstrates different qualitative behaviors than Figure \ref{fig:overparamr2}. Even when $\hr_1$ is slightly overparametrized, $\hr_1=r_1+1$, the global and local errors do not linearly decrease to $0$. Yet the fitting error for all cases decreases linearly. The comparison implies that when $\matU$ is overparametrized, the combined features $\matU$ and $\{\matV_{(i)}\}$ can still explain well the data variance, but may not exactly characterize the global and local features. 

In light of the insights gained, we recommend that practitioners carefully select small values for $r_1$ and $r_2$ in a way that ensures a small reconstruction error while also maintaining a large value for $\theta_{est}$ if they suspect heterogeneity among the data sources. 




\subsubsection{Clustering based on local principal components}
\label{exp:expclientclustering}
Apart from capturing the variance structure in the data, the learned local and global components can reveal high-level information about the client's interrelatedness. Below we present an interesting application of \name in client clustering. 

An important question in federated and distributed learning is how to cluster clients based on some summary statistics from their data. This is usually done by exploiting some distance metrics over the estimated parameters or gradients \citep{clusterfl} from each client. \name can pose an alternative approach for client clustering based on local PCs. The intuition is that by focusing on local PCs, differences across clients are more explicit compared to the raw data. More specifically, when $r_{2,(1)}=\cdots=r_{2,(N)}=r_2$, one can calculate the subspace distance between client $i$ and $j$ $\rho_{i,j}$ defined as: 
\begin{equation}
\label{eqn:rhodef}
    \rho_{i,j}=\frac{1}{r_2}\norm{\Pj{\hmatV_{(i)}}-\Pj{\hmatV_{(j)}}}_F^2
\end{equation}
If the column space of $\hmatV_{(i)}$ and $\hmatV_{(j)}$ are more similar, $\rho_{i,j}$ will be smaller. 

The $\rho_{i,j}$'s measure the closeness of local subspaces, thus revealing a similarity structure among clients. They form an $N\times N$ matrix $\bm{\rho}$. As such, simple spectral clustering \citep{esl} on $\bm{\rho}$ can be used to analyze the relations among different clients.

As an example, we generate clients from $10$ different client groups. Clients in one group have the same local PCs. Different groups have different local PCs. The data on clients within one group thus have a similar variance structure. We set $r_1 = 2$, $r_2=3$, and $d=15$. We apply \name and calculate the matrix $\bm{\rho}({\tau})$ with each communication round $\tau$. We omit the dependence $\bm{\rho}({\tau})$ on $\tau$ for simplicity. Then, we use multidimensional scaling (MDS) \citep{esl} and spectral clustering on $\bm{\rho}$. Results are shown in Figure \ref{fig:clientcluster}.

Since local PCs are randomly initialized, it is hard to find meaningful structures from initialization in Figure \ref{subfig:clientclusterinitialization}. However, after only one communication round, the true structure emerges in Figure \ref{subfig:clientclusteroneround}. After 30 communication rounds, clients can be effectively clustered based on their learned local PCs.  
\begin{figure}[h]
\centering
\subfigure[Initial local subspace]{
\includegraphics[width=.4\textwidth]{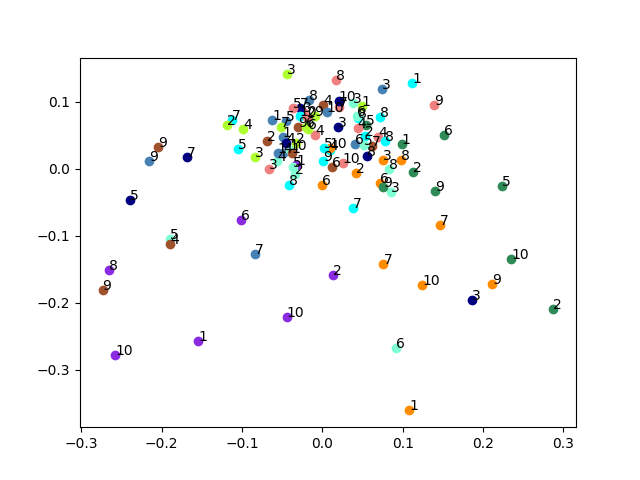}
\label{subfig:clientclusterinitialization}
}
\subfigure[After 1 communication round]{
\includegraphics[width=.4\textwidth]{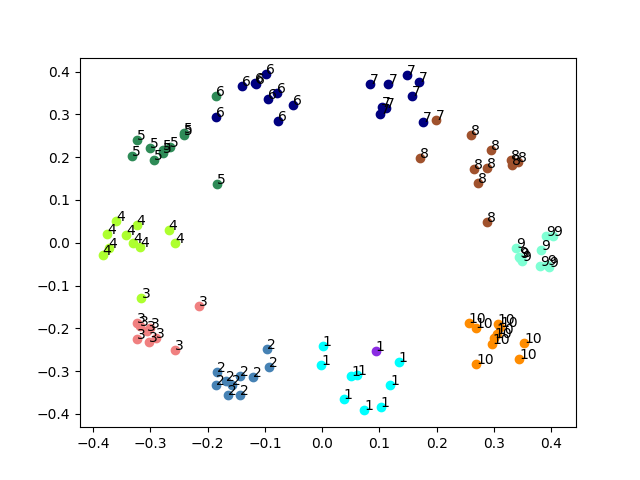}
\label{subfig:clientclusteroneround}
}
\subfigure[After 30 communication rounds]{
\includegraphics[width=.4\textwidth]{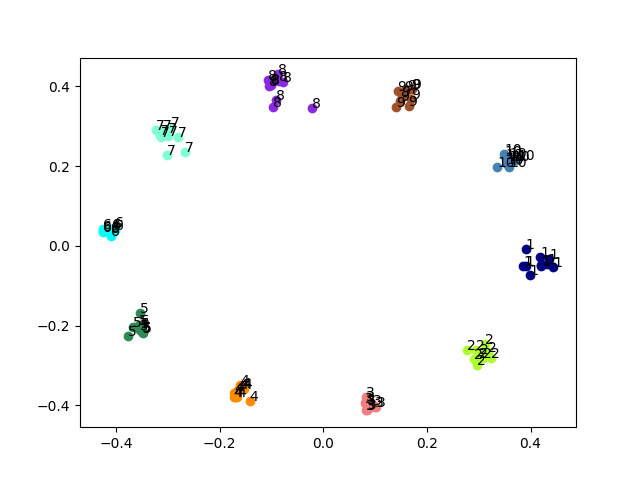}
\label{subfig:clientcluster30rounds}
}
\caption{MDS of the distance matrix $\bm{\rho}$. Color denotes the output of the spectral clustering algorithm. Numbers denote the true cluster labels. }
  \label{fig:clientcluster}
\end{figure}

\subsection{An illustrative example in comparison to \texttt{Robust PCA}}
\label{sec:expcomparisonwithrpca}
The philosophy of finding common and unique features can be applied to other tasks beyond explaining data variance. In this section, we use a simple example to demonstrate how \name can separate shared and unique features from image data. 

We start by comparing \name with \texttt{Robust PCA}. Though \texttt{Robust PCA} is proposed to learn low-rank and sparse parts, it is also potentially useful in finding irregular and common patterns from a dataset. When data come from different sources $\{\bm{Y}_{(i)}\}$ and have equal number of columns $n_{(1)}=\cdots=n_{(N)}=n$, one can stack them into one matrix $\bm{Y}_{\text{stack}}=\left[\vecr{\bm{Y}_{(1)}},\cdots,\vecr{\bm{Y}_{(N)}}\right]$. Then \texttt{Robust PCA} can be applied on the stacked matrix $\bm{Y}_{\text{stack}}\in \mathbb{R}^{nd\times N}$ to distinguish low rank and sparse parts. The common wisdom is to use a low-rank part to represent shared patterns and a sparse part to represent irregular trends \citep{robustpca}. 

The underlying assumption of such an approach is that unique features are somewhat sparse among all datasets. However, there are cases where a sparse matrix cannot model unique features. An example is shown in Table \ref{tab:videocomparisonppt}. We create $4$ images of different icons (triangle, disk, cross, and cloud) on similar background textures using PowerPoint and distinguish the icons from the background. As a greyscale image can naturally be represented by an observation matrix with dimensions of its height and width, we can construct $4$ datasets representing $4$ images. Then we apply \name and \texttt{Robust PCA} to identify the icons.   

\begin{table}[h!]
\begin{center}
\begin{small}
\begin{tabular}{ccccc}
Image & 1 & 2 & 3 & 4\\ 
Original & \begin{minipage}[c]{0.18\textwidth} \includegraphics[width=.99\linewidth]{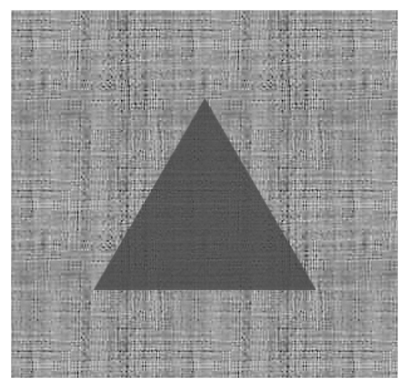} \end{minipage} &
\begin{minipage}[c]{0.18\textwidth} \includegraphics[width=.99\linewidth]{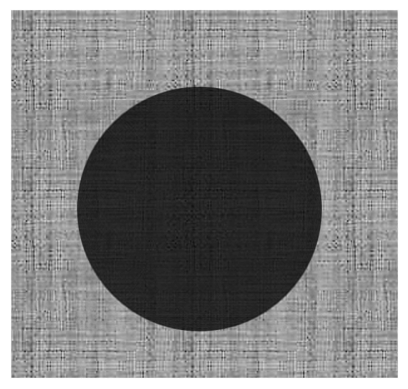} \end{minipage} &\begin{minipage}[c]{0.18\textwidth} \includegraphics[width=.99\linewidth]{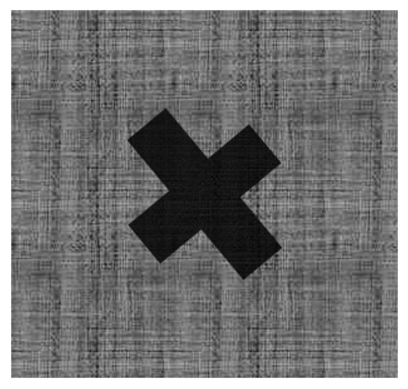} \end{minipage} & \begin{minipage}[c]{0.18\textwidth} \includegraphics[width=.99\linewidth]{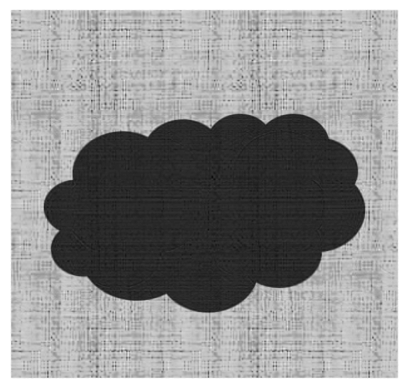} \end{minipage}  \\

\begin{tabular}{@{}c@{}}Sparse \\parts  by\\ RPCA\end{tabular} & \begin{minipage}[c]{0.18\textwidth} \includegraphics[width=.99\linewidth]{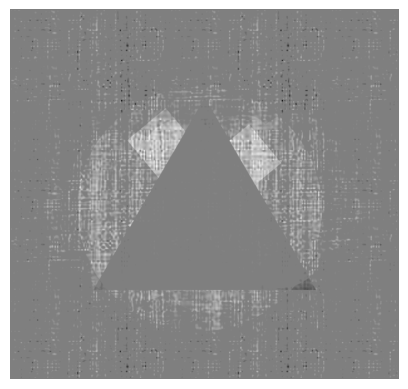} \end{minipage} &
\begin{minipage}[c]{0.18\textwidth} \includegraphics[width=.99\linewidth]{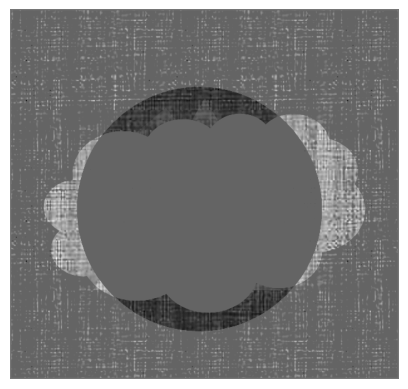} \end{minipage} &
\begin{minipage}[c]{0.18\textwidth} \includegraphics[width=.99\linewidth]{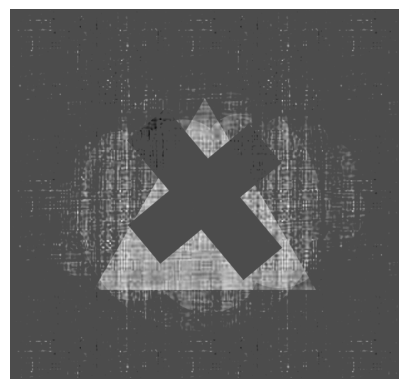} \end{minipage} & \begin{minipage}[c]{0.18\textwidth} \includegraphics[width=.99\linewidth]{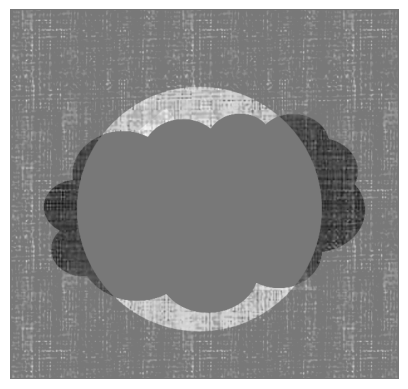} \end{minipage}  \\

\begin{tabular}{@{}c@{}}Projection to\\local PCs \\ by\\ \name\end{tabular} & \begin{minipage}[c]{0.18\textwidth} \includegraphics[width=.99\linewidth]{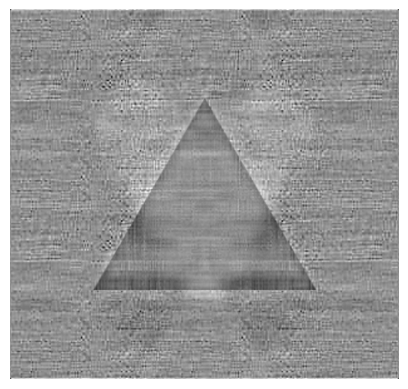} \end{minipage} &
\begin{minipage}[c]{0.18\textwidth} \includegraphics[width=.99\linewidth]{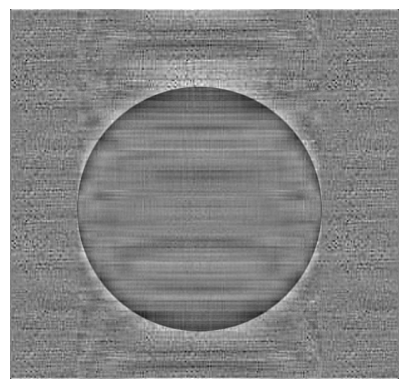} \end{minipage} &\begin{minipage}[c]{0.18\textwidth} \includegraphics[width=.99\linewidth]{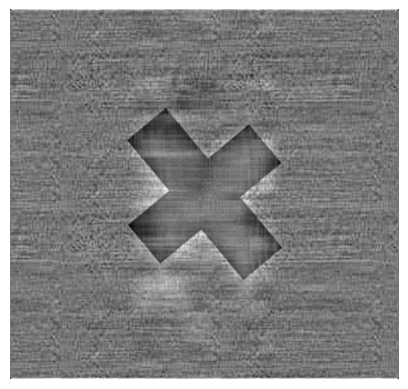} \end{minipage} & \begin{minipage}[c]{0.18\textwidth} \includegraphics[width=.99\linewidth]{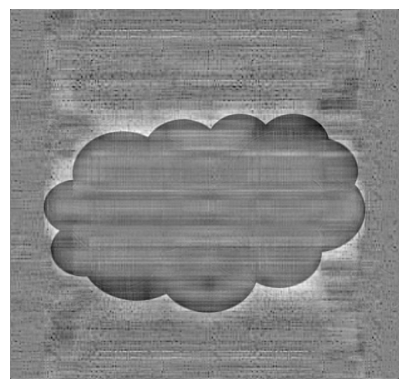} \end{minipage}  \\

\end{tabular}
\end{small}
\end{center}
\caption{A comparison of \name and \texttt{Robust PCA} on images of icons on background textures.}
\label{tab:videocomparisonppt}
\end{table}

From Table \ref{tab:videocomparisonppt}, it is apparent that \texttt{Robust PCA} does not perform well as it cannot recover the icons and always leaves shadows of icons on other images, probably because icons occupy a large space in the image and thus cannot be modeled by sparse noise. \name recovers the icons by projecting the images to the subspace spanned by local PCs. The third row in Table \ref{tab:videocomparisonppt} shows that \name has decent performance as the icons recovered have clear edges and shapes. 

This highlights the need for personalized inference in many applications where PCA is utilized.

\subsection{Real-life federated dataset}
\label{sec:expfemnist}
\revise{We also apply our algorithm on FEMNIST \citep{leaf} and CIFAR10 \citep{cifar10}. }

\revise{FEMNIST is a popular dataset in federated analytics. It consists of greyscale images of handwritten digits and English letters contributed by $3550$ different writers. Each image has $28\times 28 =784$ pixels. Different writers have different writing styles. Thus, the datasets are inherently heterogeneous. Our task is to learn a few PCs that can represent the dataset. On average, each client has 89 images. We represent an image by a vector in $\mathbb{R}^{784}$. For these vectors, we randomly choose 80\% of them to form the training set and take the rest as the test set. }

\revise{CIFAR10 is a multiclass image dataset. It consists of $60000$ images from $10$ classes. To simulate a heterogeneous setting, we separate the training and testing set of CIFAR10 into $20$ parts such that each part contains images from only $2$ classes. Then, we assign the separated parts to $20$ clients. The data partition scheme is consistent with federated learning literature \citep{fedavg}. Then we use similar data preprocessing procedures to vectorize the images on each client and construct the dataset $\{\matY_{(i)}\}$. }

\revise{We use \name, \texttt{indivPCA}, \texttt{CPCA}, and \dispca to fit PCs on training sets. Then, we evaluate the reconstruction error \eqref{eqn:reconstructionloss} on both training and test sets. The experiments are repeated $3$ times to calculate the mean and standard deviations
. The results are shown in Table \ref{tab:femnist}.}

\begin{table}[h!]
    \centering
    \begin{tabular}{ccccc}
\hline
   Reconstruction error &  \texttt{indivPCA} &  \texttt{CPCA} &  \texttt{disPCA} &  \name        \\
\hline
FEMNIST Training & $\textbf{0.49}\pm 0.01$ & $1.72 \pm 0.01$ & $1.43 \pm 0.01$ & $1.44 \pm 0.02$  \\
CIFAR10 Training & $\textbf{105.68}\pm 0.01$ & $114.79 \pm 0.01$ & $113.56 \pm 0.01$ & $113.69 \pm 0.02$  \\
\hline
FEMNIST Testing & $1.97\pm 0.03$  &  $1.73 \pm 0.01$ & $1.73 \pm 0.01$   &    $\textbf{1.70}\pm 0.01$   \\
CIFAR10 Testing & $120.79 \pm 0.02$  &  $115.44\pm 0.02$ & $115.43 \pm 0.01$   &    $\textbf{115.33}\pm 0.02$   \\
\hline
\end{tabular}
    \caption{\revise{The mean and standard deviations of the training and testing reconstruction error on FEMNIST}}
    \label{tab:femnist}
\end{table}

\revise{As the reconstruction error represents the difference between the original and reconstructed image, it represents how well the learned PCs can characterize the features in the image. In Table \ref{tab:femnist}, \texttt{indivPCA} achieves the lowest training error but incurs high testing error, suggesting that learned PCs overfit the training sets. \name has the lowest testing loss both in FEMNIST and CIFAR10, highlighting \name's ability to leverage common knowledge with unique trends to find better features from data.}

\subsection{Other Applications}
\label{sec:applicationsbeyondfederation}
Besides the experiments in the previous sections, \name can excel in various tasks that require separating shared and unique features. In this section, we will use video segmentation and topic extraction as two examples to show the applicability of \name. 

\subsubsection{Video segmentation}

The task of video segmentation is to separate moving parts (foreground) from stationary backgrounds in a video. For a video with $F$ frames, where each frame is an image with width $W$ and height $H$, we can model it as $F$ separated datasets. Each dataset has the data of one image frame or $H$ observations from $\mathbb{R}^W$. Therefore, we can naturally apply  \name to recover local and global PCs from the constructed datasets of all frames. Intuitively, the global PCs should capture shared features across all frames, representing the stationary background. Meanwhile, local PCs capture unique features in each frame corresponding to the moving parts. Hence, after obtaining the global and local PCs, we project the original picture onto the subspace spanned by these components to extract the background and foreground segments.

We use a surveillance video example from \citet{vacavant}. We set $r_1=50$ and $r_{2,(1)}=\cdots=r_{2,(N)}=50$ and apply Algorithm \ref{alg:consensus} with choice 2. Some segmentation results are shown in Table \ref{tab:simulatevideo}. From Table \ref{tab:simulatevideo}, we can see that backgrounds and moving parts are well separated by global and local PCs, validating \name's ability to find common and unique features in image datasets.

\begin{table}[h]
\begin{center}
\begin{small}
\begin{tabular}{cccc}
Sample Frame & 1 & 2 & 3 \\ 
Original & \begin{minipage}[c]{0.25\textwidth} \includegraphics[width=.99\linewidth]{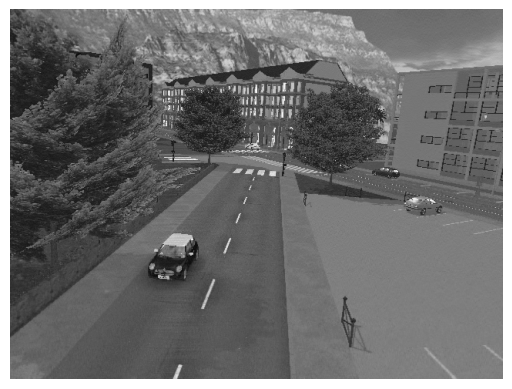} \end{minipage} & \begin{minipage}[c]{0.25\textwidth} \includegraphics[width=.99\linewidth]{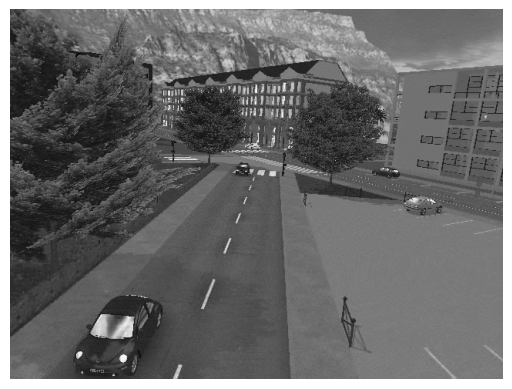} \end{minipage} & \begin{minipage}[c]{0.25\textwidth} \includegraphics[width=.99\linewidth]{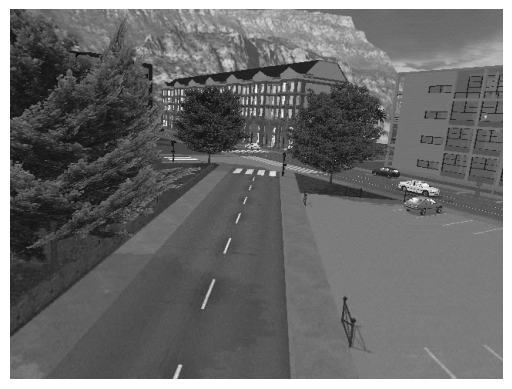} \end{minipage} \\

\begin{tabular}{@{}c@{}}Projection to\\ global\\ PC space\end{tabular} & \begin{minipage}[c]{0.25\textwidth} \includegraphics[width=.99\linewidth]{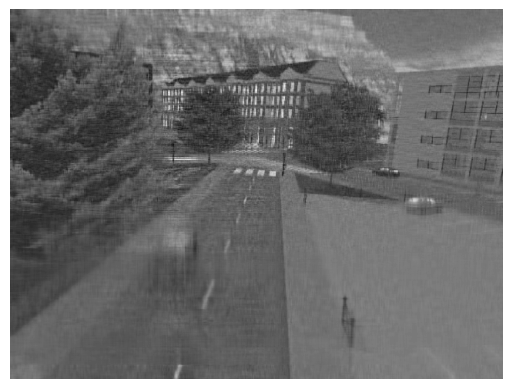} \end{minipage} & \begin{minipage}[c]{0.25\textwidth} \includegraphics[width=.99\linewidth]{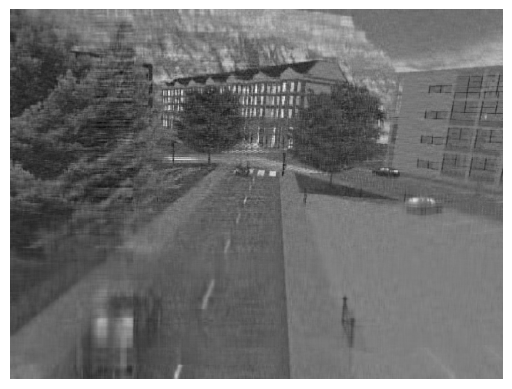} \end{minipage} & \begin{minipage}[c]{0.25\textwidth} \includegraphics[width=.99\linewidth]{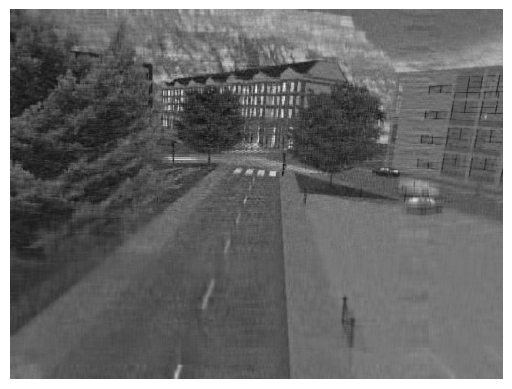} \end{minipage} \\

\begin{tabular}{@{}c@{}}Projection to\\ local\\ PC space\end{tabular} & \begin{minipage}[c]{0.25\textwidth} \includegraphics[width=.99\linewidth]{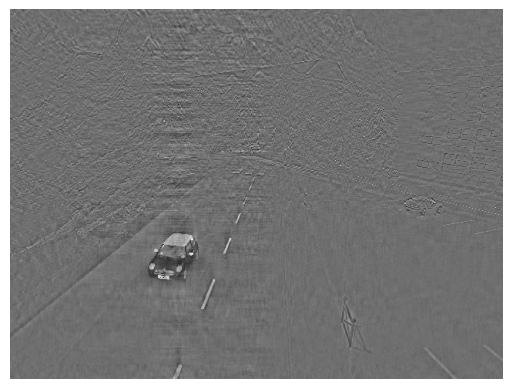} \end{minipage} & \begin{minipage}[c]{0.25\textwidth} \includegraphics[width=.99\linewidth]{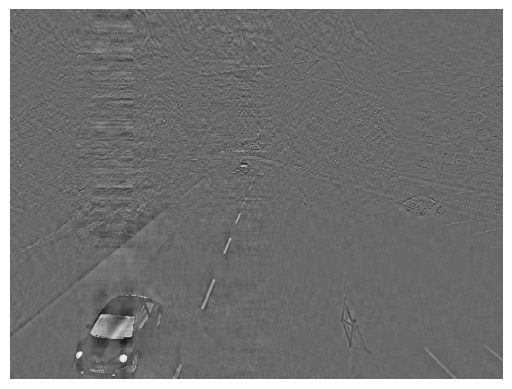} \end{minipage} & \begin{minipage}[c]{0.25\textwidth} \includegraphics[width=.99\linewidth]{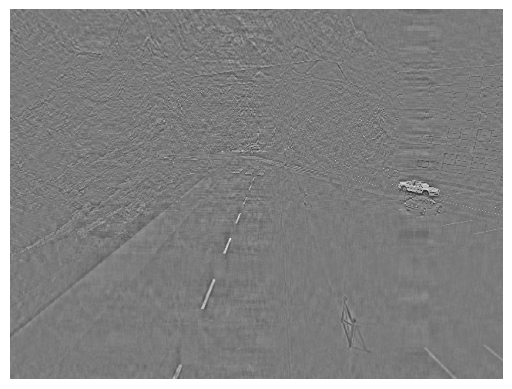} \end{minipage} \\  \\
\end{tabular}
\end{small}
\end{center}
\caption{Video segmentation. We separate moving cars from the background in a video from \citep{vacavant}. }
\label{tab:simulatevideo}
\end{table}

\subsubsection{Topic extraction}

\begin{table}[htbp!]
\centering
\caption{U.S. presidential debate key topics represented by local \& global PCs}
\label{tab:debate}
\begin{tabular}{cc}
\hline
Year & Top local principal components words                                                        \\
\hline
1960 & peace, Castro, Africa, Kennedy, now, world,  ...                                                          \\
1976 & billion, Carter, Governor, Africa, Ford, people, world, ...                                                \\
1980 & coal, oil, money, energy, Social, Security, Reagan, ...                                     \\
1984 & Union, tax, Soviet, arms, leadership, proposal, ...                                      \\
1988 & drug, young, strong, build, future, enforcement, good, ...                                  \\
1992 & Bill, school, children, care, health, taxes, reform, plan, control, ...                      \\
1996 & Clinton, Security, Medicare, budget, tax, Dole, Bob, ...                                    \\
2000 & school, public, plan, children, money, Social, Security, health, tax, ...                   \\
2004 & wrong, plan, cost, free, Saddam, troops, Iraq, war, health, tax, ...                        \\
2008 &  nuclear, oil, troops, Iraq, Afghanistan, Pakistan, health, Iran, energy, ... \\
2012 & million, small, business, China, Medicare, Romney, jobs, tax, ...                           \\
2016 & Russia, Trump, Hillary, companies, taxes, Mosul, Iran, deal, ...                            \\
2020 & Harris, Pence, Trump, down, Joe, Biden, jobs, Donald, health, ...\\
\hline
Common words & Tax, country, States, make, world, money, people, cut, ... \\
\hline
\end{tabular}
\end{table}

\name is also useful in modeling changing topics in language datasets. As a demonstration, we analyze the presidential debate transcriptions from 1960 to 2020 \citep{debatedataset}. The goal is to extract key debating topics for each specific election year.

The dataset contains $9135$ dialogues in $46$ debates from 13 election years, where one dialogue is the speech the speaker makes in the debate before another person speaks. After we remove common English words such as ``you'', ``I'', ``and'', ``at'', ``that'', from the text corpus, there are 5464 different words used in the dataset. 

We model the dataset as a collection of $13$ separate datasets, each of which has all the dialogues in one election year. To construct the observation matrix $\bm{Y}_{(i)}$, we first use one-hot encoding to map an English word into a vector in $\mathbb{R}^{5464}$. Then, we add all vectors corresponding to words that appear in one dialogue. The added vector is one observation in $\mathbb{R}^{5464}$. $\bm{Y}_{(i)}$ is formed by concatenating observations corresponding to dialogues in the election year. 

With the datasets constructed, we run \texttt{PerPCA} for $20$ communication rounds to extract local PCs. We set $r_1=2$ and $r_{2,(1)}=\cdots=r_{2,(N)}=2$. To show the key topics represented by the two local PCs, we find the words corresponding to the dimensions in each local and global PC that have the top $20$ largest absolute values. Table \ref{tab:debate} contains the most informative keywords from the top $20$ keywords obtained.

From Table \ref{tab:debate}, one can find different debating key topics for different years. For some years, the key topics are about public finance and domestic economic reform. For others, the key topics are more about international relations. These topics represent the central issues at a specific time in history. 

\section{Conclusion} \label{sec:con}
This work proposes \name, a systematic approach to decouple shared and unique features from heterogeneous datasets. We show that the problem is well formulated, and consistency can be guaranteed under mild conditions. A fully federated algorithm with a convergence guarantee is designed to efficiently obtain global and local PCs from noisy observations. Extensive simulations highlight \name's ability to separate shared and unique features in various applications.

As the first systematic approach to decouple shared and unique features quantitatively, we envision that \name can find use across various downstream analytics such as interpretability, clustering, classification, change detection, and transfer/federated learning. Within these areas, one can leverage unique knowledge so that differences become more explicit and leverage shared knowledge to transfer useful information from one source to another. Exploration along these directions may be promising. 

In addition, within \name, there are several avenues for expansion and exploration. \revise{On the optimization front, it is promising to design algorithms that can converge faster, or require lower computation resources, including Grassmannian gradient descent, and adaptive stepsize Stiefel gradient descent.} Further, extensions of \name that consider missing data, large noise, sparse factors, or malicious intruders are important directions for future work.

\section{Acknowledgements and Disclosure of Funding}
We thank the anonymous reviewers for their constructive feedback. This research is supported by Raed Al Kontar's National Science Foundation (NSF) CAREER Grant 2144147.


\newpage

\appendix

\section{Proof of Theorem \ref{thm:statisticalerror}}
\label{sec:proofofconistencytheorem}
In this section, we will show the proof of Theorem \ref{thm:statisticalerror}. We first use standard perturbation analysis on the eigenspaces of $\bm{\Sigma}_{(i)}$ \citep{fantope}. By assumption, $\matPi_g$ and $\matPi_{(i)}$ are the projections onto top eigenspaces of $\bm{\Sigma}_{(i)}$, therefore for any orthogonal projection matrix $\bm{P}_{\hat{\bm{U}}}$ and $\bm{P}_{\hat{\bm{V}}_{(i)}}$, we have:
\begin{equation*}
\begin{aligned}
&\innerp{\bm{\Sigma}_{(i)}}{\matPi_g+\matPi_{(i)}-\bm{P}_{\hat{\bm{U}}}-\bm{P}_{\hat{\bm{V}}_{(i)}}}\\
&=\innerp{\left(\matPi_g+\matPi_{(i)}\right)\bm{\Sigma}_{(i)}}{\bm{I}-\bm{P}_{\hat{\bm{U}}}-\bm{P}_{\hat{\bm{V}}_{(i)}}}-\innerp{\left(\bm{I}-\matPi_g-\matPi_{(i)}\right)\bm{\Sigma}_{(i)}}{\bm{P}_{\hat{\bm{U}}}+\bm{P}_{\hat{\bm{V}}_{(i)}}}\\
&\ge \lambda_{r_1+r_{2,(i)}}\left(\left(\matPi_g+\matPi_{(i)}\right)\matSigma_{(i)}\right)\innerp{\matPi_g+\matPi_{(i)}}{\bm{I}-\bm{P}_{\hat{\bm{U}}}-\bm{P}_{\hat{\bm{V}}_{(i)}}}\\
&-\lambda_{1}\left(\left(\matI-\matPi_g-\matPi_{(i)}\right)\matSigma_{(i)}\right)\innerp{\bm{I}-\matPi_g-\matPi_{(i)}}{\bm{P}_{\hat{\bm{U}}}+\bm{P}_{\hat{\bm{V}}_{(i)}}}\\
&\ge \delta\left(r_1+r_{2,(i)} - \innerp{\matPi_g+\matPi_{(i)}}{\bm{P}_{\hat{\bm{U}}}+\bm{P}_{\hat{\bm{V}}_{(i)}}}\right)
\end{aligned}
\end{equation*}
Summing both sides for $i$ from 1 to $N$, we have:
\begin{equation}
\label{proofeqn:differenceprojection}
\sum_{i=1}^N\innerp{\bm{\Sigma}_{(i)}}{\matPi_g+\matPi_{(i)}-\bm{P}_{\hat{\bm{U}}}-\bm{P}_{\hat{\bm{V}}_{(i)}}}\ge \delta\sum_{i=1}^N\left(r_1+r_{2,(i)} - \innerp{\matPi_g+\matPi_{(i)}}{\bm{P}_{\hat{\bm{U}}}+\bm{P}_{\hat{\bm{V}}_{(i)}}}\right)
\end{equation}
Since $\bm{P}_{\hat{\bm{U}}}$ and $\{\bm{P}_{\hat{\bm{V}}_{(i)}}\}$ are the optimal solutions to \eqref{eqn:problem}, and $\matPi_g$ and $\{\matPi_{(i)}\}$ are feasible, we know that:
\begin{equation}
\label{proofeqn:optimalitycondition}
\sum_{i=1}^N\innerp{\bm{S}_{(i)}}{\bm{P}_{\hat{\bm{U}}}+\bm{P}_{\hat{\bm{V}}_{(i)}}}\ge \sum_{i=1}^N\innerp{\bm{S}_{(i)}}{\matPi_g+\matPi_{(i)}}
\end{equation}
Combining \eqref{proofeqn:differenceprojection} and \eqref{proofeqn:optimalitycondition}, we can obtain:
$$
\sum_{i=1}^N\innerp{\bm{S}_{(i)}-\bm{\Sigma}_{(i)}}{\bm{P}_{\hat{\bm{U}}}+\bm{P}_{\hat{\bm{V}}_{(i)}}-\matPi_g-\matPi_{(i)}}\ge \delta \sum_{i=1}^N\left(r_1+r_{2,(i)}-\innerp{\matPi_g+\matPi_{(i)}}{\bm{P}_{\hat{\bm{U}}}+\bm{P}_{\hat{\bm{V}}_{(i)}} }\right)
$$
We can use the Cauchy-Schwartz inequality to further bound the left-hand side as:
$$
\innerp{\bm{S}_{(i)}-\bm{\Sigma}_{(i)}}{\bm{P}_{\hat{\bm{U}}}+\bm{P}_{\hat{\bm{V}}_{(i)}}-\matPi_g-\matPi_{(i)}}\le \norm{\bm{S}_{(i)}-\bm{\Sigma}_{(i)}}_F\norm{\bm{P}_{\hat{\bm{U}}}+\bm{P}_{\hat{\bm{V}}_{(i)}}-\matPi_g-\matPi_{(i)}}_F
$$
Notice that
\begin{equation*}
\begin{aligned}
&\norm{\bm{P}_{\hat{\bm{U}}}+\bm{P}_{\hat{\bm{V}}_{(i)}}-\matPi_g-\matPi_{(i)}}_F\\
&=\sqrt{\norm{\bm{P}_{\hat{\bm{U}}}+\bm{P}_{\hat{\bm{V}}_{(i)}}}_F^2+\norm{\matPi_g+\matPi_{(i)}}_F^2-2\innerp{\bm{P}_{\hat{\bm{U}}}+\bm{P}_{\hat{\bm{V}}_{(i)}}}{\matPi_g+\matPi_{(i)}}}\\
&=\sqrt{2}\sqrt{r_1+r_{2,(i)}-\innerp{\bm{P}_{\hat{\bm{U}}}+\bm{P}_{\hat{\bm{V}}_{(i)}}}{\matPi_g+\matPi_{(i)}}}
\end{aligned}
\end{equation*}
We thus have:
\begin{equation*}
\begin{aligned}
&\sum_{i=1}^N\innerp{\bm{S}_{(i)}-\bm{\Sigma}_{(i)}}{\bm{P}_{\hat{\bm{U}}}+\bm{P}_{\hat{\bm{V}}_{(i)}}-\matPi_g-\matPi_{(i)}}\\
&\le \sqrt{2}\sqrt{\sum_{i=1}^N\norm{\bm{S}_{(i)}-\bm{\Sigma}_{(i)}}_F^2}\sqrt{\sum_{i=1}^N\left[r_1+r_{2,(i)}-\innerp{\bm{P}_{\hat{\bm{U}}}+\bm{P}_{\hat{\bm{V}}_{(i)}}}{\matPi_g+\matPi_{(i)}}\right]}\\
\end{aligned}
\end{equation*}
by another application of Cauchy-Schwartz inequality.

Finally:
\begin{equation}
\label{proofeqn:unionspaceclose}
\frac{1}{N}\sum_{i=1}^N \left[r_1+r_{2,(i)}-\innerp{\bm{P}_{\hat{\bm{U}}}+\bm{P}_{\hat{\bm{V}}_{(i)}}}{\matPi_g+\matPi_{(i)}}\right]\le \frac{2}{N\delta^2} \sum_{i=1}^N\norm{\bm{S}_{(i)}-\bm{\Sigma}_{(i)}}_F^2   
\end{equation}
The relation \eqref{proofeqn:unionspaceclose} slightly extends the standard result from matrix perturbation theory. However, it only shows the summation of $\bm{P}_{\hat{\bm{U}}}$ and $\bm{P}_{\hat{\bm{V}}_{(i)}}$ is close to the summation of $\matPi_g$ and $\matPi_{(i)}$. One cannot infer additional information about the closeness of $\bm{P}_{\hat{\bm{U}}}$ to $\matPi_g$, or $\bm{P}_{\hat{\bm{V}}_{(i)}}$ to $\matPi_{(i)}$. In other words, \eqref{proofeqn:unionspaceclose} alone does not ensure that the recovered global and local PCs correspond to true PCs.

Such a guarantee is too weak in practice when we want to know if the solved $\bm{U}$ and $\bm{V}_{(i)}$'s are close to the ground truth. Fortunately, we can show that this is indeed the case if the problem satisfies the identifiability assumption \ref{ass:identifiability}. An important finding is the following lemma, which indicates that the closeness in direct sum space can lead to closeness in each global and local subspaces.
\begin{lemma}
\label{lm:sumspacetoindividualspace}
Suppose for $i=1,\cdots,N$, $\bm{P}_{\bm{U}}$, $\bm{P}_{\bm{V}_{(i)}}$ and $\bm{P}_{\bm{U}}^\star$, $\bm{P}_{\bm{V}_{(i)}}^\star$ are projection matrices satisfying $\bm{P}_{\bm{U}}\bm{P}_{\bm{V}_{(i)}}=0$ and $\bm{P}_{\bm{U}}^\star\bm{P}_{\bm{V}_{(i)}}^\star=0$ for each $i$. Among them, $\bm{P}_{\bm{U}}$ and $\bm{P}_{\bm{U}}^\star$ have rank $r_1$, $\bm{P}_{\bm{V}_{(i)}}$ and $\bm{P}_{\bm{V}_{(i)}}^\star$ have rank $r_{2,(i)}$. If there exists a positive constant $\theta>0$ such that
$$
\lambda_{max}(\frac{1}{N}\sum_{i=1}^N\bm{P}_{\bm{V}_{(i)}}^\star)\le 1-\theta
$$
We have the following bound:
\begin{equation}
\sum_{i=1}^Nr_1+r_{2,(i)}-\innerp{\bm{P}_{\bm{U}}+\bm{P}_{\bm{V}_{(i)}}}{\bm{P}_{\bm{U}}^\star+\bm{P}_{\bm{V}_{(i)}}^\star}\le N\left(r_1- \innerp{\bm{P}_{\bm{U}}^\star}{\bm{P}_{\bm{U}}}\right)+\sum_{i=1}^N r_{2,(i)}-\innerp{\bm{P}_{\bm{V}_{(i)}}^\star}{\bm{P}_{\bm{V}_{(i)}}}
\end{equation}
Also:
\begin{equation}
\label{eqn:individualdistancelowerbound}
\sum_{i=1}^Nr_1+r_{2,(i)}-\innerp{\bm{P}_{\bm{U}}+\bm{P}_{\bm{V}_{(i)}}}{\bm{P}_{\bm{U}}^\star+\bm{P}_{\bm{V}_{(i)}}^\star}\ge \frac{\theta}{2}\left(N\left(r_1- \innerp{\bm{P}_{\bm{U}}^\star}{\bm{P}_{\bm{U}}}\right)+\sum_{i=1}^N r_{2,(i)}-\innerp{\bm{P}_{\bm{V}_{(i)}}^\star}{\bm{P}_{\bm{V}_{(i)}}}\right)
\end{equation}
\end{lemma}
The proof of Lemma \ref{lm:sumspacetoindividualspace} is at the end of Section \ref{sec:auxlemmas}. By applying inequality \eqref{eqn:individualdistancelowerbound} to \eqref{proofeqn:unionspaceclose}, we can prove the desired error bound in Theorem \ref{thm:statisticalerror}.

\section{KKT condition }
\label{ap:proofforkkt}
We show the KKT conditions \eqref{eqn:kktcondition}. The lagrangian to the objective \eqref{eqn:problem} is:
\begin{equation}
\label{eqn:lagrangian}
\begin{aligned}
\mathscr{L}=&\sum_{i=1}^N\left[ \frac{1}{2}\tr{\bm{U}^T\bm{S}_{(i)}\bm{U}}+\frac{1}{2}\tr{\bm{V}_{(i)}^T\bm{S}_{(i)}\bm{V}_{(i)}}+\innerp{\bm{\Lambda}_{2,(i)}}{\bm{V}_{(i)}^T\bm{V}_{(i)}-\bm{I}}+\innerp{\bm{\Lambda}_{3,(i)}}{\bm{U}^T\bm{V}_{(i)}}\right]\\
&+\innerp{\bm{\Lambda}_1}{\bm{U}^T\bm{U}}
\end{aligned}
\end{equation}
where $\bm{\Lambda}_1\in\mathbb{R}^{r_1\times r_1}$, $\bm{\Lambda}_{2,(i)}\in \mathbb{R}^{r_{2,(i)}\times r_{2,(i)}}$, and $\bm{\Lambda}_{3,(i)}\in\mathbb{R}^{r_1\times r_{2,(i)}}$ are dual variables. The KKT conditions are:
\begin{equation}
\label{eqn:kktderivations}
\left\{\begin{aligned}
&\bm{S}_{(i)}\bm{V}_{(i)}+\bm{V}_{(i)}\left(\bm{\Lambda}_{2,(i)}+\bm{\Lambda}_{2,(i)}^T\right)+\bm{U}\bm{\Lambda}_{3,(i)}=0\\
&\sum_{i=1}^N\left[\bm{S}_{(i)}\bm{U}+\bm{V}_{(i)}\bm{\Lambda}_{3,(i)}^T\right]+\bm{U}\left(\bm{\Lambda}_{1}+\bm{\Lambda}_{1}^T\right)=0\\
& \bm{U}^T\bm{U}=\bm{I}_{r_1},\ \bm{V}_{(i)}^T\bm{V}_{(i)}=\bm{I}_{r_{2,(i)}},\ \bm{U}^T\bm{V}_{i}=0\\
\end{aligned}\right.
\end{equation}
By left multiplying the first equation in \eqref{eqn:kktderivations} with $\bm{I}-\Pj{\bm{U}}-\Pj{\bm{V}_{(i)}}$, we have $\left(\bm{I}-\Pj{\bm{U}}-\Pj{\bm{V}_{(i)}}\right)\bm{S}_{(i)}\bm{V}_{(i)}=0$, which is the first equation in \eqref{eqn:kktcondition}. By left multiplying the first equation in \eqref{eqn:kktderivations} with $\bm{U}^T$, we have $\bm{\Lambda}_{(3,(i))} = -\bm{U}^T\bm{S}_{(i)}\bm{V}_{(i)}$. Plugging this into the second equation in \eqref{eqn:kktderivations}, we have $\sum_{i=1}^N\left[\bm{S}_{(i)}\bm{U}-\Pj{\bm{V}_{(i)}}\bm{S}_{(i)}\bm{U}\right]+\bm{U}\left(\bm{\Lambda}_{1}+\bm{\Lambda}_{1}^T\right)=0$. We then left multiply both sides again by $\bm{I}-\Pj{\bm{U}}$. The second equation in \eqref{eqn:kktcondition} follows accordingly. One can also infer \eqref{eqn:kktderivations} from \eqref{eqn:kktcondition}.

\section{\revise{Proof of Theorem \ref{thm:statisticallowerbound}}}
\label{ap:proofstatisticallowerbound}
\revise{
Inspired by \citet{sparsepcalower}, in this section, we will use a ``spiked population model'' to demonstrate the lower bound. We will first use matrix perturbation analysis to estimate the leading order term for the estimation error of global PCs. Then, we verify our results through numerical experiments.}

\begin{proof}
To prove theorem \ref{thm:statisticallowerbound}, it suffices to find one set of parameters under which the statistical error is indeed $\Omega\left(\frac{1}{\theta}+\frac{1}{\delta^2}\right)$. For simplicity, we consider $N=2$ and $r_1=r_2=1$, i.e., each client is driven by one global feature and one local feature. We define a few needed signal vectors $\vecw_{1,1},\vecw_{1,2},\vecw_{2,1},\vecw_{2,2}\in \mathbb{R}^4$ and a noise vector $\vecw_3\in \mathbb{R}^4$ as 

\begin{equation}
\label{eqn:populatrionmodel}
\begin{aligned}
&\vecw_{1,1}=\left(\cos \gamma \sin\alpha , 
\sin\alpha \sin\gamma, \cos \alpha, 0\right)^T\\
&\vecw_{1,2}=\left(\cos \gamma \cos\alpha , 
\cos\alpha \sin\gamma, -\sin \alpha, 0\right)^T\\
&\vecw_{2,1}=\left(\cos \gamma \sin\alpha , 
-\sin\alpha \sin\gamma, \cos \alpha, 0\right)^T\\
&\vecw_{2,2}=\left(\cos \gamma \cos\alpha , 
-\cos\alpha \sin\gamma, -\sin \alpha, 0\right)^T\\
&\vecw_3 = \left(0,0,0,1\right)^T
\end{aligned}
\end{equation}
Then, we define the population covariance matrix as,
\begin{equation}
\label{eqn:lbpopcov}
\begin{aligned}
&\matSigma_{1}=2\vecw_{1,1}\vecw_{1,1}^T+\vecw_{1,2}\vecw_{1,2}^T+\varrho \vecw_3\vecw_3^T\\
&\matSigma_{2}=2\vecw_{2,1}\vecw_{2,1}^T+\vecw_{2,2}\vecw_{2,2}^T+\varrho \vecw_3\vecw_3^T\\
\end{aligned}
\end{equation}
where $\varrho$ is a constant $\varrho <1$. In \eqref{eqn:lbpopcov}, $2\vecw_{1,1}\vecw_{1,1}^T+\vecw_{1,2}\vecw_{1,2}^T$ denotes the signal part in $\matSigma_1$ and $\varrho \vecw_3\vecw_3^T$ denotes the noise part. Apparently, the model \eqref{eqn:lbpopcov} satisfies assumption 4.1 in the main paper. The eigengap $\delta$ is $\delta=1-\varrho$.

It is easy to check that if we run \name directly on the population covariance matrices $\{\matSigma_{(1)},\matSigma_{(2)}\}$ defined in \eqref{eqn:lbpopcov}, the algorithm would recover the optimal global PC as $$
\vecu = \left(0,0,1,0\right)^T
$$
and the optimal local PCs as
\begin{equation}
\begin{aligned}
&\vecv_1 = \left(\cos\gamma,\sin\gamma,0,0\right)^T\\
&\vecv_2 = \left(\cos\gamma,-\sin\gamma,0,0\right)^T
\end{aligned}
\end{equation}
It is also easy to see that the misalignment parameter $\theta=\sin^2\gamma$ when $0\le\gamma\le \frac{\pi}{4}$. 

We further introduce $\vecvperp_1$ and $\vecvperp_2$ as,
\begin{align*}
 &\vecvperp_1 = \left(-\sin\gamma,\cos\gamma,0,0\right)^T \\
&\vecvperp_2 = \left(\sin\gamma,\cos\gamma,0,0\right)^T \\ 
\end{align*}

Now we consider the sample covariance matrices $\matS_1$ and $\matS_2$. For simplicity, we assume they are only slightly perturbed from the population covariance matrices; $\matS_1=\matSigma_1+\varepsilon\dS_1$ and $\matS_2=\matSigma_2+\varepsilon\dS_2$, where $\varepsilon << 1$ is a small number, and $\dS_1$ and $\dS_2$ are defined as,
\begin{equation}
\begin{aligned}
&\dS_1=\vecv_1\vecw_{3}^T+\vecw_{3}\vecv_{1}^T+\vecv_1\vecu^T+\vecu\vecv_1^T+\vecv_1\vecvperp_2^T+\vecvperp_2\vecv_1^T+\vecw_3\vecu^T+\vecu\vecw_3^T\\
&\dS_2=\vecv_2\vecw_{3}^T+\vecw_{3}\vecv_{2}^T+\vecv_2\vecu^T+\vecu\vecv_2^T+\vecv_2\vecvperp_1^T+\vecvperp_1\vecv_2^T\\
\end{aligned}
\end{equation}
i.e., there are some small perturbations in the sample covariance matrix. $\dS_1$ and $\dS_2$ model the small differences between the sample covariance and population covariance matrices. It is easy to calculate that,
\begin{align*}
\frac{1}{2}\left(\norm{\matS_{(1)}-\matSigma_{(1)}}_F^2+\norm{\matS_{(2)}-\matSigma_{(2)}}_F^2\right)^2=7\varepsilon^2
\end{align*}

We can run \name on the sample covariance matrices $\matS_1$ and $\matS_2$. The optimal global and local optimal PCs are denoted as $\hvecu$ and $(\hvecv_1,\hvecv_2)$. Apparently, $\hvecu$ and $(\hvecv_1,\hvecv_2)$ are a function of $\varepsilon$, and as $\varepsilon$ becomes zero, the sample covariance becomes the population covariance, and $(\hvecu, \hvecv_1,\hvecv_2)$ become $(\vecu, \vecv_1,\vecv_2)$.

To estimate $(\hvecu, \hvecv_1,\hvecv_2)$ when $\varepsilon$ is nonzero, we can use the KKT conditions to analyze how $(\hvecu, \hvecv_1,\hvecv_2)$ change with respect to $\varepsilon$. Remember that the KKT conditions \eqref{eqn:kktderivations} are,
\begin{equation}
\label{eqn:kktwithusv}
\begin{aligned}
\matS_1\hvecv_1 &= \hvecv_1 \lambda_{21} + \hvecv_1 \hvecv_1^T\matS_1\hvecu\\
\matS_2\hvecv_2 &= \hvecv_2 \lambda_{22} + \hvecv_2 \hvecv_2^T\matS_1\hvecu\\
\left(\matS_1+\matS_2\right)\hvecu &= \hvecu \lambda_{1} + \hvecu \hvecu^T\matS_1\hvecv_1+ \hvecu \hvecu^T\matS_2\hvecv_2\\
\end{aligned}
\end{equation} 

Since $\matSigma_{(1)}$ and $\matSigma_{(2)}$ do not have duplicate eigenvalues, from \citet{firstorder}, $(\hvecu,\hvecv_1,\hvecv_2)$ and $(\lambda_{1},\lambda_{21},\lambda_{22})$ are analytic functions of $\varepsilon$ when $\varepsilon$ is small. We can thus write the Taylor series expansion of  $(\hvecu,\hvecv_1,\hvecv_2)$ as,

\begin{equation}
\label{eqn:expansion}
\begin{aligned}
\hvecu(\varepsilon) &= \vecu + \varepsilon \vecu^{(1)}+\varepsilon^2\vecu^{(2)}+\cdots\\
\hvecv_1(\varepsilon) &= \vecv_1 + \varepsilon \vecv_1^{(1)}+\varepsilon^2\vecv_1^{(2)}+\cdots\\
\hvecv_2(\varepsilon) &= \vecv_2 + \varepsilon \vecv_2^{(1)}+\varepsilon^2\vecv_2^{(2)}+\cdots\\
\lambda_1(\varepsilon) &= \lambda_1^{(0)} + \varepsilon \lambda_1^{(1)}+\varepsilon^2\lambda_1^{(2)}+\cdots\\
\lambda_{21}(\varepsilon) &= \lambda_{21}^{(0)} + \varepsilon \lambda_{21}^{(1)}+\varepsilon^2\lambda_{21}^{(2)}+\cdots\\
\lambda_{22}(\varepsilon) &= \lambda_{22}^{(0)} + \varepsilon \lambda_{22}^{(1)}+\varepsilon^2\lambda_{22}^{(2)}+\cdots\\
\end{aligned}
\end{equation}
where $\vecu^{(1)}$ is the first-order coefficient and $\vecu^{(2)}$ is the second-order coefficient for the expansion of $\hvecu(\varepsilon)$. Similar notations are used for other variables. 

Then we can take the expansion \eqref{eqn:expansion} into the KKT conditions \eqref{eqn:kktwithusv} and match the $O(\varepsilon)$ terms on both sides,
\begin{equation}
\label{eqn:1storder}
\begin{aligned}
\dS_1\vecv_1+\matSigma_1 \vecv_1^{(1)} &= \vecv_1 \lambda_{21}^{(1)}+\vecv_1^{(1)} \lambda_{21}^{(0)} \\
&+ \vecv_1^{(1)} \vecv_1^T\matSigma_1\vecu+ \vecv_1 \vecv_1^{(1)}{}^T\matSigma_1\vecu+ \vecv_1 \vecv_1^T\dS_1\vecu+ \vecv_1 \vecv_1^T\matSigma_1\vecu^{(1)}\\
\dS_2\vecv_2+\matSigma_2\vecv_2^{(1)} &= \vecv_2 \lambda_{22}^{(1)}+\vecv_2^{(1)} \lambda_{22}^{(0)} \\
&+ \vecv_2^{(1)} \vecv_2^T\matSigma_2\vecu+ \vecv_2 \vecv_2^{(1)}{}^T\matSigma_2\vecu+ \vecv_2 \vecv_2^T\dS_2\vecu+ \vecv_2 \vecv_2^T\matSigma_2\vecu^{(1)}\\
\left(\dS_1+\dS_2\right)\vecu+\left(\matSigma_1+\matSigma_2\right)\vecu^{(1)} 
&= \vecu \lambda_{1}^{(1)} +\vecu^{(1)} \lambda_{1}^{(0)} \\
&+ \vecu^{(1)} \vecu^T\matSigma_1\vecv_1+ \vecu \vecu^{(1)}{}^T\matSigma_1\vecv_1+ \vecu \vecu^T\dS_1\vecv_1+ \vecu \vecu^T\matSigma_1\vecv_1^{(1)}\\
&+ \vecu^{(1)} \vecu^T\matSigma_2\vecv_2+ \vecu \vecu^{(1)}{}^T\matSigma_2\vecv_2+ \vecu \vecu^T\dS_2\vecv_2+ \vecu \vecu^T\matSigma_2\vecv_2^{(1)}\\
\end{aligned}
\end{equation}

To solve equation \eqref{eqn:1storder}, we can expand $\vecu^{(1)}$, $\vecv_1^{(1)}$, and $\vecv_2^{(1)}$ over a basis,
\begin{equation}
\label{eqn:basisexpansionofu1v1}
\begin{aligned}
\vecu^{(1)} &= \varphi_{00}\vecu+\varphi_{01}\vecv_1 + \varphi_{02}\vecvperp_2+\varphi_{03}\vecw_3\\
\vecv_1^{(1)} &= \varphi_{10}\vecu+\varphi_{11}\vecv_1 + \varphi_{12}\vecvperp_2+\varphi_{13}\vecw_3\\
\vecv_2^{(1)} &= \varphi_{20}\vecu+\varphi_{21}\vecvperp_1 + \varphi_{22}\vecv_2+\varphi_{23}\vecw_3\\
\end{aligned}
\end{equation}

Since $\norm{\hvecu}=\norm{\hvecv_1}=\norm{\hvecv_2}=1$, we know that $\varphi_{00}=\varphi_{11}=\varphi_{22}=0$. Then we can take \eqref{eqn:basisexpansionofu1v1} into \eqref{eqn:1storder}, and solve $\varphi$'s as,
\begin{equation}
\begin{aligned}
&\varphi_{01}=-\frac{1}{4}   \sin (2 \alpha ) \cot (\gamma )\\
&\varphi_{02}=-\frac{1}{2}   \sin (\alpha ) \cos (\alpha )\\
&\varphi_{03}=-\frac{2 \sin (2 \alpha )+\cos (2 \alpha )+2 \varrho -3}{4 \left(\varrho ^2-3 \varrho +2\right)}\\
&\varphi_{10}=\frac{1}{4} \sin (2 \alpha ) \cot (\gamma )\\
&\varphi_{12}=\frac{1}{4} (\cos (2 \alpha )+3)\\
&\varphi_{13}=-\frac{\sin (2 \alpha )-2 \cos (2 \alpha )+4 \varrho -6}{4 \left(\varrho ^2-3 \varrho +2\right)}\\
&\varphi_{20}=\frac{1}{4} \sin (2 \alpha ) \cot (\gamma )\\
&\varphi_{21}=\frac{1}{4} (\cos (2 \alpha )+3)\\
&\varphi_{23}=-\frac{\sin (2 \alpha )-2 \cos (2 \alpha )+4 \varrho -6}{4 \left(\varrho ^2-3 \varrho +2\right)}\\
\end{aligned}
\end{equation}
Now, we obtained the closed-form formula for the first-order perturbation of global and local PCs. It is straightforward to calculate that
\begin{equation}
\begin{aligned}
&\norm{\hvecu\hvecu^T-\vecu\vecu^T}_F^2\\
&\norm{\left(\vecu + \varepsilon \vecu^{(1)}+\varepsilon^2\vecu^{(2)}+\cdots\right)\left(\vecu + \varepsilon \vecu^{(1)}+\varepsilon^2\vecu^{(2)}+\cdots\right)^T-\vecu\vecu^T}_F^2\\
&=\varepsilon^22\norm{\vecu^{(1)}}^2+O\left(\varepsilon^3\right)\\
&= \frac{\varepsilon^2}{64} \left(3 \csc ^2(\gamma )+\frac{\left(4 \varrho +2 \sqrt{3}-5\right)^2}{\left(1-\varrho\right)^2
   \left(2-\varrho\right)^2}\right)+O(\varepsilon^3)
\end{aligned}
\end{equation}

Since we know that $\theta=\sin^2(\gamma)$ and $\delta=1-\varrho$ when $\gamma\le\frac{\pi}{4}$, we have,
\begin{equation}
\label{eqn:predictederror}
 \norm{\hvecu\hvecu^T-\vecu\vecu^T}_F^2=\frac{\varepsilon^2}{64} \left(3 \frac{1}{\theta}+\frac{\left( 2 \sqrt{3}-1-4\delta\right)^2}{\delta^2
   \left(\delta+1\right)^2}\right)+O(\varepsilon^3)   
\end{equation}

When $\varepsilon$ is small, the higher order term $O\left(\varepsilon^3\right)$ can be neglected. Thus the error in \eqref{eqn:predictederror} can be further simplified to $\Omega\left(\varepsilon^2\left(\frac{1}{\theta}+\frac{1}{\delta^2}\right)\right)$ when $\theta$ and $\delta$ are small. This completes our proof.
\end{proof}

\revise{We also verify the predicted error \eqref{eqn:predictederror} via numerical simulations. In the simulations, we run \name on the sample covariance matrices $\matS_1$ and $\matS_2$ to obtain the global PC $\hvecu$. Then we use $\hvecu$ to calculate the subspace error $\norm{\hvecu\hvecu^T-\vecu\vecu^T}$. This is the actual statistical error for the estimates from \name. We compare it with the predicted values in \eqref{eqn:predictederror} under different parameter values of $\theta$ and $\delta$. Results are shown in Figure \ref{fig:global_error_lower_bound_on_delta} and \ref{fig:global_error_lower_bound_on_theta}. }

\begin{figure}[!htb]
   \begin{minipage}{0.48\textwidth}
\centering \includegraphics[width=.98\textwidth]{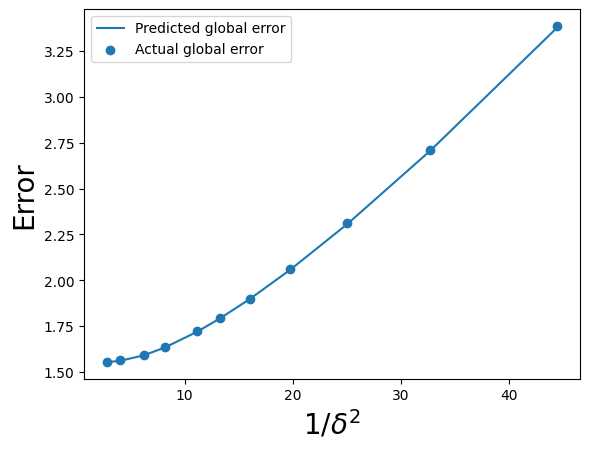}
     \caption{The (rescaled) global PC error  $\norm{\hvecu\hvecu^T-\vecu\vecu^T}_F^2/\varepsilon^2$ under different eigengap $\delta$.}\label{fig:global_error_lower_bound_on_delta}
   \end{minipage}\hfill
   \begin{minipage}{0.48\textwidth}
     \centering
     \includegraphics[width=.98\textwidth]{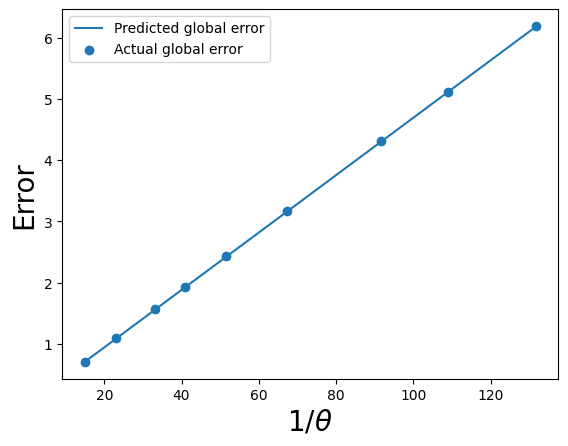}
     \caption{The (rescaled) global PC error  $\norm{\hvecu\hvecu^T-\vecu\vecu^T}_F^2/\varepsilon^2$ under different misalignment parameter $\theta$.}\label{fig:global_error_lower_bound_on_theta}
   \end{minipage}
\end{figure}
\revise{Figure \ref{fig:global_error_lower_bound_on_delta} and \ref{fig:global_error_lower_bound_on_theta} demonstrate good matches between the predicted statistical error and the actual statistical error. The two curves vividly show that when $\frac{1}{\theta}$ and $\frac{1}{\delta^2}$ is large, the global subspace error grows linearly with $\frac{1}{\theta}$ and $\frac{1}{\delta^2}$.}

\section{Proof of Theorem \ref{thm:sublinearconvergence}}
\label{ap:proofforsublinearconvergence}
Now we analyze the global convergence of Algorithm \ref{alg:consensus}. We begin by calculating the derivative of $\mathcal{L}_{(i),1}$ and $\mathcal{L}_{(i),2}$. The derivative of $\mathcal{L}_{(i),1}$ over $\bm{U}$ is:
\begin{equation}
\label{eqn:dln1du}
\nabla_{\bm{U}}\mathcal{L}_{(i),1}(\bm{U},\bm{V})=-\left(\bm{I}-\bm{P}_{\bm{V}}\right)\bm{S}_{(i)}\left(\bm{I}-\bm{P}_{\bm{V}}\right)\bm{U}
\end{equation}
When $\bm{V}^T\bm{U}=0$, this reduces to:
$$
\nabla_{\bm{U}}\mathcal{L}_{(i),1}(\bm{U},\bm{V})=-\left(\bm{I}-\bm{P}_{\bm{V}}\right)\bm{S}_{(i)}\bm{U}
$$

And the derivative of $\mathcal{L}_{(i),1}$ over $\bm{V}$ is:
\begin{equation}
\label{eqn:dln1dv}
\nabla_{\bm{V}}\mathcal{L}_{(i),1}(\bm{U},\bm{V})=\bm{P}_{\bm{U}}\bm{S}_{(i)}\bm{V}+\bm{S}_{(i)}\bm{P}_{\bm{U}}\bm{V}-\bm{P}_{\bm{U}}\bm{P}_{\bm{V}}\bm{S}_{(i)}\bm{V}-\bm{S}_{(i)}\bm{P}_{\bm{V}}\bm{P}_{\bm{U}}\bm{V}
\end{equation}
When $\bm{V}^T\bm{U}=0$, this reduces to:
$$
\nabla_{\bm{V}}\mathcal{L}_{(i),1}(\bm{U},\bm{V})=\bm{P}_{\bm{U}}\bm{S}_{(i)}\bm{V}
$$

Similarly, the derivative of $\mathcal{L}_{(i),2}$ over $\bm{V}$ is:
\begin{equation}
\label{eqn:dln2dv}
\nabla_{\bm{V}}\mathcal{L}_{(i),2}(\bm{U},\bm{V})=-\bm{S}_{(i)}\bm{V}
\end{equation}

The following lemma shows that the function we introduced is Lipschitz continuous. 
\begin{lemma}
When $\norm{\bm{U}}_{op}$ and $\norm{\bm{V}}_{op}$ are upper bounded by 1, the functions $\mathcal{L}_{(i),1}+\mathcal{L}_{(i),2}$ are Lipschitz continuous with constant $L$. More formally, for any $\bm{U}_1,\bm{U}_2,\bm{V}_1,\bm{V}_2\in \mathbb{R}^{d\times r}$, such that $\norm{\bm{U}_1}_{op},\norm{\bm{U}_2}_{op},\norm{\bm{V}_1}_{op},\norm{\bm{V}_2}_{op}\le 1$, we have:
\begin{equation}
\label{eqn:lip}
\begin{aligned}
&\Big\lVert\Big[\nabla_{\bm{U}}\mathcal{L}_{(i),1}(\bm{U}_2,\bm{V}_2)-\nabla_{\bm{U}}\mathcal{L}_{(i),1}(\bm{U}_1,\bm{V}_1),\\
&\nabla_{\bm{V}}\mathcal{L}_{(i),1}(\bm{U}_2,\bm{V}_2)+\nabla_{\bm{V}}\mathcal{L}_{(i),2}(\bm{U}_2,\bm{V}_2)-\nabla_{\bm{V}}\mathcal{L}_{(i),1}(\bm{U}_1,\bm{V}_1)-\nabla_{\bm{V}}\mathcal{L}_{(i),2}(\bm{U}_1,\bm{V}_1)\Big]\Big\rVert_F\\
&\le L \sqrt{\norm{\bm{U}_1-\bm{U}_2}_F^2+\norm{\bm{V}_1-\bm{V}_2}_F^2}
\end{aligned}
\end{equation}
where 
\begin{equation}
\label{eqn:lipconstantdef}
L=9\sqrt{2}G_{(i),op}
\end{equation}
\end{lemma}
\begin{proof}
First, we calculate the difference in the gradient of $\bm{U}$:
\begin{equation*}
\begin{aligned}
&\norm{\nabla_{\bm{U}}\mathcal{L}_{(i),1}(\bm{U}_2,\bm{V}_2)-\nabla_{\bm{U}}\mathcal{L}_{(i),1}(\bm{U}_1,\bm{V}_1)}_F\\
&=\norm{\left(\bm{I}-\bm{V}_1\bm{V}_1^T\right)\bm{S}_{(i)}\bm{U}_1-\left(\bm{I}-\bm{V}_2\bm{V}_2^T\right)\bm{S}_{(i)}\bm{U}_2}_F\\
&\le \norm{\left(\bm{I}-\bm{V}_1\bm{V}_1^T\right)\bm{S}_{(i)}\bm{U}_1-\left(\bm{I}-\bm{V}_2\bm{V}_2^T\right)\bm{S}_{(i)}\bm{U}_1}_F+\norm{\left(\bm{I}-\bm{V}_2\bm{V}_2^T\right)\bm{S}_{(i)}\bm{U}_1-\left(\bm{I}-\bm{V}_2\bm{V}_2^T\right)\bm{S}_{(i)}\bm{U}_2}_F\\
&\le \norm{\bm{V}_1\bm{V}_1^T-\bm{V}_2\bm{V}_2^T}_FG_{(i),op}+\norm{\bm{U}_1-\bm{U}_2}_FG_{(i),op}\\
&\le 2\norm{\bm{V}_1-\bm{V}_2}_FG_{(i),op}+\norm{\bm{U}_1-\bm{U}_2}_FG_{(i),op}\\
\end{aligned}
\end{equation*}
where we used the triangle inequality for the Frobenius norm for the first inequality, and Lemma \ref{lm:fnormofabbound} for the second and third inequality. Next, we calculate the difference in the gradient of $\bm{V}$:
\begin{equation*}
\begin{aligned}
&\norm{\nabla_{\bm{V}}\mathcal{L}_{(i),1}(\bm{U}_2,\bm{V}_2)+\nabla_{\bm{V}}\mathcal{L}_{(i),2}(\bm{U}_2,\bm{V}_2)-\nabla_{\bm{V}}\mathcal{L}_{(i),1}(\bm{U}_1,\bm{V}_1)-\nabla_{\bm{V}}\mathcal{L}_{(i),2}(\bm{U}_1,\bm{V}_1)}_F\\
&\le\norm{\left(\bm{P}_{\bm{U}_2}-\bm{I}\right)\bm{S}_{(i)}\bm{V}_2-\left(\bm{P}_{\bm{U}_1}-\bm{I}\right)\bm{S}_{(i)}\bm{V}_1}_F+\norm{\bm{S}_{(i)}\bm{P}_{\bm{U}_2}\bm{V}_2-\bm{S}_{(i)}\bm{P}_{\bm{U}_2}\bm{V}_2}\\
&+\norm{\bm{P}_{\bm{U}_2}\bm{P}_{\bm{V}_2}\bm{S}_{(i)}\bm{V_2}-\bm{P}_{\bm{U}_1}\bm{P}_{\bm{V}_1}\bm{S}_{(i)}\bm{V}_1}_F+\norm{\bm{S}_{(i)}\bm{P}_{\bm{V}_2}\bm{P}_{\bm{U}_2}\bm{V}_2-\bm{S}_{(i)}\bm{P}_{\bm{V}_1}\bm{P}_{\bm{U}_1}\bm{V}_1}_F\\
&\le 7\norm{\bm{V}_1-\bm{V}_2}_FG_{(i),op}+6\norm{\bm{U}_1-\bm{U}_2}_FG_{(i),op}\\
\end{aligned}
\end{equation*}
Summing them up, we know:
\begin{equation}
\begin{aligned}
&\Big\lVert\Big[\nabla_{\bm{U}}\mathcal{L}_{(i),1}(\bm{U}_2,\bm{V}_2)-\nabla_{\bm{U}}\mathcal{L}_{(i),1}(\bm{U}_1,\bm{V}_1),\\
&\nabla_{\bm{V}}\mathcal{L}_{(i),1}(\bm{U}_2,\bm{V}_2)+\nabla_{\bm{V}}\mathcal{L}_{(i),2}(\bm{U}_2,\bm{V}_2)-\nabla_{\bm{V}}\mathcal{L}_{(i),1}(\bm{U}_1,\bm{V}_1)-\nabla_{\bm{V}}\mathcal{L}_{(i),2}(\bm{U}_1,\bm{V}_1)\Big]\Big\rVert_F\\
&\le \norm{\nabla_{\bm{U}}\mathcal{L}_{(i),1}(\bm{U}_2,\bm{V}_2)-\nabla_{\bm{U}}\mathcal{L}_{(i),1}(\bm{U}_1,\bm{V}_1)}_F\\
&+\norm{\nabla_{\bm{V}}\mathcal{L}_{(i),1}(\bm{U}_2,\bm{V}_2)+\nabla_{\bm{V}}\mathcal{L}_{(i),2}(\bm{U}_2,\bm{V}_2)-\nabla_{\bm{V}}\mathcal{L}_{(i),1}(\bm{U}_1,\bm{V}_1)-\nabla_{\bm{V}}\mathcal{L}_{(i),2}(\bm{U}_1,\bm{V}_1)}_F\\
&\le 9\norm{\bm{V}_1-\bm{V}_2}_FG_{(i),op}+7\norm{\bm{U}_1-\bm{U}_2}_FG_{(i),op}\\
&\le 9\sqrt{2}G_{max,op} \sqrt{\norm{\bm{U}_1-\bm{U}_2}_F^2+\norm{\bm{V}_1-\bm{V}_2}_F^2}
\end{aligned}
\end{equation}
We thus complete the proof.
\end{proof}

Now we introduce some notations:
\begin{equation}
\label{eqn:defsquareu}
\square \bm{U}_{\tau} = \frac{1}{N}\sum_{i=1}^N\left(\bm{I}-\bm{P}_{\bm{U}_{\tau}}-\bm{P}_{\bm{V}_{(i),\tau}}\right)\bm{S}_{(i)}\bm{U}_{\tau}
\end{equation}
It is easy to verify $\square \bm{U}_{\tau}\in \mathcal{T}_{\bm{U}_{\tau}}$ when $\bm{U}^T\bm{V}_{(i),\tau}=0$:
$$
\bm{U}_{\tau}^T\square \bm{U}_{\tau}=0
$$

The Frobenius norm of $\square \bm{U}_{\tau}$ is upper bounded by:
\begin{equation}
\label{eqn:squareunormupperbound}
\begin{aligned}
&\norm{\square \bm{U}_{\tau}}_F\\
&=\norm{\frac{1}{N}\sum_{i=1}^N\left(\bm{I}-\bm{P}_{\bm{U}_{\tau}}-\bm{P}_{\bm{V}_{(i),\tau}}\right)\bm{S}_{(i)}\bm{U}_{\tau}}_F\\
&\le \frac{1}{N}\sum_{i=1}^N\norm{\left(\bm{I}-\bm{P}_{\bm{U}_{\tau}}-\bm{P}_{\bm{V}_{(i),\tau}}\right)}_{op}\norm{\bm{S}_{(i)}}_{op}\norm{\bm{U}_{\tau}}_F\\
&\le \frac{1}{N}\sum_{i=1}^NG_{max,op}\sqrt{r}\\
&= G_{max,op}\sqrt{r}
\end{aligned}
\end{equation}
where we applied Lemma \ref{lm:fnormofabbound} for the first inequality.

By the client update rule, we know that:
\begin{equation}
\bm{U}_{(i),\tau+1}=\bm{U}_{\tau}+\eta_{\tau}\left(\bm{I}-\bm{P}_{\bm{U}_{\tau}}-\bm{P}_{\bm{V}_{(i),\tau}}\right)\bm{S}_{(i)}\bm{U}_{\tau}
\end{equation}

Therefore, after the server takes the average of $\bm{U}_{(i),\tau+1}$ and performs a generalized retraction, the following holds:
\begin{equation}
\label{eqn:fullupdateofu}
 \bm{U}_{\tau+1}=\bm{U}_{\tau}+\eta_{\tau}\square \bm{U}_{\tau}+ \eta_{\tau}^2\bm{e}_{1,\tau}
\end{equation}
where $\bm{e}_{1,\tau}$ is an error term defined as:
$$
\bm{e}_{1,\tau}=\frac{1}{\eta_{\tau}^2}\left(\bm{U}_{\tau+1}-\bm{U}_{\tau}-\eta_{\tau}\square \bm{U}_{\tau}\right)
$$

By definition of a generalized retraction, since $\square \bm{U}_{\tau}$ is in the tangent space of $\bm{U}_{\tau}$, we have:
$$
\begin{aligned}
&\norm{\bm{e}_{1,\tau}}_F\le M_1 \norm{\square \bm{U}_{\tau}}_F^2
\end{aligned}
$$
where we applied the condition $\eta_{\tau}\le \frac{M_3}{\sqrt{r}G_{max,op}}$ thus $\norm{\eta_{\tau}\square\bm{U}_{\tau}}_F\le M_3$. Remember that $M_3$ is a numerical constant in the definition of generalized retraction (Definition \ref{def:generalizedretraction}).

Similarly, we define:
\begin{equation}
\label{eqn:defsquarevn}
\square \bm{V}_{(i),\tau} = \left(\bm{I}-\bm{P}_{\bm{U}_{\tau}}-\bm{P}_{\bm{V}_{(i),\tau}}\right)\bm{S}_{(i)}\bm{V}_{(i),\tau}
\end{equation}
Also by Lemma \ref{lm:fnormofabbound}, the Frobenius norm of $\square \bm{V}_{(i),\tau}$ is upper bounded by:
\begin{equation}
\label{eqn:squarevnormupperbound}
\begin{aligned}
&\norm{\square \bm{V}_{(i),\tau}}_F\\
&=\norm{\left(\bm{I}-\bm{P}_{\bm{U}_{\tau}}-\bm{P}_{\bm{V}_{(i),\tau}}\right)\bm{S}_{(i)}\bm{V}_{(i),\tau}}_F\\
&\le \norm{\left(\bm{I}-\bm{P}_{\bm{U}_{\tau}}-\bm{P}_{\bm{V}_{(i),\tau}}\right)}_{op}\norm{\bm{S}_{(i)}}_{op}\norm{\bm{V}_{(i),\tau}}_F\\
&\le G_{(i),op}\sqrt{r}
\end{aligned}
\end{equation}

Now we calculate the update of $\bm{V}_{(i),\tau}$ in one communication round. We summarize the result in the following lemma.
\begin{lemma}
\label{lm:fullupdateofv}
If we choose the stepsize $\eta_{\tau}\le \min\left\{\frac{M_3}{2},\frac{\sqrt{M_3}}{\sqrt{6+12M_1+M_1^2}}\right\}\frac{1}{G_{max,op}\sqrt{r}}$, the update of $\bm{V}_{(i)}$ is given by:
$$
\bm{V}_{(i),\tau+1}=\bm{V}_{(i),\tau}+\eta_{\tau}\square \bm{V}_{(i),\tau}-\eta_{\tau} \bm{U}_{\tau}\square\bm{U}_{\tau}^T \bm{V}_{(i),\tau}+\eta_{\tau}^2\bm{e}_{5,(i),\tau}
$$
where $\bm{e}_{5,(i),\tau}$ is an error term that satisfies:
$$
\norm{\bm{e}_{5,(i),\tau}}\le C_{5,0}\norm{\square \bm{U}_{\tau}}^2+C_{5,1}\norm{\square \bm{V}_{(i),\tau}}^2
$$
where $C_{5,0}$ and $C_{5,1}$ are two constants that only depend on $M_1$ and $M_2$ from the generalized retraction Definition \ref{def:generalizedretraction}.
\end{lemma}
\begin{proof}
We first calculate the projection:
\begin{equation*}
\begin{aligned}
&\bm{U}_{\tau+1}\bm{U}_{\tau+1}^T \\
&= \left(\bm{U}_{\tau}+\eta_{\tau}\square \bm{U}_{\tau}+ \eta_{\tau}^2\bm{e}_{1,\tau}\right)\left(\bm{U}_{\tau}+\eta_{\tau}\square \bm{U}_{\tau}+ \eta_{\tau}^2\bm{e}_{1,\tau}\right)^T\\
&=\bm{U}_{\tau}\bm{U}_{\tau}^T+\eta_{\tau}\left(\bm{U}_{\tau}\square \bm{U}_{\tau}^T+\square\bm{U}_{\tau}\bm{U}_{\tau}^T \right)+\eta_{\tau}^2 \bm{e}_{2,(i),\tau}
\end{aligned}
\end{equation*}
where $\bm{e}_{2,(i),\tau}$ is defined as:
$$
\bm{e}_{2,(i),\tau}=\square \bm{U}_{\tau}\square \bm{U}_{\tau}^T+ \bm{U}_{\tau}\bm{e}_{1,(i),\tau}^T+ \bm{e}_{1,(i),\tau}\bm{U}_{\tau}^T+\eta_{\tau} \square\bm{U}_{\tau}\bm{e}_{1,(i),\tau}^T+ \eta_{\tau}\bm{e}_{1,(i),\tau}\square\bm{U}_{\tau}^T+\eta_{\tau}^2\bm{e}_{1,(i),\tau}\bm{e}_{1,(i),\tau}^T
$$
Its norm is upper bounded by:
$$
\begin{aligned}
&\norm{\bm{e}_{2,(i),\tau}}_F\\
&\le \norm{\square \bm{U}_{\tau}}^2_F+2\norm{\bm{U}_{\tau}}_{op}\norm{\bm{e}_{1,(i),\tau}}_F+2\eta_{\tau} \norm{\square\bm{U}_{\tau}}_F\norm{\bm{e}_{1,(i),\tau}}_F+\eta_{\tau}^2\norm{\bm{e}_{1,(i),\tau}}_F^2\\
&\le \norm{\square \bm{U}_{\tau}}^2_F+2M_1\norm{\square \bm{U}_{\tau}}^2_F+2\eta_{\tau} M_1\norm{\square \bm{U}_{\tau}}^3_F+\eta_{\tau}^2 M_1^2\norm{\square \bm{U}_{\tau}}_F^4\\
&\le (1+3M_1+\frac{1}{4}M_1^2)\norm{\square \bm{U}_{\tau}}^2_F
\end{aligned}
$$
where the final inequality comes from upper bound \eqref{eqn:squareunormupperbound} and the choice of stepsize $\eta_{\tau}$: $\eta_{\tau}\le\frac{1}{G_{max,op}\sqrt{r}}\frac{1}{\sqrt{6+12M_1+M_1^2}}\le \frac{1}{2G_{max,op}\sqrt{r}}$. 

Similarly, we define $\bm{e}_{3,(i),\tau}$ as:
$$
\bm{e}_{3,(i),\tau}=\frac{1}{\eta_{\tau}^2}\left(\bm{V}_{(i),\tau+\frac{1}{2}}-\bm{V}_{(i),\tau}-\eta_{\tau}\square\bm{V}_{(i),\tau}\right)
$$
By definition of a retraction, the norm of $\bm{e}_{3,(i),\tau}$ is upper bounded by:
$$
\norm{\bm{e}_{3,(i),\tau}}_F\le M_1\norm{\square\bm{V}_{(i),\tau}}_F^2
$$
Then 
$$
\begin{aligned}
&\bm{U}_{\tau+1}\bm{U}_{\tau+1}^T\bm{V}_{(i),\tau+\frac{1}{2}}\\
&=\bm{U}_{\tau}\bm{U}_{\tau}^T\bm{V}_{(i),\tau+\frac{1}{2}}+\eta_{\tau}\left(\bm{U}_{\tau}\square \bm{U}_{\tau}^T+\square\bm{U}_{\tau}\bm{U}_{\tau}^T \right)\bm{V}_{(i),\tau+\frac{1}{2}}+\eta_{\tau}^2 \bm{e}_{2,(i),\tau}\bm{V}_{(i),\tau+\frac{1}{2}}\\
&=\bm{U}_{\tau}\bm{U}_{\tau}^T\left(\bm{V}_{(i),\tau}+\eta_{\tau}\square \bm{V}_{(i),\tau}+\eta_{\tau}^2\bm{e}_{3,(i),\tau}\right)\\
&+\eta_{\tau}\left(\bm{U}_{\tau}\square \bm{U}_{\tau}^T+\square\bm{U}_{\tau}\bm{U}_{\tau}^T \right)\left(\bm{V}_{(i),\tau}+\eta_{\tau}\square \bm{V}_{(i),\tau}+\eta_{\tau}^2\bm{e}_{3,(i),\tau}\right)+\eta_{\tau}^2 \bm{e}_{2,(i),\tau}\bm{V}_{(i),\tau+\frac{1}{2}}\\
&=\eta_{\tau}\bm{U}_{\tau}\square \bm{U}_{\tau}^T\bm{V}_{(i),\tau}+\eta_{\tau}^2\bm{U}_{\tau}\bm{U}_{\tau}^T\bm{e}_{3,(i),\tau}+\eta_{\tau}^2 \bm{e}_{2,(i),\tau}\bm{V}_{(i),\tau+\frac{1}{2}}+\eta_{\tau}^2\bm{U}_{\tau}\square \bm{U}_{\tau}^T\square \bm{V}_{(i),\tau}\\
&+\eta_{\tau}^3\left(\bm{U}_{\tau}\square \bm{U}_{\tau}^T+\square\bm{U}_{\tau}\bm{U}_{\tau}^T \right)\bm{e}_{3,(i),\tau}\\
&=\eta_{\tau}\bm{U}_{\tau}\square \bm{U}_{\tau}^T\bm{V}_{(i),\tau}+\eta_{\tau}^2\bm{e}_{4,(i),\tau}
\end{aligned}
$$
where we use $\bm{e}_{4,(i),\tau}$ to denote:
$$
\bm{e}_{4,(i),\tau}=\bm{U}_{\tau}\bm{U}_{\tau}^T\bm{e}_{3,(i),\tau}+\bm{e}_{2,(i),\tau}\bm{V}_{(i),\tau+\frac{1}{2}}+\bm{U}_{\tau}\square \bm{U}_{\tau}^T\square \bm{V}_{(i),\tau}+\eta_{\tau}\left(\bm{U}_{\tau}\square \bm{U}_{\tau}^T+\square\bm{U}_{\tau}\bm{U}_{\tau}^T \right)\bm{e}_{3,(i),\tau}
$$
Its norm is upper bounded as:
$$
\begin{aligned}
&\norm{\bm{e}_{4,(i),\tau}}_F\\
&\le \norm{\bm{e}_{3,(i),\tau}}_F+\norm{\bm{e}_{2,(i),\tau}}_F+\norm{\square \bm{U}_{\tau}}_F\norm{\square \bm{V}_{(i),\tau}}_F+\eta_{\tau} \left(\norm{\square \bm{U}_{\tau}}_F+\norm{\square \bm{V}_{(i),\tau}}_F\right)\norm{\bm{e}_{3,(i),\tau}}_F\\
&\le C_{4,0}\norm{\square \bm{U}_{\tau}}_F^2+C_{4,1}\norm{\square \bm{V}_{(i),\tau}}_F^2
\end{aligned}
$$
where 
$$
C_{4,0}=\frac{3}{2}+3M_1+\frac{1}{4}M_1^2
$$
and
$$
C_{4,1}=\frac{1}{2}+2M_1
$$
Thus we know when $\eta_{\tau}\le\min\{\frac{M_3}{2},\sqrt{\frac{M_3}{4C_{4,0}}}\}\frac{1}{G_{max,op}\sqrt{r}}$, $\norm{\bm{U}_{\tau+1}\bm{U}_{\tau+1}^T\bm{V}_{(i),\tau+\frac{1}{2}}}_F\le M_3$

Next we calculate the projection  $\mathcal{P}_{\mathcal{N}_{\bm{V}_{(i),\tau+\frac{1}{2}}}}\left(-\bm{U}_{\tau+1}\bm{U}_{\tau+1}^T\bm{V}_{(i),\tau+\frac{1}{2}}\right)$:
$$
\begin{aligned}
&\mathcal{P}_{\mathcal{N}_{\bm{V}_{(i),\tau+\frac{1}{2}}}}\left(-\bm{U}_{\tau+1}\bm{U}_{\tau+1}^T\bm{V}_{(i),\tau+\frac{1}{2}}\right)=-\bm{V}_{(i),\tau+\frac{1}{2}}\left(\bm{V}_{(i),\tau+\frac{1}{2}}^T\bm{U}_{\tau+1}\bm{U}_{\tau+1}^T\bm{V}_{(i),\tau+\frac{1}{2}}\right)\\
&=\eta_{\tau}^2\bm{V}_{(i),\tau+\frac{1}{2}}\bm{V}_{(i),\tau+\frac{1}{2}}^T\bm{e}_{4,(i),\tau}
\end{aligned}
$$
We use $\bm{e}_{5,(i),\tau}$ to denote the difference between $\grof{\bm{V}_{(i),\tau+\frac{1}{2}}}{-\bm{U}_{\tau+1}\bm{U}_{\tau+1}^T\bm{V}_{(i),\tau+\frac{1}{2}}}$ and $\bm{V}_{(i),\tau}+\eta_{\tau}\square \bm{V}_{(i),\tau}-\eta_{\tau}\bm{U}_{\tau}\square \bm{U}_{\tau}^T\bm{V}_{(i),\tau}$, then
its norm is upper bounded by:
\begin{equation*}
\begin{aligned}
&\eta_{\tau}^2\norm{\bm{e}_{5,(i),\tau}}_F\\
&=\norm{\grof{\bm{V}_{(i),\tau+\frac{1}{2}}}{-\bm{U}_{\tau+1}\bm{U}_{\tau+1}^T\bm{V}_{(i),\tau+\frac{1}{2}}}-\bm{V}_{(i),\tau}-\eta_{\tau}\square \bm{V}_{(i),\tau}+\eta_{\tau}\bm{U}_{\tau}\square \bm{U}_{\tau}^T\bm{V}_{(i),\tau}}_F\\
&\le \norm{\grof{\bm{V}_{(i),\tau+\frac{1}{2}}}{-\bm{U}_{\tau+1}\bm{U}_{\tau+1}^T\bm{V}_{(i),\tau+\frac{1}{2}}}-\bm{V}_{(i),\tau+\frac{1}{2}}+\bm{U}_{\tau+1}\bm{U}_{\tau+1}^T\bm{V}_{(i),\tau+\frac{1}{2}}}_F\\
&+\norm{\bm{V}_{(i),\tau+\frac{1}{2}}-\bm{U}_{\tau+1}\bm{U}_{\tau+1}^T\bm{V}_{(i),\tau+\frac{1}{2}}-\bm{V}_{(i),\tau}-\eta_{\tau}\square \bm{V}_{(i),\tau}+\eta_{\tau}\bm{U}_{\tau}\square \bm{U}_{\tau}^T\bm{V}_{(i),\tau}}_F
\end{aligned}
\end{equation*}

By property \ref{cons:approximate} of the generalized retraction in Definition \ref{def:generalizedretraction}, we have:
\begin{equation*}
\begin{aligned}
&\norm{\grof{\bm{V}_{(i),\tau+\frac{1}{2}}}{-\bm{U}_{\tau+1}\bm{U}_{\tau+1}^T\bm{V}_{(i),\tau+\frac{1}{2}}}-\left(\bm{V}_{(i),\tau+\frac{1}{2}}-\bm{U}_{\tau+1}\bm{U}_{\tau+1}^T\bm{V}_{(i),\tau+\frac{1}{2}}\right)}_F\\
&\le M_1\norm{\Ptangent{\bm{V}_{(i),\tau+\frac{1}{2}}}{-\bm{U}_{\tau+1}\bm{U}_{\tau+1}^T\bm{V}_{(i),\tau+\frac{1}{2}}}}_F^2+(M_2+1)\norm{\Pnormal{\bm{V}_{(i),\tau+\frac{1}{2}}}{-\bm{U}_{\tau+1}\bm{U}_{\tau+1}^T\bm{V}_{(i),\tau+\frac{1}{2}}}}_F\\
&\le M_1\norm{-\bm{U}_{\tau+1}\bm{U}_{\tau+1}^T\bm{V}_{(i),\tau+\frac{1}{2}}}_F^2+(M_2+1)\eta_{\tau}^2\norm{\bm{V}_{(i),\tau+\frac{1}{2}}\bm{V}_{(i),\tau+\frac{1}{2}}^T\bm{e}_{4,(i),\tau}}_F\\
&= M_1\norm{\eta_{\tau}\bm{U}_{\tau}\square \bm{U}_{\tau}^T\bm{V}_{(i),\tau}+\eta_{\tau}^2\bm{e}_{4,(i),\tau}}_F^2+(M_2+1)\eta_{\tau}^2\norm{\bm{V}_{(i),\tau+\frac{1}{2}}\bm{V}_{(i),\tau+\frac{1}{2}}^T\bm{e}_{4,(i),\tau}}_F\\
&\le 2M_1\eta_{\tau}^2\left(\norm{\bm{U}_{\tau}\square \bm{U}_{\tau}^T\bm{V}_{(i),\tau}}_F^2+\eta_{\tau}^4\norm{\bm{e}_{4,(i),\tau}}_F^2\right)+(M_2+1)\eta_{\tau}^2\norm{\bm{V}_{(i),\tau+\frac{1}{2}}\bm{V}_{(i),\tau+\frac{1}{2}}^T\bm{e}_{4,(i),\tau}}_F\\
&\le \eta_{\tau}^2\norm{\square \bm{U}_{\tau}}_F^2\left(2M_1(C_{4,0}+C_{4,1})C_{4,0}+2M_1+M_2C_{4,0} \right)\\
&+\eta_{\tau}^2\norm{\square \bm{V}_{(i),\tau}}_F^2\left(2M_1(C_{4,0}+C_{4,1})C_{4,1}+M_2C_{4,1} \right)
\end{aligned}
\end{equation*}
For the second part:
\begin{equation*}
\begin{aligned}
&\norm{\bm{V}_{(i),\tau+\frac{1}{2}}-\bm{U}_{\tau+1}\bm{U}_{\tau+1}^T\bm{V}_{(i),\tau+\frac{1}{2}}-\bm{V}_{(i),\tau}-\eta_{\tau}\square \bm{V}_{(i),\tau}+\eta_{\tau}\bm{U}_{\tau}\square \bm{U}_{\tau}^T\bm{V}_{(i),\tau}}_F\\
&=\eta_{\tau}^2\norm{e_{3,(n),\tau}-e_{4,(n),\tau}}_F\\
&\le \eta_{\tau}^2\norm{e_{3,(n),\tau}}_F+\eta_{\tau}^2\norm{e_{4,(n),\tau}}_F
\end{aligned}
\end{equation*}
Therefore, the norm of $\bm{e}_{5,(i),\tau}$ is upper bounded as:
\begin{equation*}
\begin{aligned}
&\norm{\bm{e}_{5,(i),\tau}}_F\\
&\le \norm{\square \bm{U}_{\tau}}_F^2\left(2M_1(C_{4,0}+C_{4,1})C_{4,0}+2M_1+M_2C_{4,0}+C_{4,0}\right)\\
&+\norm{\square \bm{V}_{(i),\tau}}_F^2\left(2M_1(C_{4,0}+C_{4,1})C_{4,1}+M_2C_{4,1}+M_1+C_{4,1} \right)
\end{aligned}
\end{equation*}

This completes our proof, with 
$$
C_{5,0}=\frac{1}{8} \left(12 \left(M_2+1\right)+M_1 \left(24 M_2+M_1 \left(M_1 \left(M_1
   \left(M_1+32\right)+254\right)+2 \left(M_2+109\right)\right)+88\right)\right)
$$
and
$$
C_{5,1}=M_1^4+\frac{81 M_1^3}{4}+13 M_1^2+\left(2 M_2+5\right) M_1+\frac{1}{2} \left(M_2+1\right)
$$
\end{proof}
The following lemma shows the sufficient decrease property:
\begin{lemma}
\label{lm:suffcientdecrease}
(Formal version of Lemma \ref{lm:informalsuffcientdecrease}) When we choose the stepsize $\eta_{\tau}\le\frac{1}{G_{max,op}\sqrt{r}}\min\left\{\frac{M_3}{2},\frac{\sqrt{M_3}}{\sqrt{6+12M_1+M_1^2}}\right\}$, and  $\bm{U}_{\tau}$ and $\bm{V}_{(i),\tau}$ satisfy the orthogonality condition $\bm{U}_{\tau}^T\bm{V}_{(i),\tau}=0$, we have:
\begin{equation}
\begin{aligned}
&\left\langle \sum_{i=1}^N\nabla_{\bm{U}}\mathcal{L}_{(i),1}(\bm{U}_{\tau},\bm{V}_{(i),\tau}) ,\bm{U}_{\tau+1}-\bm{U}_{\tau}\right\rangle\\
&+\sum_{i=1}^N\left\langle \nabla_{\bm{V}_{(i)}}\mathcal{L}_{(i),1}(\bm{U}_{\tau},\bm{V}_{(i),\tau})+\nabla_{\bm{V}_{(i)}}\mathcal{L}_{(i),2}(\bm{U}_{\tau},\bm{V}_{(i),\tau}) ,\bm{V}_{(i),\tau+1}-\bm{V}_{(i),\tau}\right\rangle\\
&\le -\eta_{\tau} N\norm{\square \bm{U}_{\tau}}_F^2 - \eta_{\tau}\sum_{i=1}^N\norm{\square \bm{V}_{(i),\tau}}_F^2+\eta_{\tau}^2\left(C_{6,0}N\norm{\square \bm{U}_{\tau}}_F^2+C_{6,1}\sum_{i=1}^N\norm{\square \bm{V}_{(i),\tau}}_F^2\right) 
\end{aligned}   
\end{equation}
where $C_{6,0}$ and $C_{6,1}$ are constants dependent only on $M_1$, $M_2$, $r$, and $G_{max,op}$:
$$
C_{6,0}=G_{max,op}\sqrt{r}(M_1+C_{5,0})
$$
and 
$$
C_{6,1}=G_{max,op}\sqrt{r}C_{5,1}
$$
\end{lemma}
\begin{proof}
We firstly calculate the sufficient decrease of $\bm{U}$:
\begin{equation*}
\begin{aligned}
&\left\langle \sum_{i=1}^N\nabla_{\bm{U}}\mathcal{L}_{(i),1}(\bm{U}_{\tau},\bm{V}_{(i),\tau}) ,\bm{U}_{\tau+1}-\bm{U}_{\tau}\right\rangle\\
&=\left\langle -\sum_{i=1}^N \left(\bm{I}-\bm{P}_{\bm{V}_{(i),\tau}}\right)\bm{S}_{(i)}\bm{U}_{\tau} ,\bm{U}_{\tau+1}-\bm{U}_{\tau}\right\rangle\\
&=-\left\langle \sum_{i=1}^N \left(\bm{I}-\bm{P}_{\bm{V}_{(i),\tau}}\right)\bm{S}_{(i)}\bm{U}_{\tau} ,\frac{\eta_{\tau}}{N}\sum_{i=1}^N\left(\bm{I}-\bm{P}_{\bm{V}_{(i),\tau}}-\bm{P}_{\bm{U}_{\tau}}\right)\bm{S}_{(i)}\bm{U}_{\tau} + \eta_{\tau}^2 \bm{e}_{1,\tau}\right\rangle\\
&=-\left\langle \sum_{i=1}^N \left(\bm{I}-\bm{P}_{\bm{V}_{(i),\tau}}\right)\bm{S}_{(i)}\bm{U}_{\tau} ,\frac{\eta_{\tau}}{N}\sum_{i=1}^N\left(\bm{I}-\bm{P}_{\bm{V}_{(i),\tau}}-\bm{P}_{\bm{U}_{\tau}}\right)\bm{S}_{(i)}\bm{U}_{\tau}\right\rangle \\
&-\left\langle \sum_{i=1}^N \left(\bm{I}-\bm{P}_{\bm{V}_{(i),\tau}}\right)\bm{S}_{(i)}\bm{U}_{\tau} ,\eta_{\tau}^2 \bm{e}_{1,\tau}\right\rangle\\
&=-\left\langle \sum_{i=1}^N \left(\bm{I}-\bm{P}_{\bm{V}_{(i),\tau}}-\bm{P}_{\bm{U}_{\tau}}\right)\bm{S}_{(i)}\bm{U}_{\tau} ,\frac{\eta_{\tau}}{N}\sum_{i=1}^N\left(\bm{I}-\bm{P}_{\bm{V}_{(i),\tau}}-\bm{P}_{\bm{U}_{\tau}}\right)\bm{S}_{(i)}\bm{U}_{\tau}\right\rangle \\
&-\left\langle \sum_{i=1}^N \left(\bm{I}-\bm{P}_{\bm{V}_{(i),\tau}}\right)\bm{S}_{(i)}\bm{U}_{\tau} ,\eta_{\tau}^2 \bm{e}_{1,\tau}\right\rangle\\
&\le -\eta_{\tau} N \norm{\square \bm{U}_{\tau}}_F^2+\eta_{\tau}^2 \norm{\bm{e}_{1,\tau}}_F\sum_{i=1}^N \norm{\left(\bm{I}-\bm{P}_{\bm{V}_{(i),\tau}}\right)\bm{S}_{(i)}\bm{U}_{\tau}}_F\\
&\le -\eta_{\tau} N \norm{\square \bm{U}_{\tau}}_F^2+M_1\eta_{\tau}^2 \norm{\square \bm{U}_{\tau}}_F^2\sum_{i=1}^NG_{(i),op}\sqrt{r}
\end{aligned}
\end{equation*}
Next, we calculate:
\begin{equation*}
\begin{aligned}
&\left\langle \nabla_{\bm{V}_{(i)}}\mathcal{L}_{(i),1}(\bm{U}_{\tau},\bm{V}_{(i),\tau})+\nabla_{\bm{V}_{(i)}}\mathcal{L}_{(i),2}(\bm{U}_{\tau},\bm{V}_{(i),\tau}) ,\bm{V}_{(i),\tau+1}-\bm{V}_{(i),\tau}\right\rangle\\
&=\left\langle-\left(\bm{I}-\bm{P}_{\bm{U}_{\tau}}\right)\bm{S}_{(i)}\bm{V}_{(i),\tau},\eta_{\tau} \square\bm{V}_{(i),\tau}+\eta_{\tau}\bm{U}_{\tau}\square\bm{U}_{\tau}^T\bm{V}_{(i),\tau}+\eta_{\tau}^2\bm{e}_{5,(i),\tau} \right\rangle\\
&=\left\langle-\left(\bm{I}-\bm{P}_{\bm{U}_{\tau}}\right)\bm{S}_{(i)}\bm{V}_{(i),\tau},\eta_{\tau} \square\bm{V}_{(i),\tau}\right\rangle+\left\langle-\left(\bm{I}-\bm{P}_{\bm{U}_{\tau}}\right)\bm{S}_{(i)}\bm{V}_{(i),\tau},\eta_{\tau}\bm{U}_{\tau}\square\bm{U}_{\tau}^T\bm{V}_{(i),\tau}\right\rangle\\
&+\left\langle-\left(\bm{I}-\bm{P}_{\bm{U}_{\tau}}\right)\bm{S}_{(i)}\bm{V}_{(i),\tau},\eta_{\tau}^2\bm{e}_{5,(i),\tau} \right\rangle\\
&=\left\langle-\left(\bm{I}-\bm{P}_{\bm{U}_{\tau}}-\bm{P}_{\bm{V}_{(i),\tau}}\right)\bm{S}_{(i)}\bm{V}_{(i),\tau},\eta_{\tau} \square\bm{V}_{(i),\tau}\right\rangle+\left\langle-\left(\bm{I}-\bm{P}_{\bm{U}_{\tau}}\right)\bm{S}_{(i)}\bm{V}_{(i),\tau},\eta_{\tau}^2\bm{e}_{5,(i),\tau} \right\rangle\\
&\le -\eta_{\tau} \norm{\square \bm{V}_{(i),\tau}}_F^2+\eta_{\tau}^2\norm{\left(\bm{I}-\bm{P}_{\bm{U}_{\tau}}\right)\bm{S}_{(i)}\bm{V}_{(i),\tau}}_F\norm{\bm{e}_{5,(i),\tau}}_F\\
&\le -\eta_{\tau} \norm{\square \bm{V}_{(i),\tau}}_F^2+\eta_{\tau}^2G_{(i),op}\sqrt{r}\left(C_{5,0}\norm{\square \bm{U}_{\tau}}_F^2+C_{5,1}\norm{\square \bm{V}_{(i),\tau}}^2\right)
\end{aligned}
\end{equation*}

Adding them, we have:
\begin{equation*}
\begin{aligned}
&\left\langle \sum_{i=1}^N\nabla_{\bm{U}}\mathcal{L}_{(i),1}(\bm{U}_{\tau},\bm{V}_{(i),\tau}) ,\bm{U}_{\tau+1}-\bm{U}_{\tau}\right\rangle\\
&+\left\langle \nabla_{\bm{V}_{(i)}}\mathcal{L}_{(i),1}(\bm{U}_{\tau},\bm{V}_{(i),\tau})+\nabla_{\bm{V}_{(i)}}\mathcal{L}_{(i),2}(\bm{U}_{\tau},\bm{V}_{(i),\tau}) ,\bm{V}_{(i),\tau+1}-\bm{V}_{(i),\tau}\right\rangle\\
&\le -\eta_{\tau} N \norm{\square \bm{U}_{\tau}}_F^2-\sum_{i=1}^N \eta_{\tau} \norm{\square \bm{V}_{(i),\tau}}_F^2 +\eta_{\tau}^2\left(NC_{6,0}\norm{\square \bm{U}_{\tau}}_F^2+C_{6,1}\sum_{i=1}^N\norm{\square \bm{V}_{(i),\tau}}_F^2\right) 
\end{aligned}
\end{equation*}
where the constants are:
$$
C_{6,0}=G_{max,op}\sqrt{r}(M_1+C_{5,0})
$$
and 
$$
C_{6,1}=G_{max,op}\sqrt{r}C_{5,1}
$$
\end{proof}

Finally, we come to the proof of Theorem \ref{thm:sublinearconvergence}. 
\begin{proof}
We choose constant a stepsize $\eta_{\tau}=\eta_1$ small enough:
\begin{equation}
\label{eqn:etacdefinition}
\begin{aligned}
\eta_1\le\eta_c=&\min\Big\{\frac{1}{2C_{6,0}+L\left(\left(\left(1+\frac{M_2}{2}\right)^2+2\left(1+\frac{C_{5,0}}{2}\right)^2\right)\right)},\frac{1}{2C_{6,1}+2L\left(1+\frac{C_{5,1}}{2}\right)^2},\\
&\frac{1}{G_{max,op}\sqrt{r}}\frac{M_3}{2},\frac{1}{G_{max,op}\sqrt{r}}\frac{\sqrt{M_3}}{\sqrt{6+12M_1+M_1^2}}\Big\}
\end{aligned}
\end{equation}
Obviously, $\eta_{1}$ satisfies the requirement in Lemma \ref{lm:fullupdateofv} and \ref{lm:suffcientdecrease}.

By the property of Lipschitz continuity, we have:
$$
\begin{aligned}
&\mathcal{L}_{(i)}\left(\bm{U}_{\tau+1},\bm{V}_{(i),\tau+1}\right)\le \mathcal{L}_{(i)}\left(\bm{U}_{\tau},\bm{V}_{(i),\tau}\right)\\
&+\left\langle \nabla_{\bm{U}}\mathcal{L}_{(i)}\left(\bm{U}_{\tau},\bm{V}_{(i),\tau}\right), \bm{U}_{\tau+1}-\bm{U}_{\tau}\right\rangle+\left\langle \nabla_{\bm{U}}\mathcal{L}_{(i)}\left(\bm{U}_{\tau},\bm{V}_{(i),\tau}\right), \bm{V}_{(i),\tau+1}-\bm{V}_{(i),\tau}\right\rangle\\
&+\frac{L}{2}\left(\norm{\bm{V}_{(i),\tau+1}-\bm{V}_{(i),\tau}}_F^2+\norm{\bm{U}_{\tau+1}-\bm{U}_{\tau}}_F^2\right)
\end{aligned}
$$
where $L$ is defined in \eqref{eqn:lipconstantdef}. Since $\bm{U}_{\tau+1}^T\bm{V}_{(i),\tau+1}=0$ and $\bm{U}_{\tau}^T\bm{V}_{(i),\tau}=0$, we know that 
$$
\mathcal{L}_{(i)}\left(\bm{U}_{\tau+1},\bm{V}_{(i),\tau+1}\right)=-f_i\left(\bm{U}_{\tau+1},\bm{V}_{(i),\tau+1}\right)
$$
and that 
$$
\mathcal{L}_{(i)}\left(\bm{U}_{\tau},\bm{V}_{(i),\tau}\right)=-f_i\left(\bm{U}_{\tau},\bm{V}_{(i),\tau}\right)
$$
Then, summing up both sides for $n$ from 1 to $N$, we have:
$$
\begin{aligned}
&-f\left(\bm{U}_{\tau+1},\{\bm{V}_{(i),\tau+1}\}\right)\le -f\left(\bm{U}_{\tau},\{\bm{V}_{(i),\tau}\}\right)\\
&+\left\langle \sum_{i=1}^N\nabla_{\bm{U}}\mathcal{L}_{(i)}\left(\bm{U}_{\tau},\bm{V}_{(i),\tau}\right), \bm{U}_{\tau+1}-\bm{U}_{\tau}\right\rangle+\sum_{i=1}^N\left\langle \nabla_{\bm{U}}\mathcal{L}_{(i)}\left(\bm{U}_{\tau},\bm{V}_{(i),\tau}\right), \bm{V}_{(i),\tau+1}-\bm{V}_{(i),\tau}\right\rangle\\
&+\sum_{i=1}^N\frac{L}{2}\left(\norm{\bm{V}_{(i),\tau+1}-\bm{V}_{(i),\tau}}_F^2+\norm{\bm{U}_{\tau+1}-\bm{U}_{\tau}}_F^2\right)
\end{aligned}
$$

From equation \eqref{eqn:fullupdateofu}, we know
$$
\begin{aligned}
&\norm{\bm{U}_{\tau+1}-\bm{U}_{\tau}}_F\\
&=\norm{\eta_{\tau}\square \bm{U}_{\tau}+\eta_{\tau}^2\bm{e}_{1,\tau}}_F\\
&\le \eta_{\tau}\norm{\square \bm{U}_{\tau}}_F+\norm{\eta_{\tau}^2\bm{e}_{1,\tau}}_F\\
&\le \eta_{\tau}\norm{\square \bm{U}_{\tau}}_F+\eta_{\tau}^2M_1\norm{\square \bm{U}_{\tau}}_F^2\\
&\le \eta_{\tau}(1+\frac{M_1}{2})\norm{\square \bm{U}_{\tau}}_F
\end{aligned}
$$
Similarly, from Lemma \ref{lm:fullupdateofv}, we have:
$$
\begin{aligned}
&\norm{\bm{V}_{(i),\tau+1}-\bm{V}_{(i),\tau}}_F\\
&=\norm{\eta_{\tau}\square \bm{V}_{(i),\tau}+\eta_{\tau}\bm{U}_{\tau}\square\bm{U}_{\tau}^T\bm{V}_{(i),\tau}+\eta_{\tau}^2\bm{e}_{5,(i),\tau}}_F\\
&\le \eta_{\tau}\norm{\square \bm{V}_{(i),\tau}}_F+\eta_{\tau}\norm{\bm{U}_{\tau}\square\bm{U}_{\tau}^T\bm{V}_{(i),\tau}}_F+\eta_{\tau}^2\norm{\bm{e}_{5,(i),\tau}}_F\\
&\le \eta_{\tau}\norm{\square \bm{V}_{(i),\tau}}_F+\eta_{\tau}\norm{\square \bm{U}_{\tau}}_F+\eta_{\tau}^2\norm{\bm{e}_{5,(i),\tau}}_F\\
&\le \eta_{\tau}\norm{\square \bm{U}_{\tau}}_F\left(1+C_{5,0}\eta_{\tau}\norm{\square \bm{U}_{\tau}}_F\right)+\eta_{\tau}\norm{\square \bm{V}_{(i),\tau}}_F\left(1+C_{5,1}\eta_{\tau}\norm{\square \bm{V}_{(i),\tau}}_F\right)\\
&\le \eta_{\tau}\norm{\square \bm{V}_{(i),\tau}}_F\left(1+\frac{1}{2}C_{5,1}\right)+\eta_{\tau}\norm{\square \bm{U}_{\tau}}_F\left(1+\frac{1}{2}C_{5,0}\right)
\end{aligned}
$$

Combining the two inequalities and Lemma \ref{lm:suffcientdecrease}, we have:
\begin{equation}
\label{eqn:decreaseinequality}
\begin{aligned}
&-f\left(\bm{U}_{\tau+1},\{\bm{V}_{(i),\tau+1}\}\right)\\
&\le -f\left(\bm{U}_{\tau},\{\bm{V}_{(i),\tau}\}\right)-\eta_{\tau}\left(N\norm{\square\bm{U}_{\tau}}_F^2+\sum_{i=1}^N\norm{\square \bm{V}_{(i),\tau}}_F^2\right)\\
&+\eta_{\tau}^2NC_{6,0}\norm{\square\bm{U}_{\tau}}_F^2+\eta_{\tau}^2\sum_{i=1}^NC_{6,1}\norm{\square \bm{V}_{(i),\tau}}_F^2\\
&+\eta_{\tau}^2\frac{L}{2}\sum_{i=1}^N\left(\left(\left(1+\frac{M_2}{2}\right)^2+2\left(1+\frac{C_{5,0}}{2}\right)^2\right)\norm{\square \bm{U}_{\tau}}_F^2+2\left(1+\frac{C_{5,1}}{2}\right)^2\norm{\square \bm{V}_{(i),\tau}}_F^2\right)\\
&\le -f\left(\bm{U}_{\tau},\{\bm{V}_{(i),\tau}\}\right)-\frac{\eta_{\tau}}{2}\left(N\norm{\square\bm{U}_{\tau}}_F^2+\sum_{i=1}^N\norm{\square \bm{V}_{(i),\tau}}_F^2\right)\\
\end{aligned}
\end{equation}
Summing up both sides for $\tau$ from $1$ to $R$ and rearranging terms, we have:
$$
\frac{\eta_{1}}{2}\sum_{\tau=1}^R\left(N\norm{\square\bm{U}_{\tau}}_F^2+\sum_{i=1}^N\norm{\square \bm{V}_{(i),\tau}}_F^2\right)\le -f(\bm{U}_{1},\{\bm{V}_{(i),1}\})+f(\bm{U}_{R+1},\{\bm{V}_{(i),R+1}\})
$$
As a result,
$$
\min_{\tau\in\{1,\cdots,N\}}\sum_{\tau=1}^R\left(N\norm{\square\bm{U}_{\tau}}_F^2+\sum_{i=1}^N\norm{\square \bm{V}_{(i),\tau}}_F^2\right)\le \frac{2\left(f(\bm{U}_{R+1},\{\bm{V}_{(i),R+1}\})-f(\bm{U}_{1},\{\bm{V}_{(i),1}\})\right)}{R\eta_{1}}
$$
This completes the proof of Theorem \ref{thm:sublinearconvergence}. Notice that $C_{6,0}$, $C_{6,1}$, and $L$ are of the order $G_{max,op}\sqrt{r}$, thus the requirement on $\eta_c$ in equation \eqref{eqn:etacdefinition} becomes:
$$
\eta_c=C_{\eta}\frac{1}{G_{max,op}\sqrt{r}}
$$
where $C_{\eta}$ is a constant that only depends on $M_1$, $M_2$, and $M_3$ from the generalized retraction Definition \ref{def:generalizedretraction}.
\end{proof}

\section{\revise{Proof for local linear convergence}}
\label{ap:proofforlinearconvergence}
In this section, we will show the full proof of Theorem \ref{thm:linearconvergence}. A formal theorem is stated below. 
\begin{theorem}
\label{thm:formallinearconvergence} (Formal version of theorem \ref{thm:linearconvergence}) Under assumptions \ref{ass:identifiability}, \ref{ass:snormupperbound}, and \ref{ass:noiseless}, if the difference between the population and sample covariance is small $\sqrt{\sum_{i=1}^N\norm{\matS_{(i)}-\matSigma_{(i)}}_F^2}\le \min\{\frac{\sqrt{2}-1}{4}\mu\theta^{3/2},\frac{\mu^2\theta^2}{128^2\times 2\gmop}\}$ and $\norm{\matS_{(i)}-\matSigma_{(i)}}\le \gmop$, when we initialize close to the global optimum $\phi_0\le \phi_{\tau}\le\frac{\mu^3\theta^3}{411041792\gmop^2}$, and choose a constant stepsize $\eta_t = \eta= O\left(\frac{1}{G_{op,max}\sqrt{r}}\right)$, then Algorithm \ref{alg:consensus} with choice 1 will converge into the global optimum:
$$
f(\hmatU,\{\hmatV_{(i)}\}) - f(\bm{U}_{R},\{\bm{V}_{(i),R}\})  = O\left(\left(1-\eta \frac{\mu\theta}{32}\right)^{R}\right) 
$$
where $\{\hmatU,\{\hmatV_{(i)}\}\}$ is a set of optimal solutions to problem \eqref{eqn:problem}.

Furthermore, we can recover the exact global optimal solutions:
$$
\norm{\Pj{\bm{U}_{R}}-\Pj{\hmatU_g}}_F^2+\frac{1}{N}\sum_{i=1}^N\norm{\Pj{\bm{V}_{(i),R}}-\Pj{\hmatV_{(i)}}}_F^2= O\left(\left(1-\eta \frac{\mu\theta}{32}\right)^{R}\right) 
$$
\end{theorem}

We will start by introducing needed notations, then proceed to establish some lemmas that characterize the local geometry of the optimization objective, then prove Theorem \ref{thm:linearconvergence} at the end. 

At communication round $\tau$, remember that we use $\matU_{\tau}$ and $\matV_{(i),\tau}$ to denote the updated variables. We use $\left(\hmatU, \{\hmatV_{(i)}\}\right)$ to denote one set of optimal solutions to \eqref{eqn:problem}. For simplicity, we use $\hmatPi_g$ to denote the projection  $\hmatPi_g=\hmatU\hmatU^T$ and $\hmatPi_{(i)}$ to denote the projection  $\hmatPi_{(i)}=\hmatV_{(i)}\hmatV_{(i)}^T$.

Since each covariance matrix $\matS_{(i)}$ is symmetric positive semidefinite, we can find matrix $\matF_{(i)}\in\mathbb{R}^{d\times d}$ such that $\matF_{(i)}\matF_{(i)}^T=\matS_{(i)}$ by Cholesky factorization. Furthermore, we can define $\matF_{(i),g}$, $\matF_{(i),l}$, and $\matR_{(i)}$ as,
\begin{align}
\label{eqn:deffglr}
\matF_{(i),g}&=\hmatPi_g \matF_{(i)}\notag \\
\matF_{(i),l}&=\hmatPi_{(i)} \matF_{(i)}\notag\\
\hmatF_{(i),l}&=\matF_{(i),g}+\matF_{(i),l}\notag\\
\matR_{(i)}&=\matF_{(i)} - \hmatPi_g \matF_{(i)} -\hmatPi_{(i)} \matF_{(i)}\notag\\
\end{align}
Apparently, $\matF_{(i),g}^T\matR_{(i)}=\matF_{(i),l}^T\matR_{(i)}=0$. 

Next we will introduce a set of optimal solutions $\left( \hmatU_{\tau},\{\hmatV_{(i),\tau}\}\right)$ that is close to the current updates $(\matU_{\tau},\{\matV_{(i),\tau}\}$. The variables $\hmatU_{\tau}$ and $\hmatV_{(i),\tau}$'s are defined as
\begin{equation}
\label{eqn:defofutaustar}
\hmatU_{\tau} = \hmatPi_{g}\bm{U}_{\tau}\left((\bm{U}_{\tau})^T\hmatPi_{g}\bm{U}_{\tau}\right)^{-1/2}
\end{equation}
and 
\begin{equation}
\label{eqn:defofvntaustar}
\hmatV_{(i),\tau} = \hmatPi_{(i)}\bm{V}_{(i),\tau}\left(\bm{V}_{(i),\tau}^T\hmatPi_{(i)}\bm{V}_{(i),\tau}\right)^{-1/2}
\end{equation}
for each $n=1,...N$. 

It's easy to verify that 
$$
\hmatU_{\tau}(\hmatU_{\tau})^T=\hmatPi_g
$$
and that
$$
\hmatV_{(i),\tau}(\hmatV_{(i),\tau})^T=\hmatPi_{(i)}
$$
Notice that $\{\hmatU_{\tau},\{\hmatV_{(i),\tau}\}\}$
is one set of global optimal solutions that is dependent on the iteration index $\tau$. The $\hmatU_{\tau}$ and $\hmatV_{(i),\tau}$'s are dependent on the communication round $\tau$. We use $\Delta \bm{U}_{\tau}$ to denote the difference between $\bm{U}_{\tau}$ and $\hmatU_{\tau}$:
\begin{equation}
\label{eqn:defofdeltau}
\Delta \bm{U}_{\tau} = \bm{U}_{\tau}-\hmatU_{\tau}
\end{equation}
and similarly:
\begin{equation}
\label{eqn:defofdeltavi}
\Delta \bm{V}_{(i),\tau} = \bm{V}_{(i),\tau}-\hmatV_{(i),\tau}
\end{equation}

Since  $\hmatU_{\tau}$ and $\hmatV_{(i),\tau}$'s are optimal, we can simplify the KKT conditions in \eqref{eqn:kktcondition} as
\begin{align}
\label{eqn:kktwithf}
&\matR_{(i)}\matF_{(i)}^T\hmatV_{(i),\tau}=0,\quad\forall i \in [N]\notag \\
&\sum_{i=1}^N\matR_{(i)}\matF_{(i)}^T\hmatU_{\tau}=0 
\end{align}
We can replace $\matF_{(i)}$ by $\hmatF_{(i)}$ in \eqref{eqn:kktwithf} since $\matF_{(i)}^T\hmatVit=\hmatF_{(i)}^T\hmatVit$ and $\matF_{(i)}^T\hmatUt=\hmatF_{(i)}^T\hmatUt$.

We will first show some properties of the introduced variables.
\begin{lemma}
\label{lm:hatvariables}
Under the same conditions as Theorem \ref{thm:linearconvergence}, there exists constants $\htheta=\frac{\theta}{\sqrt{2}}$, $\hmu=\frac{\mu}{\sqrt{2}}$, such that the following holds,
\begin{enumerate}
    \item \label{claim:fnormupperbound} $\norm{\hmatF_{(i)}}\le \sqrt{2\gmop}$.
    \item \label{claim:fsvlowerbound} The smallest nonzero eigenvalue of $\hmatF_{(i)}\hmatF_{(i)}^T$ is lower bounded by $\hmu$.
    \item \label{claim:thetahat} $\norm{\sum_{i=1}^N\frac{1}{N}\hmatPi_{(i)}}\le 1-\htheta$.
    \item \label{claim:rinormupperbound} $\norm{\matR_{(i)}}\le \frac{\hmu\htheta}{64\sqrt{2\gmop}}=\frac{\mu\theta}{128\sqrt{2\gmop}}$
\end{enumerate}
\end{lemma}
We will use the $\htheta$ and $\hmu$ notations in the remaining parts of the section.
\begin{proof}
We will prove the claims one by one. 

From \eqref{proofeqn:unionspaceclose}, we know,
\begin{align*}
\frac{1}{N}\sum_{i=1}^N \norm{\matPi_g+\matPi_{(i)}-\hmatPi_g-\hmatPi_{(i)}}_F^2\le \frac{4}{N\delta^2} \sum_{i=1}^N\norm{\bm{S}_{(i)}-\bm{\Sigma}_{(i)}}_F^2   
\end{align*}
where we replace $\delta$ by $\mu$ since $\left(\matI-\matPi_g-\matPi_{(i)}\right)\matSigma_{(i)}=0$.

Therefore, the difference between $\hmatF_{(i)}\hmatF_{(i)}^T$ and $\matSigma_{(i)}$ is upper bounded by,
\begin{align*}
&\norm{\hmatF_{(i)}\hmatF_{(i)}^T-\matSigma_{(i)}}  \\
&=\norm{\left(\hmatPi_g+\hmatPi_{(i)}\right)\matS_{(i)}\left(\hmatPi_g+\hmatPi_{(i)}\right)-\matSigma_{(i)}} \\
&\le\norm{\left(\hmatPi_g+\hmatPi_{(i)}\right)\matSigma_{(i)}\left(\hmatPi_g+\hmatPi_{(i)}\right)-\matSigma_{(i)}}+\norm{\left(\hmatPi_g+\hmatPi_{(i)}\right)\left(\matSigma_{(i)}-\matS_{(i)}\right)\left(\hmatPi_g+\hmatPi_{(i)}\right)} \\
&\le 2\norm{\hmatPi_g+\hmatPi_{(i)}-\matPi_g-\matPi_{(i)}}_F\norm{\matSigma_{(i)}}+\norm{\matSigma_{(i)}-\matS_{(i)}}\\
&\le \frac{8\gmop}{\mu}\sqrt{\sum_{j=1}^N\norm{\matSigma_{(j)}-\matS_{(j)}}_F^2}+\norm{\matSigma_{(i)}-\matS_{(i)}}\\
&\le \frac{\mu^2\theta^2}{128^2\times 2 \gmop}
\end{align*}

From Weyl's theorem, we know,
\begin{align*}
&\norm{\hmatF_{(i)}}^2=\norm{\hmatF_{(i)}\hmatF_{(i)}^T}\\
&\le \norm{\hmatF_{(i)}\hmatF_{(i)}^T-\matSigma_{(i)}}+\norm{\matSigma_{(i)}}\\
&\le 2\gmop
\end{align*}
This proves Claim \ref{claim:fnormupperbound}.

Also by Weyl's theorem, we know,
\begin{align*}
&\lambda_{2r}\left(\hmatF_{(i)}\hmatF_{(i)}^T\right)\\
&\ge \lambda_{2r}\left(\hmatF_{(i)}\hmatF_{(i)}^T-\matSigma_{(i)}\right)-\norm{\hmatF_{(i)}\hmatF_{(i)}^T-\matSigma_{(i)}}\\
&\ge \mu - \frac{\mu^2\theta^2}{32768 \gmop} \\
&\ge \mu\frac{1}{\sqrt{2}}
\end{align*}
This proves Claim \ref{claim:fsvlowerbound}.

Next, we consider the results from Theorem \ref{thm:statisticalerror}, 
\begin{align*}
\norm{\bm{P}_{\hat{\bm{U}}}-\matPi_g}_F^2+\frac{1}{N}\sum_{i=1}^N\norm{\bm{P}_{\hat{\bm{V}}_{(i)}}-\matPi_{(i)}}_F^2\le \frac{8}{\theta\mu^2} \frac{1}{N}\sum_{i=1}^N\norm{ \bm{\Sigma}_{(i)}-\bm{S}_{(i)}}_{F}^2
\end{align*}
 Therefore, an upper bound for $\norm{\bm{P}_{\hat{\bm{V}}_{(i)}}-\matPi_{(i)}}_F$ is
 \begin{align*}
 &\norm{\bm{P}_{\hat{\bm{V}}_{(i)}}-\matPi_{(i)}}_F   \\
 &\le 2\sqrt{2}\frac{1}{\sqrt{\theta}\mu}\sqrt{\sum_{i=1}^N\norm{ \bm{\Sigma}_{(i)}-\bm{S}_{(i)}}_{F}^2}\\
 &\le \theta\frac{2-\sqrt{2}}{2}
 \end{align*}
where we applied the condition that $\sqrt{\sum_{i=1}^N\norm{ \bm{\Sigma}_{(i)}-\bm{S}_{(i)}}_{F}^2}\le \mu\theta^{1.5}\frac{\sqrt{2}-1}{4}$ in the last inequality.

As a result, we have,
\begin{align*}
&\norm{\frac{1}{N}\sum_{i=1}^N\hmatPi_{(i)}}\\
&\le\norm{\frac{1}{N}\sum_{i=1}^N\matPi_{(i)}}+\frac{1}{N}\sum_{i=1}^N\norm{\matPi_{(i)}-\hmatPi_{(i)}} \\
&\le 1-\theta +\theta\left(1-\frac{1}{\sqrt{2}}\right)\\
&=\frac{\theta}{\sqrt{2}}
\end{align*}
This proves Claim \ref{claim:thetahat}.

Then we analyze the norm of $\matR_{(i)}\matR_{(i)}^T$,
\begin{align*}
&\norm{\matR_{(i)}}^2=\norm{\matR_{(i)}^T\matR_{(i)}}\\
&=\norm{\left(\matI-\hmatPi_g-\hmatPi_{(i)}\right)\matS_{(i)}\left(\matI-\hmatPi_g-\hmatPi_{(i)}\right)}\\
&\le \norm{\left(\matI-\hmatPi_g-\hmatPi_{(i)}\right)\matSigma_{(i)}\left(\matI-\hmatPi_g-\hmatPi_{(i)}\right)}\\
&+\norm{\left(\matI-\hmatPi_g-\hmatPi_{(i)}\right)\left(\matS_{(i)}-\matSigma_{(i)}\right)\left(\matI-\hmatPi_g-\hmatPi_{(i)}\right)}\\
&\le \norm{\matPi_g+\matPi_{(i)}-\hmatPi_g-\hmatPi_{(i)}}_F\norm{\matSigma_{(i)}}+\norm{\matS_{(i)}-\matSigma_{(i)}}\\
&\le \sqrt{\sum_{i=1}^N\norm{ \bm{\Sigma}_{(i)}-\bm{S}_{(i)}}_{F}^2}\left(1+4\frac{\gmop}{\delta}\right)\frac{\mu\theta}{128\sqrt{2\gmop}}
\end{align*}
where the last inequality comes from the fact that $\sum_{i=1}^N\norm{ \bm{\Sigma}_{(i)}-\bm{S}_{(i)}}_{F}^2\le \frac{1}{\left(1+\frac{8\gmop}{\mu}\right)^2}\left(\frac{\mu^2\theta^2}{32768\gmop}\right)^2$.

\end{proof}

As discussed in Section \ref{sec:proofsketchsublinearconvergence}, different $\bm{U}$ and $\bm{V}_{(i)}$'s may have the same objective value, as long as they span the same column space. We introduce a variable $\zeta$ to denote the subspace distance between the estimate and ground truth:
\begin{equation}
\label{eqn:defzetai}
\zeta_{(i),\tau} = r-\innerp{\bm{P}_{\bm{V}_{(i),\tau}}}{\matPi_{(i)}}
\end{equation}
for each $i=1,...N$, and, 
\begin{equation}
\label{eqn:defzeta0}
\zeta_{(0),\tau} = r-\innerp{\bm{P}_{\bm{U}_{\tau}}}{\bm{\Pi}_{g}}
\end{equation}
We use $\zeta_{\tau}$ to denote:
\begin{equation}
\label{eqn:defzeta}
\zeta_{\tau} = \zeta_{(0),\tau}+\frac{1}{N}\sum_{i=1}^N\zeta_{(i),\tau}
\end{equation}
The $\zeta_{(0),\tau}$ and $\zeta_{(i),\tau}$'s defined represent how far away the iterates are from the ground truth, measured by subspace distance.

We can also define,
\begin{equation}
\label{eqn:deftildezetai}
\widetilde{\zeta}_{(i),\tau} = 2r-\innerp{\Pj{\bm{U}_{\tau}}+\bm{P}_{\bm{V}_{(i),\tau}}}{\hmatPi_g+\hmatPi_{(i)}}
\end{equation}
for each $i=1,...N$. We use $\wz_{\tau}$ to denote:
\begin{equation}
\label{eqn:tildedefzeta}
\wz_{\tau} = \sum_{i=1}^N\wz_{(i),\tau}
\end{equation}
From Lemma \ref{lm:sumspacetoindividualspace}, since $\hmatPi_{(i)}$'s are $\htheta$-misaligned, there exists a relation between $\zeta_{\tau}$ and $\widetilde{\zeta}_{\tau}$:
\begin{equation}
\label{eqn:relationbetweenzetaandtildezeta}
\frac{\htheta}{2} N\zeta_{\tau}\le\widetilde{\zeta}_{\tau}\le N\zeta_{\tau}
\end{equation}

For simplicity, we also define the optimality gap $\phi_{\tau}$ as,
\begin{align}
\label{eqn:defphitau}
\phi_{\tau}=\frac{1}{2}\sum_{i=1}^N\left(\tr{\hmatPi_g\matS_{(i)}}+\tr{\hmatPi_{(i)}\matS_{(i)}}-\tr{\Pj{\matU_{\tau}}\matS_{(i)}}-\tr{\Pj{\matV_{(i),\tau}}\matS_{(i)}}\right)
\end{align}

We then use the optimality gap $\phi_{\tau}$ to upper bound the norm of $\dUt$ and $\dVit$.

\begin{lemma}
\label{lm:gaptodelta}
Under the same conditions as Theorem \ref{thm:linearconvergence}, we have,
\begin{align}
\phi_{\tau}\ge \frac{\htheta \hmu}{16}\left(N\norm{\dUt}_F^2 + \sum_{i=1}^N\norm{\dVit}_F^2\right)
\end{align}
\end{lemma}

\begin{proof}
By definition of the optimality gap $\phi_{\tau}$, we have,
\begin{align*}
&2\phi_{\tau}=\sum_{i=1}^N\left(\tr{\hmatPi_g\matS_{(i)}}+\tr{\hmatPi_{(i)}\matS_{(i)}}-\tr{\Pj{\matU_{\tau}}\matS_{(i)}}-\tr{\Pj{\matV_{(i),\tau}}\matS_{(i)}}\right)\\
&=\sum_{i=1}^N\left(\norm{\matF_{(i)}-\left(\Pj{\matU_{\tau}}+\Pj{\matV_{(i),\tau}}\right)\matF_{(i)}}_F^2-\norm{\matF_{(i)}-\left(\hmatPi_g+\hmatPi_{(i)}\right)\matF_{(i)}}_F^2\right)\\
&=\sum_{i=1}^N\left(\norm{\left(\matI-\Pj{\matU_{\tau}}-\Pj{\matV_{(i),\tau}}\right)\hmatF_{(i)}+\left(\matI-\Pj{\matU_{\tau}}-\Pj{\matV_{(i),\tau}}\right)\matR_{(i)}}_F^2-\norm{\matR_{(i)}}_F^2\right)\\
&=\sum_{i=1}^N\norm{\left(\matI-\Pj{\matU_{\tau}}-\Pj{\matV_{(i),\tau}}\right)\hmatF_{(i)}}_F^2+\underbrace{2\innerp{\left(\matI-\Pj{\matU_{\tau}}-\Pj{\matV_{(i),\tau}}\right)\hmatF_{(i)}}{\left(\matI-\Pj{\matU_{\tau}}-\Pj{\matV_{(i),\tau}}\right)\matR_{(i)}}}_{ \texttt{Term I} }\\
&+\underbrace{\norm{\left(\matI-\Pj{\matU_{\tau}}-\Pj{\matV_{(i),\tau}}\right)\matR_{(i)}}_F^2-\norm{\matR_{(i)}}_F^2}_{\tm{II}}\\
\end{align*}

The above can be further simplified. For \tm{I}, we have,
\begin{align*}
&\sum_{i=1}^N2\innerp{\left(\matI-\Pj{\matU_{\tau}}-\Pj{\matV_{(i),\tau}}\right)\hmatF_{(i)}}{\left(\matI-\Pj{\matU_{\tau}}-\Pj{\matV_{(i),\tau}}\right)\matR_{(i)}}\\
&=\sum_{i=1}^N2\tr{\matR_{(i)}^T\left(\matI-\Pj{\matU_{\tau}}-\Pj{\matV_{(i),\tau}}\right)\hmatF_{(i)}}\\
&=\sum_{i=1}^N2\tr{\matR_{(i)}^T\left(\hmatPi_g+\hmatPi_{(i)}-\Pj{\matU_{\tau}}-\Pj{\matV_{(i),\tau}}\right)\hmatF_{(i)}}\\
&=\sum_{i=1}^N2\tr{\matR_{(i)}^T\left(-\dUt\dUt^T-\dVit\dVit^T\right)\hmatF_{(i)}}\\
\end{align*}
where we have applied the KKT conditions \eqref{eqn:kktwithf} that $\hmatVit^T\hmatF_{(i)}\matR_{(i)}^T=0$ and $\sum_{i=1}^N\hmatUt^T\hmatF_{(i)}\matR_{(i)}^T=0$ in the third equality.

For term \tm{II}, we can also derive
\begin{align*}
&\norm{\left(\matI-\Pj{\matU_{\tau}}-\Pj{\matV_{(i),\tau}}\right)\matR_{(i)}}_F^2-\norm{\matR_{(i)}}_F^2\\
&=\tr{\matR_{(i)}^T\left(\matI-\Pj{\matU_{\tau}}-\Pj{\matV_{(i),\tau}}\right)\matR_{(i)}}-\tr{\matR_{(i)}^T\matR_{(i)}}\\
&=\tr{\matR_{(i)}^T\left(\hmatPi_g+\hmatPi_{(i)}-\Pj{\matU_{\tau}}-\Pj{\matV_{(i),\tau}}\right)\matR_{(i)}}\\
&=-\tr{\matR_{(i)}^T\left(\dUt\dUt^T+\dVit\dVit^T\right)\matR_{(i)}}
\end{align*}
where we used the condition $\hmatUt^T\matR_{(i)}=\hmatVit^T\matR_{(i)}=0$ in the last equality.

Combining these, we have,
\begin{align}
\label{eqn:optimalitygapbound}
&2\phi_{\tau}\notag\\
&=\sum_{i=1}^N\Big(\norm{\left(\matI-\Pj{\matU_{\tau}}-\Pj{\matV_{(i),\tau}}\right)\hmatF_{(i)}}_F^2-\tr{\matR_{(i)}^T\left(\dUt\dUt^T+\dVit\dVit^T\right)\matR_{(i)}}\notag \\
&-2\tr{\matR_{(i)}^T\left(\dUt\dUt^T+\dVit\dVit^T\right)\hmatF_{(i)}}\Big)
\end{align}

From Lemma \ref{lm:hatvariables}, we know that $\hmatF_{(i)}\hmatF_{(i)}^T\succeq \hmu \left(\hmatPi_g+\hmatPi_{(i)}\right)$, thus 
\begin{align*}
&\sum_{i=1}^N\tr{\left(\matI-\Pj{\matU_{\tau}}-\Pj{\matV_{(i),\tau}}\right)\hmatF_{(i)}\hmatF_{(i)}^T}\ge \sum_{i=1}^N\tr{\left(\matI-\Pj{\matU_{\tau}}-\Pj{\matV_{(i),\tau}}\right)\left(\hmatPi_g+\hmatPi_{(i)}\right)}\hmu \\
&\ge \frac{\hmu\htheta}{2}\sum_{i=1}^N\left(r-\tr{\hmatPi_g\Pj{\matUt}}+r-\tr{\hmatPi_{(i)}\Pj{\matVit}}\right) \\
&\ge \frac{\hmu\htheta}{4}\sum_{i=1}^N \left(\norm{\dUt}_F^2+\norm{\dVit}_F^2\right)
\end{align*}
By Cauchy-Schwartz inequality, we have,
\begin{align*}
&\tr{\matR_{(i)}^T\left(\dUt\dUt^T+\dVit\dVit^T\right)\matR_{(i)}}\\
&=\tr{\matR_{(i)}^T\dUt\dUt^T\matR_{(i)}}+\tr{\matR_{(i)}^T\dVit\dVit^T\matR_{(i)}}\\
&\le \norm{\dUt}_F^2\norm{\matR_{(i)}}^2+\norm{\dVit}_F^2\norm{\matR_{(i)}}^2
\end{align*}
and
\begin{align*}
&2\tr{\matR_{(i)}^T\left(\dUt\dUt^T+\dVit\dVit^T\right)\hmatF_{(i)}}\Big)\\
&\le 2\norm{\dUt}_F^2\norm{\matR_{(i)}}\norm{\hmatF_{(i)}}+2\norm{\dVit}_F^2\norm{\matR_{(i)}}\norm{\hmatF_{(i)}}
\end{align*}

Since $\norm{\matR_{(i)}}\le \frac{\hmu\htheta}{64\sqrt{2\gmop}}$ and $\norm{\hmatF_{(i)}}\le \sqrt{2\gmop}$, we have
\begin{align*}
&\norm{\matR_{(i)}}^2+2\norm{\matR_{(i)}}\norm{\hmatF_{(i)}}\\
&\le 3\norm{\matR_{(i)}}\norm{\hmatF_{(i)}}\\
&\le \frac{\hmu\htheta}{8}
\end{align*}
Thus we have,
\begin{align*}
2\phi_{\tau}\ge\frac{\hmu\htheta}{8}\sum_{i=1}^N \left(\norm{\dUt}_F^2+\norm{\dVit}_F^2\right)
\end{align*}
This completes our proof.
\end{proof}

Next we will provide a lemma that characterizes the landscape of the objective.

\begin{lemma}
\label{lm:gradientalign}
Under the same conditions as Theorem \ref{thm:linearconvergence}, we have,
\begin{align*}
&-\innerp{\sum_{i=1}^N\left(\matI-\Pj{\matU_{\tau}}-\Pj{\matV_{(i),\tau}}\right)\matS_{(i)}\matU_{\tau}}{\dU_{\tau}}-\sum_{i=1}^N\innerp{\left(\matI-\Pj{\matU_{\tau}}-\Pj{\matV_{(i),\tau}}\right)\matS_{(i)}\matV_{(i),\tau}}{\dV_{(i),\tau}}\\
&\ge \phi_{\tau} 
\end{align*}
\end{lemma}
\begin{proof}

We first consider the inner product term,
\begin{align*}
&\innerp{\left(\matI-\Pj{\matU_{\tau}}-\Pj{\matV_{(i),\tau}}\right)\matS_{(i)}\matV_{(i),\tau}}{\dV_{(i),\tau}}\\
&=\innerp{\left(\matI-\Pj{\matU_{\tau}}-\Pj{\matV_{(i),\tau}}\right)\matF_{(i)}}{\dV_{(i),\tau}\matV_{(i),\tau}^T\matF_{(i)}}\\
&=\underbrace{\innerp{\left(\matI-\Pj{\matU_{\tau}}-\Pj{\matV_{(i),\tau}}\right)\hmatF_{(i)}}{\dV_{(i),\tau}\matV_{(i),\tau}^T\hmatF_{(i)}}}_{\tm{III}}+\underbrace{\innerp{\left(\matI-\Pj{\matU_{\tau}}-\Pj{\matV_{(i),\tau}}\right)\matR_{(i)}}{\dV_{(i),\tau}\matV_{(i),\tau}^T\matR_{(i)}}}_{\tm{IV}}\\
&+\underbrace{\innerp{\left(\matI-\Pj{\matU_{\tau}}-\Pj{\matV_{(i),\tau}}\right)\hmatF_{(i)}}{\dV_{(i),\tau}\matV_{(i),\tau}^T\matR_{(i)}}}_{\tm{V}}+\underbrace{\innerp{\left(\matI-\Pj{\matU_{\tau}}-\Pj{\matV_{(i),\tau}}\right)\matR_{(i)}}{\dV_{(i),\tau}\matV_{(i),\tau}^T\hmatF_{(i)}}}_{\tm{VI}}\\
\end{align*}

We will analyze each term separately. For \tm{III}, we know that $\dVit\matVit^T=\matVit\matVit^T-\hmatVit\hmatVit-\hmatVit\dVit^T$. Therefore,
\begin{align*}
&\innerp{\left(\matI-\Pj{\matU_{\tau}}-\Pj{\matV_{(i),\tau}}\right)\hmatF_{(i)}}{\dV_{(i),\tau}\matV_{(i),\tau}^T\hmatF_{(i)}}\\
&=\innerp{\left(\matI-\Pj{\matU_{\tau}}-\Pj{\matV_{(i),\tau}}\right)\hmatF_{(i)}}{\left(\Pj{\matVit}-\hmatPi_{(i)}-\hmatVit\dVit^T\right)\hmatF_{(i)}}\\
&=\innerp{\left(\matI-\Pj{\matU_{\tau}}-\Pj{\matV_{(i),\tau}}\right)\hmatF_{(i)}}{\left(\Pj{\matVit}-\hmatPi_{(i)}\right)\hmatF_{(i)}}\\
&-\innerp{\left(\matI-\Pj{\matU_{\tau}}-\Pj{\matV_{(i),\tau}}\right)\hmatF_{(i)}}{\left(\hmatVit\dVit^T\right)\hmatF_{(i)}}\\
&=\innerp{\left(\matI-\Pj{\matU_{\tau}}-\Pj{\matV_{(i),\tau}}\right)\hmatF_{(i)}}{\left(\Pj{\matVit}-\hmatPi_{(i)}\right)\hmatF_{(i)}}+\bm{\epsilon}_{1,(i),\tau}\\
\end{align*}
where $\bm{\epsilon}_{1,(i),\tau}$ is defined as $\bm{\epsilon}_{1,(i),\tau}=-\innerp{\left(\matI-\Pj{\matU_{\tau}}-\Pj{\matV_{(i),\tau}}\right)\hmatF_{(i)}}{\left(\hmatVit\dVit^T\right)\hmatF_{(i)}}$. Its norm is upper bounded by
\begin{align*}
&\abs{\bm{\epsilon}_{1,(i),\tau}}\\
&=\abs{\tr{\hmatF_{(i)}^T\left(\matI-\Pj{\matU_{\tau}}-\Pj{\matV_{(i),\tau}}\right)\hmatVit\dVit^T\hmatF_{(i)}}}\\
&=\abs{\tr{\hmatF_{(i)}^T\left(\matI-\Pj{\matU_{\tau}}-\Pj{\matV_{(i),\tau}}\right)\dVit\dVit^T\hmatF_{(i)}}}\\
&=\abs{\tr{\hmatF_{(i)}^T\left(\hmatPi_{g}+\hmatPi_{(i)}-\Pj{\matU_{\tau}}-\Pj{\matV_{(i),\tau}}\right)\dVit\dVit^T\hmatF_{(i)}}}\\
&\le \norm{\hmatF_{(i)}^T\left(\hmatPi_{g}+\hmatPi_{(i)}-\Pj{\matU_{\tau}}-\Pj{\matV_{(i),\tau}}\right)}_F\norm{\dVit\dVit^T\hmatF_{(i)}}_F\\
&\le \norm{\hmatPi_{g}+\hmatPi_{(i)}-\Pj{\matU_{\tau}}-\Pj{\matV_{(i),\tau}}}_F\norm{\dVit}_F^2\norm{\hmatF_{(i)}}^2\\
&\le 4\sqrt{2}\sqrt{\wzit}\zit \gmop
\end{align*}

For \tm{IV}, we have,
\begin{align*}
&\innerp{\left(\matI-\Pj{\matU_{\tau}}-\Pj{\matV_{(i),\tau}}\right)\matR_{(i)}}{\dV_{(i),\tau}\matV_{(i),\tau}^T\matR_{(i)}}\\
&=\innerp{\left(\matI-\Pj{\matU_{\tau}}-\Pj{\matV_{(i),\tau}}\right)\matR_{(i)}}{\dV_{(i),\tau}\dVit^T\matR_{(i)}}\\
&=\innerp{\matR_{(i)}}{\dV_{(i),\tau}\dVit^T\matR_{(i)}}+\bm{\epsilon}_{2,(i),\tau}
\end{align*}
where $\bm{\epsilon}_{2,(i),\tau}$ is defined as
\begin{align*}
    \bm{\epsilon}_{2,(i),\tau}=-\innerp{\left(\Pj{\matU_{\tau}}+\Pj{\matV_{(i),\tau}}\right)\matR_{(i)}}{\dV_{(i),\tau}\dVit^T\matR_{(i)}}
\end{align*}
and its norm is upper bounded by,
\begin{align*}
&\abs{\bm{\epsilon}_{2,(i),\tau}}_F\\
&=\abs{\tr{\matR_{(i)}^T\left(\Pj{\matU_{\tau}}+\Pj{\matV_{(i),\tau}}\right) \dV_{(i),\tau}\dVit^T\matR_{(i)}}}\\
&=\abs{\tr{\matR_{(i)}^T\left(\Pj{\matU_{\tau}}+\Pj{\matV_{(i),\tau}}-\hmatPi_g-\hmatPi_{(i)}\right) \left(\Pj{\matU_{\tau}}+\Pj{\matV_{(i),\tau}}\right) \dV_{(i),\tau}\dVit^T\matR_{(i)}}}\\
&\le \norm{\matR_{(i)}^T\left(\Pj{\matU_{\tau}}+\Pj{\matV_{(i),\tau}}-\hmatPi_g-\hmatPi_{(i)}\right)}_F\norm{\left(\Pj{\matU_{\tau}}+\Pj{\matV_{(i),\tau}}\right) \dV_{(i),\tau}\dVit^T\matR_{(i)}}_F\\
&\le \norm{\Pj{\matU_{\tau}}+\Pj{\matV_{(i),\tau}}-\hmatPi_g-\hmatPi_{(i)}}_F\norm{\dVit}_F^2\norm{\matR_{(i)}}^2\\
&\le 2\sqrt{2}\sqrt{\wzit}\zit \gmop
\end{align*}

For the \tm{V}, 
\begin{align*}
&\innerp{\left(\matI-\Pj{\matU_{\tau}}-\Pj{\matV_{(i),\tau}}\right)\hmatF_{(i)}}{\dV_{(i),\tau}\matV_{(i),\tau}^T\matR_{(i)}}\\
&=\innerp{\left(\matI-\Pj{\matU_{\tau}}-\Pj{\matV_{(i),\tau}}\right)\hmatF_{(i)}}{\dV_{(i),\tau}\dVit^T\matR_{(i)}}\\
&=\bm{\epsilon}_{3,(i),\tau}
\end{align*}
where the norm of $\bm{\epsilon}_{3,(i),\tau}$ is upper bounded by,
\begin{align*}
&\abs{\bm{\epsilon}_{3,(i),\tau}}\\
&=\abs{\tr{\hmatF_{(i)}^T\left(\matI-\Pj{\matU_{\tau}}-\Pj{\matV_{(i),\tau}}\right)\dV_{(i),\tau}\dVit^T\matR_{(i)}}}\\
&=\abs{\tr{\hmatF_{(i)}^T\left(\hmatPi_g+\hmatPi_{(i)}-\Pj{\matU_{\tau}}-\Pj{\matV_{(i),\tau}}\right)\dV_{(i),\tau}\dVit^T\matR_{(i)}}}\\
&\le \norm{\hmatF_{(i)}^T\left(\hmatPi_g+\hmatPi_{(i)}-\Pj{\matU_{\tau}}-\Pj{\matV_{(i),\tau}}\right)}_F\norm{\dV_{(i),\tau}\dVit^T\matR_{(i)}}_F\\
&\le \norm{\Pj{\matU_{\tau}}+\Pj{\matV_{(i),\tau}}-\hmatPi_g-\hmatPi_{(i)}}_F\norm{\dVit}_F^2\norm{\matR_{(i)}}\norm{\hmatF_{(i)}}\\
&\le 4\sqrt{2}\sqrt{\wzit}\zit \gmop
\end{align*}

For \tm{VI}, 
\begin{align*}
&\innerp{\left(\matI-\Pj{\matU_{\tau}}-\Pj{\matV_{(i),\tau}}\right)\matR_{(i)}}{\dV_{(i),\tau}\matV_{(i),\tau}^T\hmatF_{(i)}}\\
&=\innerp{\left(\matI-\Pj{\matU_{\tau}}-\Pj{\matV_{(i),\tau}}\right)\matR_{(i)}}{\dV_{(i),\tau}\dVit^T\hmatF_{(i)}}\\
&=\innerp{\matR_{(i)}}{\dV_{(i),\tau}\dVit^T\hmatF_{(i)}}+\bm{\epsilon}_{4,(i),\tau}
\end{align*}
where $\bm{\epsilon}_{4,(i),\tau}$ is defined as
\begin{align*}
\bm{\epsilon}_{4,(i),\tau}= -\innerp{\left(\Pj{\matU_{\tau}}+\Pj{\matV_{(i),\tau}}\right)\matR_{(i)}}{\dV_{(i),\tau}\dVit^T\hmatF_{(i)}}  
\end{align*}
Its norm is upper bounded by,
\begin{align*}
&\abs{\bm{\epsilon}_{4,(i),\tau}}\\
&=\abs{\tr{\matR_{(i)}^T\left(\Pj{\matU_{\tau}}+\Pj{\matV_{(i),\tau}}\right)\dV_{(i),\tau}\dVit^T\hmatF_{(i)}}}\\
&=\abs{\tr{\matR_{(i)}^T\left(-\matPi_g-\matPi_{(i)}+\Pj{\matU_{\tau}}+\Pj{\matV_{(i),\tau}}\right)\dV_{(i),\tau}\dVit^T\hmatF_{(i)}}}\\
&\le \norm{\dVit}_F^2\norm{\matR_{(i)}}\norm{\hmatF_{(i)}}\norm{-\matPi_g-\matPi_{(i)}+\Pj{\matU_{\tau}}+\Pj{\matV_{(i),\tau}}}_F\\
&\le 4\sqrt{2} \sqrt{\wzit}\zit\gmop
\end{align*}

Combining these terms, we have,
\begin{align}
\label{eqn:grddv}
&\innerp{\left(\matI-\Pj{\matU_{\tau}}-\Pj{\matV_{(i),\tau}}\right)\matS_{(i)}\matV_{(i),\tau}}{\dV_{(i),\tau}}\notag\\
&=\innerp{\left(\matI-\Pj{\matU_{\tau}}-\Pj{\matV_{(i),\tau}}\right)\hmatF_{(i)}}{\left(\Pj{\matVit}-\hmatPi_{(i)}\right)\hmatF_{(i)}}+\innerp{\matR_{(i)}}{\dVit\dVit^T\matR_{(i)}}\notag\\
&+\innerp{\matR_{(i)}}{\dVit\dVit^T\hmatF_{(i)}}\notag\\
&+\bm{\epsilon}_{1,(i),\tau}+\bm{\epsilon}_{2,(i),\tau}+\bm{\epsilon}_{3,(i),\tau}+\bm{\epsilon}_{4,(i),\tau}
\end{align}

Similarly, we can calculate the inner product term,
\begin{align*}
&\sum_{i=1}^N\innerp{\left(\matI-\Pj{\matU_{\tau}}-\Pj{\matV_{(i),\tau}}\right)\matS_{(i)}\matU_{\tau}}{\dU_{\tau}}\\
&=\sum_{i=1}^N\innerp{\left(\matI-\Pj{\matU_{\tau}}-\Pj{\matV_{(i),\tau}}\right)\matF_{(i)}}{\dU_{\tau}\matUt^T\matF_{(i)}}\\
&=\sum_{i=1}^N\underbrace{\innerp{\left(\matI-\Pj{\matU_{\tau}}-\Pj{\matV_{(i),\tau}}\right)\hmatF_{(i)}}{\dU_{\tau}\matUt^T\hmatF_{(i)}}}_{\tm{VII}}+\underbrace{\innerp{\left(\matI-\Pj{\matU_{\tau}}-\Pj{\matV_{(i),\tau}}\right)\matR_{(i)}}{\dU_{\tau}\matUt^T\matR_{(i)}}}_{\tm{VIII}}\\
&+\underbrace{\innerp{\left(\matI-\Pj{\matU_{\tau}}-\Pj{\matV_{(i),\tau}}\right)\hmatF_{(i)}}{\dU_{\tau}\matUt^T\matR_{(i)}}}_{\tm{IX}}+\underbrace{\innerp{\left(\matI-\Pj{\matU_{\tau}}-\Pj{\matV_{(i),\tau}}\right)\matR_{(i)}}{\dU_{\tau}\matUt^T\hmatF_{(i)}}}_{\tm{X}}
\end{align*}

For the \tm{VII}, we can simplify it as,
\begin{align*}
&\innerp{\left(\matI-\Pj{\matU_{\tau}}-\Pj{\matV_{(i),\tau}}\right)\hmatF_{(i)}}{\dU_{\tau}\matUt^T\hmatF_{(i)}}\\
&=\innerp{\left(\matI-\Pj{\matU_{\tau}}-\Pj{\matV_{(i),\tau}}\right)\hmatF_{(i)}}{\left(\Pj{\matUt}-\hmatPi_g\right)\hmatF_{(i)}}+\bm{\epsilon}_{5,(i),\tau}
\end{align*}
where $\bm{\epsilon}_{5,(i),\tau}$ is defined as,
\begin{align*}
&\bm{\epsilon}_{5,(i),\tau}=-\innerp{\left(\matI-\Pj{\matU_{\tau}}-\Pj{\matV_{(i),\tau}}\right)\hmatF_{(i)}}{\hmatUt\dUt^T\hmatF_{(i)}}\\
&=\innerp{\left(\matI-\Pj{\matU_{\tau}}-\Pj{\matV_{(i),\tau}}\right)\hmatF_{(i)}}{\dUt\dUt^T\hmatF_{(i)}}\\
\end{align*}
Its norm is upper bounded by
\begin{align*}
&\abs{\bm{\epsilon}_{5,(i),\tau}}_F\\
&=\abs{\tr{\matF_{(i)}^T\left(\matI-\Pj{\matU_{\tau}}-\Pj{\matV_{(i),\tau}}\right)\dUt\dUt^T\hmatF_{(i)} }}\\
&=\abs{\tr{\hmatF_{(i)}^T\left(\hmatPi_g+\hmatPi_{(i)}-\Pj{\matU_{\tau}}-\Pj{\matV_{(i),\tau}}\right)\dUt\dUt^T\hmatF_{(i)} }}\\
&\le \norm{\hmatF_{(i)}^T\left(\hmatPi_g+\hmatPi_{(i)}-\Pj{\matU_{\tau}}-\Pj{\matV_{(i),\tau}}\right)}_F\norm{\dUt\dUt^T\hmatF_{(i)}}_F\\
&\le \norm{\hmatPi_g+\hmatPi_{(i)}-\Pj{\matU_{\tau}}-\Pj{\matV_{(i),\tau}}}_F\norm{\dUt}_F^2\norm{\hmatF_{(i)}}^2\\
&\le 4\sqrt{2}\sqrt{\wzit}\zzt \gmop
\end{align*}

For \tm{VIII}, also we have,
\begin{align*}
&\innerp{\left(\matI-\Pj{\matU_{\tau}}-\Pj{\matV_{(i),\tau}}\right)\matR_{(i)}}{\dU_{\tau}\matUt^T\matR_{(i)}}\\
&\innerp{\left(\matI-\Pj{\matU_{\tau}}-\Pj{\matV_{(i),\tau}}\right)\matR_{(i)}}{\dU_{\tau}\dUt^T\matR_{(i)}}\\
&=\innerp{\matR_{(i)}}{\dU_{\tau}\dUt^T\matR_{(i)}}+\bm{\epsilon}_{6,(i),\tau}
\end{align*}
where $\bm{\epsilon}_{6,(i),\tau}$ is defined as,
\begin{align*}
\bm{\epsilon}_{6,(i),\tau}=-\innerp{\left(\Pj{\matU_{\tau}}+\Pj{\matV_{(i),\tau}}\right)\matR_{(i)}}{\dUt\dUt^T\matR_{(i)}}
\end{align*}
and its norm is upper bounded by,
\begin{align*}
&\abs{\bm{\epsilon}_{6,(i),\tau}}_F\\
&=\abs{\tr{\matR_{(i)}^T\left(\Pj{\matU_{\tau}}+\Pj{\matV_{(i),\tau}}\right)\dUt\dUt^T\matR_{(i)}}}\\
&=\abs{\tr{\matR_{(i)}^T\left(\Pj{\matU_{\tau}}+\Pj{\matV_{(i),\tau}}-\hmatPi_g-\hmatPi_{(i)}\right)\left(\Pj{\matU_{\tau}}+\Pj{\matV_{(i),\tau}}\right)\dUt\dUt^T\matR_{(i)}}}\\
&\le \norm{\matR_{(i)}^T\left(\Pj{\matU_{\tau}}+\Pj{\matV_{(i),\tau}}-\hmatPi_g-\hmatPi_{(i)}\right)}_F\norm{\left(\Pj{\matU_{\tau}}+\Pj{\matV_{(i),\tau}}\right)\dUt\dUt^T\matR_{(i)}}_F\\
&\le \norm{\Pj{\matU_{\tau}}+\Pj{\matV_{(i),\tau}}-\hmatPi_g-\hmatPi_{(i)}}_F\norm{\dUt}_F^2\norm{\matR_{(i)}}^2\\
&\le 2\sqrt{2}\sqrt{\wzit}\zzt \gmop
\end{align*}

For \tm{IX}, we have,
\begin{align*}
&\innerp{\left(\matI-\Pj{\matU_{\tau}}-\Pj{\matV_{(i),\tau}}\right)\hmatF_{(i)}}{\dU_{\tau}\matUt^T\matR_{(i)}}\\
&=\innerp{\left(\matI-\Pj{\matU_{\tau}}-\Pj{\matV_{(i),\tau}}\right)\hmatF_{(i)}}{\dU_{\tau}\dUt^T\matR_{(i)}}\\
&=\bm{\epsilon}_{7,(i),\tau}
\end{align*}
where the norm of $\bm{\epsilon}_{7,(i),\tau}$ is upper bounded by,
\begin{align*}
&\abs{\bm{\epsilon}_{7,(i),\tau}}\\
&=\abs{\tr{\hmatF_{(i)}^T\left(\matI-\Pj{\matU_{\tau}}-\Pj{\matV_{(i),\tau}}\right)\dU_{\tau}\dUt^T\matR_{(i)}}}\\
&=\abs{\tr{\hmatF_{(i)}^T\left(\hmatPi_g+\hmatPi_{(i)}-\Pj{\matU_{\tau}}-\Pj{\matV_{(i),\tau}}\right)\dU_{\tau}\dUt^T\matR_{(i)}}}\\
&\le \norm{\hmatF_{(i)}^T\left(\hmatPi_g+\hmatPi_{(i)}-\Pj{\matU_{\tau}}-\Pj{\matV_{(i),\tau}}\right)}_F\norm{\dU_{\tau}\dUt^T\matR_{(i)}}_F\\
&\le \norm{\hmatPi_g+\hmatPi_{(i)}-\Pj{\matU_{\tau}}-\Pj{\matV_{(i),\tau}}}_F\norm{\dUt}_F^2\norm{\matR_{(i)}}\norm{\hmatF_{(i)}}\\
&\le 4\sqrt{2}\sqrt{\wzit}\zzt \gmop
\end{align*}

Finally, for \tm{X}, we have,
\begin{align*}
&\sum_{i=1}^N\innerp{\left(\matI-\Pj{\matU_{\tau}}-\Pj{\matV_{(i),\tau}}\right)\matR_{(i)}}{\dU_{\tau}\matUt^T\hmatF_{(i)}}\\
&=\sum_{i=1}^N\innerp{\left(\matI-\Pj{\matU_{\tau}}-\Pj{\matV_{(i),\tau}}\right)\matR_{(i)}}{\dU_{\tau}\hmatUt^T\hmatF_{(i)}}+\innerp{\left(\matI-\Pj{\matU_{\tau}}-\Pj{\matV_{(i),\tau}}\right)\matR_{(i)}}{\dU_{\tau}\dUt^T\hmatF_{(i)}}\\
&=\sum_{i=1}^N\innerp{\left(\matI-\Pj{\matU_{\tau}}-\Pj{\matV_{(i),\tau}}\right)\matR_{(i)}}{\dU_{\tau}\dUt^T\hmatF_{(i)}}-\innerp{\Pj{\matVit}\matR_{(i)}}{\dUt\hmatUt\hmatF_{(i)}}\\
&=\sum_{i=1}^N \innerp{\matR_{(i)}}{\dU_{\tau}\dUt^T\hmatF_{(i)}}+\bm{\epsilon}_{8,(i),\tau}
\end{align*}
where $\bm{\epsilon}_{8,(i),\tau}$ is defined as,
\begin{align*}
\bm{\epsilon}_{8,(i),\tau}=\innerp{\left(-\Pj{\matU_{\tau}}-\Pj{\matV_{(i),\tau}}\right)\matR_{(i)}}{\dU_{\tau}\dUt^T\hmatF_{(i)}}-\innerp{\Pj{\matVit}\matR_{(i)}}{\dUt\hmatUt\hmatF_{(i)}}
\end{align*}
Its norm is upper bounded by
\begin{align*}
&\abs{\bm{\epsilon}_{8,(i),\tau}}\\
&\le \abs{\tr{\matR_{(i)}^T\left(-\Pj{\matU_{\tau}}-\Pj{\matV_{(i),\tau}}\right)\dU_{\tau}\dUt^T\hmatF_{(i)}}}+\abs{\tr{\matR_{(i)}^T\Pj{\matVit}\dUt\hmatUt\hmatF_{(i)} }}\\
&\le \abs{\tr{\matR_{(i)}^T\left(-\Pj{\matU_{\tau}}-\Pj{\matV_{(i),\tau}}\right)\dU_{\tau}\dUt^T\hmatF_{(i)}}}\\
&+\abs{\tr{\matR_{(i)}^T\left(\Pj{\matVit}-\hmatPi_{(i)}\right)\Pj{\matVit}\dUt\hmatUt\hmatF_{(i)} }}\\
&\le 4\sqrt{2}\sqrt{\wzit}\zzt\gmop+\norm{\Pj{\matVit}-\hmatPi_{(i)}}_F\norm{\dUt}_F\norm{\matR_{(i)}}\norm{\hmatF_{(i)}}
\end{align*}

Combining them, we have,
\begin{align}
\label{eqn:grddu}
&\sum_{i=1}^N\innerp{\left(\matI-\Pj{\matU_{\tau}}-\Pj{\matV_{(i),\tau}}\right)\matS_{(i)}\matU_{\tau}}{\dU_{\tau}}\notag\\
&=\sum_{i=1}^N\innerp{\left(\matI-\Pj{\matU_{\tau}}-\Pj{\matV_{(i),\tau}}\right)\hmatF_{(i)}}{\left(\Pj{\matUt}-\hmatPi_g\right)\hmatF_{(i)}}+\innerp{\matR_{(i)}}{\dUt\dUt^T\matR_{(i)}}\notag\\
&+\innerp{\matR_{(i)}}{\dUt\dUt^T\hmatF_{(i)}}\notag\\
&+\bm{\epsilon}_{5,(i),\tau}+\bm{\epsilon}_{6,(i),\tau}+\bm{\epsilon}_{7,(i),\tau}+\bm{\epsilon}_{8,(i),\tau}
\end{align}

Comparing \eqref{eqn:optimalitygapbound}, \eqref{eqn:grddv}, and \eqref{eqn:grddu}, we know that,
\begin{align*}
&-\innerp{\sum_{i=1}^N\left(\matI-\Pj{\matU_{\tau}}-\Pj{\matV_{(i),\tau}}\right)\matS_{(i)}\matU_{\tau}}{\dU_{\tau}}-\sum_{i=1}^N\innerp{\left(\matI-\Pj{\matU_{\tau}}-\Pj{\matV_{(i),\tau}}\right)\matS_{(i)}\matV_{(i),\tau}}{\dV_{(i),\tau}}\\
&=2\phi_{\tau}-\sum_{i=1}^N\tr{\matR^T_{(i)}\left(\dUt\dUt^T+\dVit\dVit^T\right)\matF_{(i)}}-\sum_{i=1}^N\sum_{\alpha=1}^8\bm{\epsilon}_{\alpha,(i),\tau}
\end{align*}
From the estimated upper bounds of $\abs{\bm{\epsilon}_{1,(i),\tau}}$ to $\abs{\bm{\epsilon}_{8,(i),\tau}}$, we know that,
\begin{align*}
&\abs{\sum_{i=1}^N\tr{\matR^T_{(i)}\left(\dUt\dUt^T+\dVit\dVit^T\right)\matF_{(i)}}}+\abs{\sum_{i=1}^N\sum_{\alpha=1}^8\bm{\epsilon}_{\alpha,(i),\tau}}_F\\
&\le 2\sum_{i=1}^N \left(\norm{\dUt}_F^2+\norm{\dVit}_F^2\right)\norm{\matR_{(i)}}\norm{\hmatF_{(i)}}+14\sqrt{2}\gmop\sqrt{\wzit}\left(\zit+\zzt\right)\\
\end{align*}

From Lemma \ref{lm:hatvariables}, we know that $\norm{\matR_{(i)}}\norm{\hmatF_{(i)}}\le \frac{\hmu\htheta}{64}$, we can thus upper bound the first term as,
\begin{align*}
&2\sum_{i=1}^N \left(\norm{\dUt}_F^2+\norm{\dVit}_F^2\right)\norm{\matR_{(i)}}\norm{\hmatF_{(i)}}\\
&\le \sum_{i=1}^N \left(\norm{\dUt}_F^2+\norm{\dVit}_F^2\right)\frac{\hmu\htheta}{32}\\
&\le \phi_{\tau}/2
\end{align*}
where the last inequality comes from Lemma \ref{lm:gaptodelta}.

Also, the second term can be bounded as,
\begin{align*}
&14\sqrt{2}\gmop\sum_{i=1}^N\sqrt{\wzit}\left(\zit+\zzt\right)\\
&\le 14\sqrt{2}\gmop\sum_{i=1}^N\sqrt{\wzit}\sqrt{\zit+\zzt}\sqrt{\sum_{j=1}^N\zeta_{(j),\tau}+\zeta_{(0),\tau}}\\
&\le 14\sqrt{2}\gmop\sqrt{\sum_{i=1}^N\wzit}\sqrt{\sum_{i=1}^N\zit+\zzt}\sqrt{\sum_{j=1}^N\zeta_{(j),\tau}+\zeta_{(0),\tau}}\\
&\le 14\sqrt{2}\gmop\left(\sum_{i=1}^N\zit+\zzt\right)^{1.5}\\
&\le 14\sqrt{2}\gmop\left(2\sum_{i=1}^N\norm{\dUt}_F^2+\norm{\dVit}_F^2\right)^{1.5}\\
&\le 56\gmop \left(\frac{\phi_{\tau}}{\frac{\hmu\htheta}{16}}\right)^{3/2}\le \phi_{\tau}/2
\end{align*}
where the second inequality comes from Cauchy-Schwrtz inequality, the third inequality comes from Lemma \ref{lm:sumspacetoindividualspace}, the fourth inequality comes from Lemma \ref{lm:udifffnormupper}, the fifth inequality comes from Lemma \ref{lm:gaptodelta}, and the last inequality comes from the fact that $\phi_{\tau}\le\frac{\hmu^3\htheta^3}{51380224\gmop^2}$.

This completes the proof.

\end{proof}
Combining Lemma \ref{lm:gaptodelta} with Lemma \ref{lm:gradientalign}, we can prove the following PL-inequality.
\begin{lemma}
\label{lm:plinequality}
(Lemma \ref{lm:mainpaperplinequality} in the main paper)
Under the same conditions as Theorem \ref{thm:linearconvergence}, we have
\begin{align*}
    N\norm{\square \matUt}_F^2+\sum_{i=1}^N\norm{\square \matVit}_F^2\ge \frac{\htheta\hmu}{16}\phi_{\tau}
\end{align*}
\end{lemma}
\begin{proof}
From Cauchy-Schwartz inequality, we know that,
\begin{align*}
&-N\innerp{\square\matUt}{\dUt}-\sum_{i=1}^N\innerp{\square \matVit}{\dVit}\\
&\le N\norm{\square\matUt}_F\norm{\dUt}+\sum_{i=1}^N\norm{\square \matVit}_F\norm{\dVit}_F\\
&\le \sqrt{N\norm{\square\matUt}_F^2+\sum_{i=1}^N\norm{\square \matVit}_F^2}\sqrt{N\norm{\dUt}_F^2+\sum_{i=1}^N\norm{\dVit}_F^2}\\
&\le \sqrt{N\norm{\square\matUt}_F^2+\sum_{i=1}^N\norm{\square \matVit}_F^2}\sqrt{\frac{\phi_{\tau}}{\frac{\hmu\htheta}{16}}}
\end{align*}
where the last inequality comes from Lemma \ref{lm:gaptodelta}. 

From Lemma \ref{lm:gradientalign}, we know,
\begin{align*}
&-N\innerp{\square\matUt}{\dUt}-\sum_{i=1}^N\innerp{\square \matVit}{\dVit}\\
&\ge \phi_{\tau}
\end{align*}

Combining them, we have,
\begin{align*}
N\norm{\square\matUt}_F^2+\sum_{i=1}^N\norm{\square \matVit}_F^2\ge \frac{\htheta\hmu}{16}\phi_{\tau}
\end{align*}
\end{proof}

Finally, we come to the proof of Theorem \ref{thm:linearconvergence}:
\begin{proof}
Combining Lemma \ref{lm:plinequality} with equation \eqref{eqn:decreaseinequality}, we know:
$$
\begin{aligned}
-f\left(\bm{U}_{\tau+1},\{\bm{V}_{(i),\tau+1}\}\right)\le -f\left(\bm{U}_{\tau},\{\bm{V}_{(i),\tau}\}\right)-\frac{\eta}{2}\frac{\hmu\htheta}{16}
\end{aligned}
$$
We add $f^{\star}$ on both sides. Since $\phi_{\tau}=f^{\star}--f\left(\bm{U}_{\tau},\{\bm{V}_{(i),\tau}\}\right)$, we have:
$$
\begin{aligned}
&\phi_{\tau+1}\\
&\le \phi_{\tau}-\frac{\eta}{2}\frac{\hmu\htheta}{16}\phi_{\tau}\\
&=\left(1-\frac{\eta\theta\mu}{64}\right)\phi_{\tau}
\end{aligned}
$$
Thus $\phi_{\tau}$ decreases linearly with $\tau$. From Lemma \ref{lm:gaptodelta} and Lemma \ref{lm:udifffnormupper}, we can show $\norm{\Pj{\matUt}-\hmatPi_g}_F$ and $\norm{\Pj{\matVit}-\hmatPi_{(i)}}_F$ decrease linearly to zero as well.

This completes the proof of Theorem \ref{thm:linearconvergence}.
\end{proof}

\section{Some examples of generalized retraction}
\label{lm:exampleofgr}
In this section, we discuss two popular normalization schemes: polar projection and QR decomposition. We prove that both fit the Definition \ref{def:generalizedretraction} of a generalized retraction. The analysis in this section is inspired by \citet{qrisretraction}. However, \citet{qrisretraction} only considers conventional retraction operations, while we consider generalized retractions.

\subsection{Polar projection}
\label{ap:polarisgr}
Polar projection is defined as:
$$
\polarof{\bm{U}}{\bm{\xi}}=\left(\bm{U}+\bm{\xi}\right)\left(\bm{I}+\bm{U}^T\bm{\xi}+\bm{\xi}^T\bm{U}+\bm{\xi}^T\bm{\xi}\right)^{-\frac{1}{2}}
$$
Then obviously, 
$$
col(\grof{\bm{U}}{\bm{\xi}})=col\left(\bm{U}+\bm{\xi}\right)
$$
To verify the second property, we can calculate the difference between $\grof{\bm{U}}{\bm{\xi}}$ and $\bm{U}+\mathcal{P}_{\mathcal{T}_{\bm{U}}}(\bm{\xi})$. 

Notice that
$$
\begin{aligned}
&\left(\bm{I}+\bm{U}^T\bm{\xi}+\bm{\xi}^T\bm{U}+\bm{\xi}^T\bm{\xi}\right)^{-\frac{1}{2}}\\
&=\bm{I}-\frac{1}{2}\bm{U}^T\bm{\xi}-\frac{1}{2}\bm{\xi}^T\bm{U}-\frac{1}{2}\bm{\xi}^T\bm{\xi}+\sum_{n=2}^{\infty}\left(\bm{U}^T\bm{\xi}+\bm{\xi}^T\bm{U}+\bm{\xi}^T\bm{\xi}\right)^n\frac{(2n-1)!!(-1)^n}{2^nn!}
\end{aligned}
$$
We have
\begin{equation}
\begin{aligned}
&\grof{\bm{U}}{\bm{\xi}}-(\bm{U}+\mathcal{P}_{\mathcal{T}_{\bm{U}}}(\bm{\xi}))\\
&= \left(\bm{U}+\bm{\xi}\right)\left(\bm{I}-\frac{1}{2}\bm{U}^T\bm{\xi}-\frac{1}{2}\bm{\xi}^T\bm{U}-\frac{1}{2}\bm{\xi}^T\bm{\xi}+\sum_{n=2}^{\infty}\left(\bm{U}^T\bm{\xi}+\bm{\xi}^T\bm{U}+\bm{\xi}^T\bm{\xi}\right)^n\frac{(2n-1)!!(-1)^n}{2^nn!}\right)\\
&-\left(\bm{U}-\bm{\xi}+\frac{1}{2}\bm{U}^T\bm{\xi}
+\frac{1}{2}\bm{\xi}^T\bm{U}\right)\\
&=\left(-\frac{1}{2}\bm{\xi}^T\bm{U}^T\bm{\xi}-\frac{1}{2}\bm{\xi}^T\bm{\xi}^T\bm{U}-\frac{1}{2}\bm{\xi}^T\bm{\xi}^T\bm{\xi}+\left(\bm{U}+\bm{\xi}\right)\sum_{n=2}^{\infty}\left(\bm{U}^T\bm{\xi}+\bm{\xi}^T\bm{U}+\bm{\xi}^T\bm{\xi}\right)^n\frac{(2n-1)!!(-1)^n}{2^nn!}\right)
\end{aligned}
\end{equation}

By the property of Frobinius norm:
$$
\begin{aligned}
&\norm{\bm{U}^T\bm{\xi}+\bm{\xi}^T\bm{U}+\bm{\xi}^T\bm{\xi}}_F\\
&\le 2\norm{\bm{\xi}}_F\norm{\bm{U}^T}_{op}+\norm{\bm{\xi}}_F^2\\
&= 2\norm{\bm{\xi}}_F+\norm{\bm{\xi}}_F^2\\
&\le 3\norm{\bm{\xi}}_F
\end{aligned}
$$

Therefore, 
\begin{equation*}
\begin{aligned}
\label{eqn:polarprojectionproperty}
&\norm{\grof{\bm{U}}{\bm{\xi}}-(\bm{U}+\mathcal{P}_{\mathcal{T}_{\bm{U}}}(\bm{\xi}))}_F\\
&\le \frac{1}{2}\norm{\bm{\xi}^T\bm{U}^T\bm{\xi}}_F+\frac{1}{2}\norm{\bm{\xi}^T\bm{\xi}^T\bm{U}}_F+\frac{1}{2}\norm{\bm{\xi}^T\bm{\xi}^T\bm{\xi}}_F\\
&+\norm{\bm{U}+\bm{\xi}}_F\sum_{n=2}^{\infty}\norm{\bm{U}^T\bm{\xi}+\bm{\xi}^T\bm{U}+\bm{\xi}^T\bm{\xi}}_F^n\frac{(2n-1)!!}{2^nn!}\\
&\le \norm{\bm{\xi}}_F^2+\frac{1}{2}\norm{\bm{\xi}}_F^3+\left(1+\norm{\bm{\xi}}_F\right)\sum_{n=2}^{\infty}\left(3\norm{\bm{\xi}}_F\right)^n\frac{(2n-1)!!}{2^nn!}\\
&= \norm{\bm{\xi}}_F^2+\frac{1}{2}\norm{\bm{\xi}}_F^3+\left(1+\norm{\bm{\xi}}_F\right)\frac{3\left(3\norm{\bm{\xi}}_F\right)^2+\left(3\norm{\bm{\xi}}_F\right)^3}{2}\\
&\le M_{polar} \norm{\bm{\xi}}_F^2
\end{aligned}
\end{equation*}
where $M_{polar}=\frac{253}{8}$. We applied the following summation in the derivation:
$$
\begin{aligned}
&\sum_{n=2}^{\infty}x^n\frac{(2n-1)!!}{2^nn!}\\
&=(1-x)^{-1/2}-(1+\frac{x}{2})\\
&=\frac{3x^2+x^3}{\sqrt{1-x}+(1-x)(1+\frac{x}{2})}\\
\end{aligned}
$$
and the fact that $x\le \frac{1}{2}$ in the third inequality.

Since 
$$
\begin{aligned}
&\norm{\bm{\xi}}_F^2\\
&=\norm{\mathcal{P}_{\mathcal{T}_{\bm{U}}}(\bm{\xi})+\mathcal{P}_{\mathcal{N}_{\bm{U}}}(\bm{\xi})}_F^2\\
&\le 2\norm{\mathcal{P}_{\mathcal{T}_{\bm{U}}}(\bm{\xi})}_F^2+2\norm{\mathcal{P}_{\mathcal{N}_{\bm{U}}}(\bm{\xi})}_F^2\\
&\le 2\norm{\mathcal{P}_{\mathcal{T}_{\bm{U}}}(\bm{\xi})}_F^2+\norm{\mathcal{P}_{\mathcal{N}_{\bm{U}}}(\bm{\xi})}_F\\
\end{aligned}
$$
We prove that polar projection is a generalized retraction with $M_1=\frac{253}{4}$ and $M_2=\frac{253}{8}$.

Polar projection can be implemented via singular value decomposition of $\bm{U}+\bm{\xi}$, whose computational complexity is $O(dr^2+r^3)$ \citep{mmalgorithm}.
\subsection{QR decomposition}
\label{ap:qrisgr}
QR decomposition is an extension of Gram-Schmidt orthonormalization. For a matrix $\bm{U}+\bm{\xi}\in \mathbb{R}^{d\times r}$, the method finds a orthogonal matrix $\bm{Q}\in\mathbb{R}^{d\times r} $ and an upper triangular matrix $\bm{R}\in\mathbb{R}^{r\times r}$, such that $\bm{Q}\bm{R}=\bm{U}+\bm{\xi}$. Then $\qrof{\bm{U}}{\bm{\xi}}=\bm{Q}$. 

In this section, we will prove that QR decomposition is a generalized retraction for $\norm{\bm{\xi}}\le \frac{1}{4}$. Our proof in this section extends that in \citet{qrisretraction}.

Notice that $col(\bm{U}+\bm{\xi})=col(\bm{Q})$, thus the first property of generalized retraction in Definition \ref{def:generalizedretraction} is satisfied. We will prove the second in the case $M_3=\frac{1}{4}$

Similar to \citet{qrisretraction}, we define $\bm{U}(t)=\bm{U}+t\bm{\xi}$, for $t\in[0,1]$, and use $\bm{Q}(t)\bm{R}(t)$ to denote the QR decomposition of $\bm{U}(t)$. Then:
$$
\begin{aligned}
&\norm{\qrof{\bm{U}}{\bm{\xi}}-\left(\bm{U}+\bm{\xi}\right)}_F\\
&=\norm{\bm{Q}(1)-\bm{Q}(1)\bm{R}(1)}_F=\norm{\bm{Q}(1)\left(\bm{I}-\bm{R}(1)\right)}_F\\
&\le \norm{\bm{R}(1)-\bm{R}(0)}_F\\
&=\norm{\int_{0}^1\bm{R}^{'}(t)dt}_F\\
&\le \int_{0}^1\norm{\bm{R}^{'}(t)}_Fdt\\
\end{aligned}
$$
Since $\bm{Q}(t)\bm{R}(t)$ is the QR decomposition of $\bm{U}(t)$, we have:
\begin{equation}
\label{eqn:rtr}
\bm{R}^T(t)\bm{R}(t)=\bm{U}^T(t)\bm{U}(t)=\bm{U}^T\bm{U}+t\bm{\xi}^T\bm{U}+t\bm{U}^T\bm{\xi}+t^2\bm{\xi}^T\bm{\xi}    
\end{equation}

Taking the derivative with respect to $t$ on both sides, we have:
$$
\begin{aligned}
&\left(\bm{R}^{'}\right)^T(t)\bm{R}(t)+\bm{R}^T(t)\bm{R}^{'}(t)\\
&=\bm{\xi}^T\bm{U}+\bm{U}^T\bm{\xi}+2t\bm{\xi}^T\bm{\xi}\\
\end{aligned}
$$
We can left multiply both sides by $\left(\bm{R}^{-1}\right)^T(t)$, and right multiply both sides by $\bm{R}^{-1}(t)$, to obtain:
$$
\begin{aligned}
&\left(\bm{R}^{-1}\right)^T(t)\left(\bm{R}^{'}\right)^T(t)+\bm{R}^{'}(t)\bm{R}^{-1}(t)=\left(\bm{R}^{-1}\right)^T(t)\left(\bm{\xi}^T\bm{U}+\bm{U}^T\bm{\xi}+2t\bm{\xi}^T\bm{\xi}\right)\bm{R}^{-1}(t)\\
\end{aligned}
$$
Since on the left hand side, $\bm{R}^{'}(t)\bm{R}^{-1}(t)$ is an upper triangular matrix, its transpose $\left(\bm{R}^{-1}\right)^T(t)\left(\bm{R}^{'}\right)^T(t)$ is a lower triangular matrix, we have:
$$
\bm{R}^{'}(t)\bm{R}^{-1}(t)=\text{up}\left[\left(\bm{R}^{-1}\right)^T(t)\left(\bm{\xi}^T\bm{U}+\bm{U}^T\bm{\xi}+2t\bm{\xi}^T\bm{\xi}\right)\bm{R}^{-1}(t)\right]
$$
where for $\bm{C}\in \mathbb{R}^{d\times d}$, $\text{up}\left[\cdot\right]$ is defined as:
$$
\text{up}\left[\bm{C}\right]_{ij}=\left\{\begin{aligned}
& C_{ij} \text{,  if  } j>i\\
&\frac{1}{2} C_{ii} \text{,  if  } j=i\\
& 0\text{,  if  } j<i\\
\end{aligned}\right.
$$

Therefore,
$$
\bm{R}^{'}(t)=\text{up}\left[\left(\bm{R}^{-1}\right)^T(t)\left(\bm{\xi}^T\bm{U}+\bm{U}^T\bm{\xi}+2t\bm{\xi}^T\bm{\xi}\right)\bm{R}^{-1}(t)\right]\bm{R}(t)
$$
and accordingly:
$$
\begin{aligned}
&\norm{\bm{R}^{'}(t)}_F=\norm{\text{up}\left[\left(\bm{R}^{-1}\right)^T(t)\left(\bm{\xi}^T\bm{U}+\bm{U}^T\bm{\xi}+2t\bm{\xi}^T\bm{\xi}\right)\bm{R}^{-1}(t)\right]\bm{R}(t)}_F\\
&\le \norm{\text{up}\left[\left(\bm{R}^{-1}\right)^T(t)\left(\bm{\xi}^T\bm{U}+\bm{U}^T\bm{\xi}+2t\bm{\xi}^T\bm{\xi}\right)\bm{R}^{-1}(t)\right]}_F\norm{\bm{R}(t)}_{op}\\
&\le \norm{\left(\bm{R}^{-1}\right)^T(t)\left(\bm{\xi}^T\bm{U}+\bm{U}^T\bm{\xi}+2t\bm{\xi}^T\bm{\xi}\right)\bm{R}^{-1}(t)}_F\norm{\bm{R}(t)}_{op}\\
\end{aligned}
$$
where we used Lemma \ref{lm:fnormofabbound} for the first inequality. 

From \eqref{eqn:rtr}, we know that:
$$
\begin{aligned}
&\norm{\bm{R}(t)}_{op}^2=\norm{\bm{R}(t)^T\bm{R}(t)}_{op}\\
&=\norm{\bm{I}+t\left(\bm{\xi}^T\bm{U}+\bm{U}^T\bm{\xi}\right)+t^2\bm{\xi}^T\bm{\xi}}_{op}\\
&\ge 1-t\norm{\bm{\xi}^T\bm{U}+\bm{U}^T\bm{\xi}}_{op}-t^2\norm{\bm{\xi}^T\bm{\xi}}_{op}\\
&\ge 1- 2\norm{\bm{\xi}}_F - \norm{\bm{\xi}^T\bm{\xi}}_F\\
&\ge \frac{7}{16} 
\end{aligned}
$$
where the first inequality comes from the triangle inequality, the second comes from the fact that $\norm{\cdot}_{F}\ge\norm{\cdot}_{op}$, and the third comes from the requirement $\norm{\bm{\xi}}_F\le \frac{1}{4}$.

Similarly, we can derive:
$$
\begin{aligned}
&\norm{\bm{R}(t)}_{op}^2=\norm{\bm{R}(t)^T\bm{R}(t)}_{op}\le 1 + 2\norm{\bm{\xi}}_F + \norm{\bm{\xi}^T\bm{\xi}}_F \le \frac{25}{16} 
\end{aligned}
$$
As a result,
$$
\begin{aligned}
&\norm{\bm{R}^{'}(t)}_F\le \norm{\left(\bm{R}^{-1}\right)^T(t)\left(\bm{\xi}^T\bm{U}+\bm{U}^T\bm{\xi}+2t\bm{\xi}^T\bm{\xi}\right)\bm{R}^{-1}(t)}_F\norm{\bm{R}(t)}_{op}\\
&\le \left(\norm{\left(\bm{R}^{-1}\right)^T(t)\left(\bm{\xi}^T\bm{U}+\bm{U}^T\bm{\xi}\right)\bm{R}^{-1}(t)}_F+\norm{\left(\bm{R}^{-1}\right)^T(t)\left(2t\bm{\xi}^T\bm{\xi}\right)\bm{R}^{-1}(t)}_F\right)\frac{5}{4}\\
&\le \frac{5}{4}\left(\norm{\bm{\xi}^T\bm{U}+\bm{U}^T\bm{\xi}}_F\norm{\left(\bm{R}^{-1}\right)^T(t)\bm{R}^{-1}(t)}_{op}+2t\norm{\bm{\xi}^T\bm{\xi}}_F\norm{\left(\bm{R}^{-1}\right)^T(t)\bm{R}^{-1}(t)}_{op}\right)\\
&\le \frac{20}{7}\left(\norm{\bm{\xi}^T\bm{U}+\bm{U}^T\bm{\xi}}_F+2t\norm{\bm{\xi}^T\bm{\xi}}_F\right)\\
\end{aligned}
$$

Hence,
$$
\begin{aligned}
&\norm{\qrof{\bm{U}}{\bm{\xi}}-\left(\bm{U}+\bm{\xi}\right)}_F\\
&\le\frac{20}{7}\left(\norm{\bm{\xi}^T\bm{U}+\bm{U}^T\bm{\xi}}_F+\norm{\bm{\xi}^T\bm{\xi}}_F\right)
\end{aligned}
$$

Since $\bm{U}$ is an orthogonal matrix, $\norm{\Pnormal{\bm{U}}{\bm{\xi}}}_F=\frac{1}{2}\norm{\bm{U}\left(\bm{\xi}^T\bm{U}+\bm{U}^T\bm{\xi}\right)}_F=\frac{1}{2}\norm{\bm{\xi}^T\bm{U}+\bm{U}^T\bm{\xi}}_F$. By Cauchy-Schwartz inequality, $\norm{\bm{\xi}^T\bm{\xi}}_F=\norm{\left(\Pnormal{\bm{U}}{\bm{\xi}}+\Ptangent{\bm{U}}{\bm{\xi}}\right)^T\left(\Pnormal{\bm{U}}{\bm{\xi}}+\Ptangent{\bm{U}}{\bm{\xi}}\right)}_F\le 2\norm{\Pnormal{\bm{U}}{\bm{\xi}}}_F^2+2\norm{\Ptangent{\bm{U}}{\bm{\xi}}}_F^2$. 

Thus we have:
$$
\begin{aligned}
&\norm{\qrof{\bm{U}}{\bm{\xi}}-\left(\bm{U}+\bm{\xi}\right)}_F\\
&\le\frac{20}{7}\left(2\norm{\Pnormal{\bm{U}}{\bm{\xi}}}_F+2\norm{\Pnormal{\bm{U}}{\bm{\xi}}}^2_F+2\norm{\Ptangent{\bm{U}}{\bm{\xi}}}_F^2\right)\\
&\le\frac{80}{7}\norm{\Pnormal{\bm{U}}{\bm{\xi}}}_F+\frac{40}{7}\norm{\Ptangent{\bm{U}}{\bm{\xi}}}_F^2
\end{aligned}
$$
Hence the second property of definition holds with $M_1=\frac{40}{7}$ and $M_2=\frac{80}{7}$. 

QR decomposition can be implemented by Gram-Schmidt or Householder algorithm with computation complexity of $O(dr^2)$
\section{Auxiliary lemmas}
\label{sec:auxlemmas}
In this section, we show some auxiliary lemmas needed for the proof in earlier Sections. Most lemmas are derived from basic facts in linear algebra.

We begin with some general inequalities related to matrix trace norms.
\begin{lemma}
\label{lm:traceofabisnonnegative}
For two matrices $\matA,\matB\in\mathbb{R}^{d\times d}$, if both $\matA,\matB$ are symmetric positive definite, then:
$$
\tr{\matA\matB}\ge 0
$$
\end{lemma}
A simple corollary is that if $\matA_1,\matA_2,\matB\in \mathbb{R}^{d\times d}$ are symmetric and $\matB$ is positive semi-definite, and $\matA_1 \succeq \matA_2$, then 
$$
\tr{\matA_1\matB}\ge \tr{\matA_2\matB}
$$
\begin{proof}
Since both $\matA$ and $\matB$ are positive symmetric, there exists $\matX,\matY\in \mathbb{R}^{d\times d}$, such that $A=X^TX$ and $\matB=\matY^T\matY$, therefore:
$$
\begin{aligned}
&\tr{\matA\matB}\\
&=\tr{\matX^T\matX\matY^T\matY}\\
&=\tr{(\matY\matX^T)^T\matY\matX^T}\\
&\ge 0
\end{aligned}
$$
\end{proof}
The following lemma presents an upper bound of the Frobenius norm of the product of two matrices.
\begin{lemma}
\label{lm:fnormofabbound}
For two matrices $\matA\in\mathbb{R}^{m\times n}$, and $\matB\in\mathbb{R}^{n\times k}$, we have:
$$
\norm{\matA\matB}_F\le \norm{\matA}_{op}\norm{\matB}_F
$$
and:
$$
\norm{\matA\matB}_F\le \norm{\matA}_F\norm{\matB}_{op}
$$
\end{lemma}
The proof of the lemma can be found in \citet{ruoyufactorization}. 

The following lemma introduces a simple upper bound on the Frobenius norm of $\bm{I}_r-\bm{U}^T\bm{P}\bm{U}$. 
\begin{lemma}
\label{lm:difffnormupperbound}
For any rank-r orthonormal matrix $\bm{U}\in \mathbb{R}^{d\times r}$, and rank-r projection matrix $\bm{P}\in \mathbb{R}^{d\times d}$, we have:
\begin{equation}
\label{eqn:imutpufnormupperbound}
\norm{\bm{I}_r-\bm{U}^T\bm{P}\bm{U}}_F\le r - \tr{\bm{U}^T\bm{P}\bm{U}}
\end{equation}
\end{lemma}
\begin{proof}
It is easy to see that $\bm{I}_r-\bm{U}^T\bm{P}\bm{U}$ is positive semidefinite. Also, for a positive semidefinite matrix, its Frobenius norm is upper bounded by its trace. Inequality \eqref{eqn:imutpufnormupperbound} follows accordingly.
\end{proof}
We can proceed to the following lemma that upper bounds the trace of the  $k$-th power of $\bm{I}_r-\bm{U}^T\bm{P}\bm{U}$.
\begin{lemma}
\label{lm:powerdifftraceupperbound}
For any rank-r orthonormal matrix $\bm{U}\in \mathbb{R}^{d\times r}$, and rank-r projection matrix $\bm{P}\in \mathbb{R}^{d\times r}$, $k=1,2,...$, $\left(\bm{I}_r-\bm{U}^T\bm{P}\bm{U}\right)^k$ is positive semi-definite and:
\begin{equation}
0\le Tr\left(\left(\bm{I}-\bm{U}^T\bm{P}\bm{U}\right)^k\right)\le \left(r - \tr{\bm{U}^T\bm{P}\bm{U}}\right)^k
\end{equation}
\end{lemma}
\begin{proof}
Since $\bm{I}_r-\bm{U}^T\bm{P}\bm{U}$ is symmetric positive semidefinite, $\left(\bm{I}_r-\bm{U}^T\bm{P}\bm{U}\right)^k$ is also symmetric positive semidefinite. Assume eigenvalues of $\bm{I}_r-\bm{U}^T\bm{P}\bm{U}$ are $\lambda_1, \lambda_2,\cdots, \lambda_d$, with $\lambda_1\ge \lambda_2\ge\cdots\ge \lambda_d$, we know that $Tr\left(\left(\bm{I}_r-\bm{U}^T\bm{P}\bm{U}\right)^k\right)=\sum_{i=1}^d\lambda_i^k\le \lambda_1^{k-1}\sum_{i=1}^d\lambda_i$.

By Lemma \ref{lm:difffnormupperbound}, we know that 
$$
\lambda_1^{k-1}\le \norm{\bm{I}_r-\bm{U}^T\bm{P}\bm{U}}_F^{k-1}\le \left(r-\tr{\bm{U}^T\bm{P}\bm{U}}\right)^{k-1}
$$
This completes our proof.
\end{proof}
Based on the above results, we can discuss some properties of the projection of a matrix onto a subspace. Suppose we know the column space of $\bm{U}\in\mathbb{R}^{d\times r}$ is close to that of $\bm{P}\in \mathbb{R}^{d\times d}$, can we find a matrix $\bm{U}^{*}$ close to $\bm{U}$ with column vectors in $col(\bm{P})$? The following two lemmas give affirmative answers. 
\begin{lemma}
\label{lm:udifffnormupper}
For any rank-r orthonormal matrix $\bm{U}\in \mathbb{R}^{d\times r}$, and rank-r projection matrix $\bm{P}\in \mathbb{R}^{d\times d}$, we define:
$$
\bm{U}^{\star} = \bm{P}\bm{U}\left(\bm{U}^T\bm{P}\bm{U}\right)^{-1/2}
$$
If $r-\tr{\bm{U}^T\bm{P}\bm{U}}\le 1$, we have:
\begin{equation}
\label{eqn:generaldunormlowerbound}
\norm{\bm{U}-\bm{U}^\star}_F^2\ge r-\tr{\bm{U}^T\bm{P}\bm{U} }
\end{equation}
and,
\begin{equation}
\label{eqn:generaldunormupperbound}
\norm{\bm{U}-\bm{U}^\star}_F^2\le 2\left(r-\tr{\bm{U}^T\bm{P}\bm{U} }\right)
\end{equation}
\end{lemma}
\begin{proof}
To prove the lower bound \eqref{eqn:generaldunormlowerbound} and upper bound \eqref{eqn:generaldunormupperbound}, we can write $\norm{\bm{U}-\bm{U}^\star}_F^2$ as,
$$
\begin{aligned}
&\norm{\bm{U}-\bm{U}^\star}_F^2 = \left\langle\bm{U},\bm{U}\right\rangle+\left\langle\bm{U}^\star,\bm{U}^\star\right\rangle-2\left\langle\bm{U},\bm{U}^\star\right\rangle\\
&=2r - 2\left\langle\bm{U},\bm{U}^\star\right\rangle
\end{aligned}
$$

We first find an upper bound for $\left\langle\bm{U},\bm{U}^\star\right\rangle$.

Notice that:
$$
\begin{aligned}
&\left\langle\bm{U},\bm{U}^\star\right\rangle\\
&=\tr{\bm{U}^T\bm{U}^\star}\\
&=\tr{\bm{U}^T\bm{P}\bm{U}\left(\bm{U}^T\bm{P}\bm{U}\right)^{-1/2}}\\
&=\tr{\left(\bm{U}^T\bm{P}\bm{U}\right)^{1/2}}\\
&=\tr{\left(\bm{I}_r-\left(\bm{I}_r-\bm{U}^T\bm{P}\bm{U}\right)\right)^{1/2}}\\
&=\tr{\bm{I}_r-\frac{1}{2}\left(\bm{I}_r-\bm{U}^T\bm{P}\bm{U}\right)-\sum_{n=2}^{\infty}\frac{(2n-3)!!}{2^nn!}\left(\bm{I}_r-\bm{U}^T\bm{P}\bm{U}\right)^n}\\
&= r-\frac{1}{2}\tr{\bm{I}_r-\bm{U}^T\bm{P}\bm{U}}-\sum_{n=2}^{\infty}\frac{(2n-3)!!}{2^nn!}\tr{\left(\bm{I}_r-\bm{U}^T\bm{P}\bm{U}\right)^n}\\
&\le r-\frac{1}{2}\tr{\bm{I}_r-\bm{U}^T\bm{P}\bm{U}}
\end{aligned}
$$
We used the series $(1-x)^{\frac{1}{2}}=1-\frac{1}{2}x-\sum_{n=2}^{\infty}\frac{(2n-3)!!}{2^nn!}x^n$, and the result $\tr{\left(\bm{I}_r-\bm{U}^T\bm{P}\bm{U}\right)^n}\ge 0$ from Lemma \ref{lm:powerdifftraceupperbound}.

As a result:
$$
\norm{\bm{U}-\bm{U}^\star}_F^2\ge \tr{\bm{I}_r- \bm{U}^T\bm{P}\bm{U}}=r-\tr{ \bm{U}^T\bm{P}\bm{U}}
$$

Similarly, from Lemma \ref{lm:powerdifftraceupperbound}, $\tr{\left(\bm{I}_r-\bm{U}^T\bm{P}\bm{U}\right)^n}\le \tr{\bm{I}_r-\bm{U}^T\bm{P}\bm{U}}^n$, thus:
$$
\begin{aligned}
&\left\langle\bm{U},\bm{U}^\star\right\rangle\\
&= r-\frac{1}{2}\tr{\bm{I}_r-\bm{U}^T\bm{P}\bm{U}}-\sum_{n=2}^{\infty}\frac{(2n-3)!!}{2^nn!}\tr{\left(\bm{I}_r-\bm{U}^T\bm{P}\bm{U}\right)^n}\\
&\ge r-\frac{1}{2}\tr{\bm{I}_r-\bm{U}^T\bm{P}\bm{U}}-\sum_{n=2}^{\infty}\frac{(2n-3)!!}{2^nn!}\tr{\bm{I}_r-\bm{U}^T\bm{P}\bm{U}}^n\\
&=r+\left(1-\tr{\bm{I}_r-\bm{U}^T\bm{P}\bm{U}}\right)^{\frac{1}{2}}-1\\
&\ge -\tr{\bm{I}_r-\bm{U}^T\bm{P}\bm{U}}
\end{aligned}
$$
where we used the relation $\sqrt{1-x}-1\ge -x, \forall x\in [0,1]$, in the last inequality.

Thus
$$
\norm{\bm{U}-\bm{U}^\star}_F^2\le 2\tr{\bm{I}_r- \bm{U}^T\bm{P}\bm{U}}=2\left(r-\tr{ \bm{U}^T\bm{P}\bm{U}}\right)
$$

This completes our proof.
\end{proof}

The following lemma shows that we can identify global PCs from local PCs.
\begin{lemma}
Suppose for $i=1,\cdots,N$, $\bm{P}_{\bm{U}}$, $\bm{P}_{\bm{V}_{(i)}}$ and $\bm{P}_{\bm{U}}^\star$, $\bm{P}_{\bm{V}_{(i)}}^\star$ are projection matrices satisfying $\bm{P}_{\bm{U}}\bm{P}_{\bm{V}_{(i)}}=0$ and $\bm{P}_{\bm{U}}^\star\bm{P}_{\bm{V}_{(i)}}^\star=0$ for each $i$. Among them, $\bm{P}_{\bm{U}}$ and $\bm{P}_{\bm{U}}^\star$ have rank $r_1$, $\bm{P}_{\bm{V}_{(i)}}$ and $\bm{P}_{\bm{V}_{(i)}}^\star$ have rank $r_{2,(i)}$. If there exists a positive constant $\theta>0$ such that
$$
\lambda_{max}(\frac{1}{N}\sum_{i=1}^N\bm{P}_{\bm{V}_{(i)}}^\star)\le 1-\theta
$$
we have the following bound:
\begin{equation}
\label{eqn:upperboundontrace}
\begin{aligned}
\sum_{i=1}^N&r_1+r_{2,(i)}-Tr\left(\left(\bm{P}_{\bm{U}}+\bm{P}_{\bm{V}_{(i)}}\right)\left(\bm{P}_{\bm{U}}^\star+\bm{P}_{\bm{V}_{(i)}}^\star\right)\right)\\
&\le N\left(r_1- Tr(\bm{P}_{\bm{U}}^\star\bm{P}_{\bm{U}})\right)+\sum_{i=1}^N r_{2,(i)}-Tr(\bm{P}_{\bm{V}_{(i)}}^\star\bm{P}_{\bm{V}_{(i)}})
\end{aligned}
\end{equation}
And also:
\begin{equation}
\label{eqn:lowerboundontrace}
\begin{aligned}
&\sum_{i=1}^Nr_1+r_{2,(i)}-Tr\left(\left(\bm{P}_{\bm{U}}+\bm{P}_{\bm{V}_{(i)}}\right)\left(\bm{P}_{\bm{U}}^\star+\bm{P}_{\bm{V}_{(i)}}^\star\right)\right)\\
&\ge \frac{\theta}{2}\left(N\left(r_1- Tr(\bm{P}_{\bm{U}}^\star\bm{P}_{\bm{U}})\right)+\sum_{i=1}^N r_{2,(i)}-Tr(\bm{P}_{\bm{V}_{(i)}}^\star\bm{P}_{\bm{V}_{(i)}})\right)
\end{aligned}
\end{equation}
\end{lemma}
Notice that we can replace $+$ by $\oplus$ on the left hand side of \eqref{eqn:upperboundontrace} and \eqref{eqn:lowerboundontrace}
\begin{proof}
We first calculate the upper bound.

Since $\bm{P}_{\bm{U}}\bm{P}_{\bm{V}_{(i)}}^\star\bm{P}_{\bm{U}}$ is positive semidefinite, we know that:
$$
\tr{\bm{P}_{\bm{U}}\bm{P}_{\bm{V}_{(i)}}^\star\bm{P}_{\bm{U}}}\ge 0
$$
Thus
$$
\tr{\bm{P}_{\bm{U}}\bm{P}_{\bm{V}_{(i)}}^\star\bm{P}_{\bm{U}}}=\tr{\bm{P}_{\bm{U}}\bm{P}_{\bm{V}_{(i)}}^\star}\ge 0
$$
Similarly, we have:
$$
\tr{\bm{P}_{\bm{U}}^\star\bm{P}_{\bm{V}_{(i)}}}\ge 0
$$
Combining them, we have:
\begin{equation*}
\begin{aligned}
&\tr{\left(\bm{P}_{\bm{U}}+\bm{P}_{\bm{V}_{(i)}}\right)\left(\bm{P}_{\bm{U}}^\star+\bm{P}_{\bm{V}_{(i)}}^\star\right)}\\
&=\tr{\bm{P}_{\bm{U}}\bm{P}_{\bm{U}}^\star}+\tr{\bm{P}_{\bm{U}}\bm{P}_{\bm{V}_{(i)}}^\star}+\tr{\bm{P}_{\bm{V}_{(i)}}\bm{P}_{\bm{U}}^\star}+\tr{\bm{P}_{\bm{V}_{(i)}}\bm{P}_{\bm{V}_{(i)}}^\star}\\
&\ge \tr{\bm{P}_{\bm{U}}\bm{P}_{\bm{U}}^\star}+\tr{\bm{P}_{\bm{V}_{(i)}}\bm{P}_{\bm{V}_{(i)}}^\star}\\
\end{aligned}
\end{equation*}
This proves inequality \eqref{eqn:upperboundontrace}.

Next, we calculate the lower bound.
\begin{equation*}
\begin{aligned}
&\sum_{i=1}^N\tr{\left(\bm{P}_{\bm{U}}+\bm{P}_{\bm{V}_{(i)}}\right)\left(\bm{P}_{\bm{U}}^\star+\bm{P}_{\bm{V}_{(i)}}^\star\right)}\\
&=\sum_{i=1}^N\tr{\bm{P}_{\bm{U}}\bm{P}_{\bm{U}}^\star}+\tr{\bm{P}_{\bm{U}}\bm{P}_{\bm{V}_{(i)}}^\star)}+\tr{\bm{P}_{\bm{V}_{(i)}}\bm{P}_{\bm{U}}^\star}+\tr{\bm{P}_{\bm{V}_{(i)}}\bm{P}_{\bm{V}_{(i)}}^\star}\\
&=\sum_{i=1}^N\tr{\bm{P}_{\bm{U}}\bm{P}_{\bm{U}}^\star}+\tr{\bm{P}_{\bm{U}}\left(\bm{I}-\bm{P}_{\bm{U}}^\star\right)\bm{P}_{\bm{V}_{(i)}}^\star\left(\bm{I}-\bm{P}_{\bm{U}}^\star\right)}+\tr{\bm{P}_{\bm{V}_{(i)}}\bm{P}_{\bm{U}}^\star}+\tr{\bm{P}_{\bm{V}_{(i)}}\bm{P}_{\bm{V}_{(i)}}^\star}\\
&=\sum_{i=1}^N\tr{\bm{P}_{\bm{U}}\bm{P}_{\bm{U}}^\star}+\tr{\left(\bm{I}-\bm{P}_{\bm{U}}^\star\right)\bm{P}_{\bm{U}}\left(\bm{I}-\bm{P}_{\bm{U}}^\star\right)\bm{P}_{\bm{V}_{(i)}}^\star}+\tr{\bm{P}_{\bm{V}_{(i)}}\bm{P}_{\bm{U}}^\star}+\tr{\bm{P}_{\bm{V}_{(i)}}\bm{P}_{\bm{V}_{(i)}}^\star}\\
\end{aligned}
\end{equation*}
Since $\left(\bm{I}-\bm{P}_{\bm{U}}^\star\right)\bm{P}_{\bm{U}}\left(\bm{I}-\bm{P}_{\bm{U}}^\star\right)$ and $\bm{P}_{\bm{V}_{(i)}}^\star$ are both symmetric positive semidefinite, we have:
\begin{equation*}
\begin{aligned}
&\tr{\left(\bm{I}-\bm{P}_{\bm{U}}^\star\right)\bm{P}_{\bm{U}}\left(\bm{I}-\bm{P}_{\bm{U}}^\star\right)\frac{1}{N}\sum_{i=1}^N\bm{P}_{\bm{V}_{(i)}}^\star}\\
&\le \tr{\left(\bm{I}-\bm{P}_{\bm{U}}^\star\right)\bm{P}_{\bm{U}}\left(\bm{I}-\bm{P}_{\bm{U}}^\star\right)}\lambda_{max}\left(\frac{1}{N}\sum_{i=1}^N\bm{P}_{\bm{V}_{(i)}}^\star\right)\\
&\le \tr{\bm{P}_{\bm{U}}-\bm{P}_{\bm{U}}\bm{P}_{\bm{U}}^\star}\left(1-\theta\right)\\
&=\left(r_1-\tr{\bm{P}_{\bm{U}}\bm{P}_{\bm{U}}^\star}\right)\left(1-\theta\right)\\
\end{aligned}
\end{equation*}
For notation simplicity, we define $z_0 = r_1-\tr{\bm{P}_{\bm{U}}\bm{P}_{\bm{U}}^\star}$ and $z_i = r_{2,(i)}-\tr{\bm{P}_{\bm{V}_{(i)}}\bm{P}_{\bm{V}_{(i)}}^\star}$.

From the orthogonality, we have:
\begin{equation*}
\begin{aligned}
&\tr{\bm{P}_{\bm{V}_{(i)}}\bm{P}_{\bm{U}}^\star}\\
&=\tr{\bm{P}_{\bm{V}_{(i)}}\left(\bm{I}-\bm{P}_{\bm{V}_{(i)}}^\star\right)\bm{P}_{\bm{U}}^\star\left(\bm{I}-\bm{P}_{\bm{V}_{(i)}}^\star\right)}\\
&=\tr{\left(\bm{I}-\bm{P}_{\bm{V}_{(i)}}^\star\right)\bm{P}_{\bm{V}_{(i)}}\left(\bm{I}-\bm{P}_{\bm{V}_{(i)}}^\star\right)\bm{P}_{\bm{U}}^\star}\\
&\le \tr{\left(\bm{I}-\bm{P}_{\bm{V}_{(i)}}^\star\right)\bm{P}_{\bm{V}_{(i)}}\left(\bm{I}-\bm{P}_{\bm{V}_{(i)}}^\star\right)}\lambda_{max}\left(\bm{P}_{\bm{U}}^\star\right)\\
&\le \tr{\left(\bm{I}-\bm{P}_{\bm{V}_{(i)}}^\star\right)\bm{P}_{\bm{V}_{(i)}}\left(\bm{I}-\bm{P}_{\bm{V}_{(i)}}^\star\right)}\\
&= \tr{\bm{P}_{\bm{V}_{(i)}}-\bm{P}_{\bm{V}_{(i)}}\bm{P}_{\bm{V}_{(i)}}^\star}\\
&=z_i
\end{aligned}
\end{equation*}

Also, from the orthogonality, we have:
\begin{equation*}
\begin{aligned}
&\tr{\bm{P}_{\bm{V}_{(i)}}\bm{P}_{\bm{U}}^\star}\\
&=\tr{\left(\bm{I}-\bm{P}_{\bm{U}}\right)\bm{P}_{\bm{V}_{(i)}}\left(\bm{I}-\bm{P}_{\bm{U}}\right)\bm{P}_{\bm{U}}^\star}\\
&=\tr{\bm{P}_{\bm{V}_{(i)}}\left(\bm{I}-\bm{P}_{\bm{U}}\right)\bm{P}_{\bm{U}}^\star\left(\bm{I}-\bm{P}_{\bm{U}}\right)}\\
&\le \tr{\left(\bm{I}-\bm{P}_{\bm{U}}\right)\bm{P}_{\bm{U}}^\star\left(\bm{I}-\bm{P}_{\bm{U}}\right)}\lambda_{max}\left(\bm{P}_{\bm{V}_{(i)}}\right)\\
&\le \tr{\left(\bm{I}-\bm{P}_{\bm{U}}\right)\bm{P}_{\bm{U}}^\star\left(\bm{I}-\bm{P}_{\bm{U}}\right)}\\
&= \tr{\bm{P}_{\bm{U}}^\star-\bm{P}_{\bm{U}}\bm{P}_{\bm{U}}^\star}\\
&=z_0
\end{aligned}
\end{equation*}
Combining the two:
$$
\tr{\bm{P}_{\bm{V}_{(i)}}\bm{P}_{\bm{U}}^\star}\le \min\{z_0,z_i\}
$$
As a result:
\begin{equation*}
\begin{aligned}
&\sum_{i=1}^Nr_1+r_{2,(i)}\\
&-\left[\tr{\bm{P}_{\bm{U}}\bm{P}_{\bm{U}}^\star}+\tr{\left(\bm{I}-\bm{P}_{\bm{U}}^\star\right)\bm{P}_{\bm{U}}\left(\bm{I}-\bm{P}_{\bm{U}}^\star\right)\bm{P}_{\bm{V}_{(i)}}^\star}+\tr{\bm{P}_{\bm{V}_{(i)}}\bm{P}_{\bm{U}}^\star}+\tr{\bm{P}_{\bm{V}_{(i)}}\bm{P}_{\bm{V}_{(i)}}^\star}\right]\\
&\ge \sum_{i=1}^N z_0-\left(1-\theta\right)z_0+z_i-\min\{z_0,z_i\}
\end{aligned}
\end{equation*}
Since for any number $\nu\in\left(0,1\right)$, we know:
$$
z_i-\min\{z_0,z_i\}\ge \nu\left(z_i-z_0\right)
$$
We can set $\nu=\frac{\theta}{2}$, then
\begin{equation*}
\begin{aligned}
&\sum_{i=1}^N z_0-\left(1-\theta\right)z_0+z_i-\min\{z_0,z_i\}\\
&\ge \sum_{i=1}^N \theta z_0+\frac{\theta}{2}\left(z_i-z_0\right)\\
&=\frac{\theta}{2}\sum_{i=1}^N z_0+z_i
\end{aligned}
\end{equation*}
This proves inequality \eqref{eqn:lowerboundontrace}.
\end{proof}

\newpage

\bibliography{ref}

\begin{thebibliography}{59}
\providecommand{\natexlab}[1]{#1}
\providecommand{\url}[1]{\texttt{#1}}
\expandafter\ifx\csname urlstyle\endcsname\relax
  \providecommand{\doi}[1]{doi: #1}\else
  \providecommand{\doi}{doi: \begingroup \urlstyle{rm}\Url}\fi

\bibitem[Absil et~al.(2008)Absil, Mahony, and Sepulchre]{optimizationonmatrixmanifold}
P.-A. Absil, R.~Mahony, and R.~Sepulchre.
\newblock \emph{Optimization algorithms on matrix manifolds}.
\newblock Princeton University Press, 2008.

\bibitem[Aguilera et~al.(1999)Aguilera, Oca{\~n}a, and Valderrama]{pcatimeseries2}
Ana~M Aguilera, Francisco~A Oca{\~n}a, and Mariano~J Valderrama.
\newblock Forecasting time series by functional pca. discussion of several weighted approaches.
\newblock \emph{Computational Statistics}, 14\penalty0 (3):\penalty0 443--467, 1999.

\bibitem[Alimisis et~al.(2021)Alimisis, Davies, Vandereycken, and Alistarh]{quantizedriemann}
Foivos Alimisis, Peter Davies, Bart Vandereycken, and Dan Alistarh.
\newblock Distributed principal component analysis with limited communication.
\newblock In A.~Beygelzimer, Y.~Dauphin, P.~Liang, and J.~Wortman Vaughan, editors, \emph{Advances in Neural Information Processing Systems}, 2021.
\newblock URL \url{https://openreview.net/forum?id=edCFRvlWqV}.

\bibitem[Asokan(2022)]{debatedataset}
Rohan Asokan.
\newblock Us presidential debate transcripts 1960-2020.
\newblock In \emph{Kaggle dataset}, 2022.
\newblock \doi{10.34740/KAGGLE/DSV/3690532}.
\newblock URL \url{https://www.kaggle.com/datasets/arenagrenade/us-presidential-debate-transcripts-19602020}.

\bibitem[Bhatia(1997)]{matrixbook}
Rajendra Bhatia.
\newblock \emph{Matrix Analysis}.
\newblock Springer, New York, NY, 1997.

\bibitem[Birnbaum et~al.(2013)Birnbaum, Johnstone, Nadler, and Paul]{sparsepcalower}
Aharon Birnbaum, Iain~M Johnstone, Boaz Nadler, and Debashis Paul.
\newblock Minimax bounds for sparse pca with noisy high-dimensional data.
\newblock \emph{Annals of statistics}, 41\penalty0 (3):\penalty0 1055, 2013.

\bibitem[Boumal(2022)]{intromanifolds}
Nicolas Boumal.
\newblock An introduction to optimization on smooth manifolds.
\newblock To appear with Cambridge University Press, Apr 2022.
\newblock URL \url{http://www.nicolasboumal.net/book}.

\bibitem[Bouwmans et~al.(2018)Bouwmans, Javed, Zhang, Lin, and Otazo]{robustpcaimage}
Thierry Bouwmans, Sajid Javed, Hongyang Zhang, Zhouchen Lin, and Ricardo Otazo.
\newblock On the applications of robust pca in image and video processing.
\newblock \emph{Proceedings of the IEEE}, 106\penalty0 (8):\penalty0 1427--1457, 2018.
\newblock \doi{10.1109/JPROC.2018.2853589}.

\bibitem[Breloy et~al.(2021)Breloy, Kumar, Sun, and Palomar]{mmalgorithm}
Arnaud Breloy, Sandeep Kumar, Ying Sun, and Daniel~P. Palomar.
\newblock Majorization-minimization on the stiefel manifold with application to robust sparse pca.
\newblock \emph{IEEE Transactions on Signal Processing}, 69:\penalty0 1507--1520, 2021.
\newblock \doi{10.1109/TSP.2021.3058442}.

\bibitem[Caldas et~al.(2019)Caldas, Duddu, Wu, Li, Konecny, McMahan, Smith, and Talwalkar]{leaf}
Sebastian Caldas, Sai Meher~Karthik Duddu, Peter Wu, Tian Li, Jakub Konecny, H.~Brendan McMahan, Virginia Smith, and Ameet Talwalkar.
\newblock Leaf: A benchmark for federated settings.
\newblock In \emph{NeurIPS}, 2019.

\bibitem[Candes et~al.(2011)Candes, Li, Ma, and Wright]{robustpca}
Emmanuel~J. Candes, Xiaodong Li, Yi~Ma, and John Wright.
\newblock Robust principal component analysis?
\newblock \emph{Journal of the ACM (JACM)}, 2011.

\bibitem[Chen et~al.(2021{\natexlab{a}})Chen, Garcia, Hong, and Shahrampour]{mingyiconsensus}
Shixiang Chen, Alfredo Garcia, Mingyi Hong, and Shahin Shahrampour.
\newblock On the local linear rate of consensus on the stiefel manifold.
\newblock In \emph{Arxiv}, 2021{\natexlab{a}}.
\newblock URL \url{https://arxiv.org/pdf/2101.09346.pdf}.

\bibitem[Chen et~al.(2021{\natexlab{b}})Chen, Garcia, Hong, and Shahrampour]{mingyistiefel}
Shixiang Chen, Alfredo Garcia, Mingyi Hong, and Shahin Shahrampour.
\newblock Decentralized riemannian gradient descent on the stiefel manifold.
\newblock In Marina Meila and Tong Zhang, editors, \emph{Proceedings of the 38th International Conference on Machine Learning}, volume 139 of \emph{Proceedings of Machine Learning Research}, pages 1594--1605. PMLR, 18--24 Jul 2021{\natexlab{b}}.
\newblock URL \url{https://proceedings.mlr.press/v139/chen21g.html}.

\bibitem[Chen et~al.(2020)Chen, Lee, Li, and Yang]{shiftinverse}
Xi~Chen, Jason~D. Lee, He~Li, and Yun Yang.
\newblock Distributed estimation for principal component analysis: a gap-free approach.
\newblock \emph{CoRR}, abs/2004.02336, 2020.
\newblock URL \url{https://arxiv.org/abs/2004.02336}.

\bibitem[Deledalle et~al.(2011)Deledalle, Salmon, and Dalalyan]{pcaimage1}
Charles-Alban Deledalle, Joseph Salmon, and Arnak Dalalyan.
\newblock Image denoising with patch-based pca: local versus global.
\newblock In \emph{The 22nd British Machine Vision Conference}, 2011.

\bibitem[Edelman et~al.(1998)Edelman, Arias, and Smith]{stiefelgeometry}
Alan Edelman, Tom\'as~A. Arias, and Steven~T. Smith.
\newblock The geometry of algorithms with orthogonality constraints.
\newblock \emph{SIAM Journal on Matrix Analysis and Applications}, 20\penalty0 (2):\penalty0 303--353, 1998.
\newblock \doi{10.1137/S0895479895290954}.

\bibitem[Fan et~al.(2019)Fan, Wang, Wang, and Zhu]{oneshotdpca}
Jianqing Fan, Dong Wang, Kaizheng Wang, and Ziwei Zhu.
\newblock Distributed estimation of principal eigenspaces.
\newblock \emph{Annals of statistics}, 47,6:\penalty0 3009--3031, 2019.
\newblock \doi{10.1214/18-AOS1713}.

\bibitem[Feldman et~al.(2013)Feldman, Schmidt, and Sohler]{dispca1}
Dan Feldman, Melanie Schmidt, and Christian Sohler.
\newblock Turning big data into tiny data: Constant-size coresets for k-means, pca and projective clustering.
\newblock In \emph{Proceedings of the Twenty-Fourth Annual ACM-SIAM Symposium on Discrete Algorithms}, SODA '13, pages 1434--1453, USA, 2013. Society for Industrial and Applied Mathematics.
\newblock ISBN 9781611972511.

\bibitem[F.R.S.(1901)]{pearsonpca}
Karl~Pearson F.R.S.
\newblock Liii. on lines and planes of closest fit to systems of points in space.
\newblock \emph{Philosophical Magazine Series 1}, 2:\penalty0 559--572, 1901.

\bibitem[Garber and Hazan(2015)]{shiftandinversecentral}
Dan Garber and Elad Hazan.
\newblock Fast and simple pca via convex optimization.
\newblock \emph{ArXiv}, abs/1509.05647, 2015.
\newblock URL \url{https://arxiv.org/abs/1509.05647}.

\bibitem[Garber et~al.(2017)Garber, Shamir, and Srebro]{shiftinverse2}
Dan Garber, Ohad Shamir, and Nathan Srebro.
\newblock Communication-efficient algorithms for distributed stochastic principal component analysis.
\newblock In \emph{ICML}, pages 1203--1212, 2017.
\newblock URL \url{http://proceedings.mlr.press/v70/garber17a.html}.

\bibitem[Grammenos et~al.(2020)Grammenos, Mendoza~Smith, Crowcroft, and Mascolo]{federatedpca}
Andreas Grammenos, Rodrigo Mendoza~Smith, Jon Crowcroft, and Cecilia Mascolo.
\newblock Federated principal component analysis.
\newblock In H.~Larochelle, M.~Ranzato, R.~Hadsell, M.F. Balcan, and H.~Lin, editors, \emph{Advances in Neural Information Processing Systems}, volume~33, pages 6453--6464. Curran Associates, Inc., 2020.
\newblock URL \url{https://proceedings.neurips.cc/paper/2020/file/47a658229eb2368a99f1d032c8848542-Paper.pdf}.

\bibitem[Greenbaum et~al.(2020)Greenbaum, Li, and Overton]{firstorder}
Anne Greenbaum, Ren-Cang Li, and Michael~L. Overton.
\newblock First-order perturbation theory for eigenvalues and eigenvectors.
\newblock \emph{SIAM Review}, 62\penalty0 (2):\penalty0 463--482, 2020.
\newblock \doi{10.1137/19M124784X}.

\bibitem[Hastie et~al.(2009)Hastie, Tibshirani, and Friedman]{esl}
Trevor Hastie, Robert Tibshirani, and Jerome Friedman.
\newblock \emph{The Elements of Statistical Learning}.
\newblock Springer Series in Statistics, 2009.

\bibitem[Hong et~al.(2021)Hong, Yang, Fessler, and Balzano]{hcaoptimalweight}
David Hong, Fan Yang, Jeffrey~A. Fessler, and Laura Balzano.
\newblock Optimally weighted pca for high-dimensional heteroscedastic data.
\newblock In \emph{Arxiv}, 2021.
\newblock URL \url{https://arxiv.org/pdf/1810.12862.pdf}.

\bibitem[Hotelling(1933)]{hotellingpca}
H.~Hotelling.
\newblock Analysis of a complex of statistical variables into principal components.
\newblock \emph{Journal of Educational Psychology}, 24:\penalty0 417--441, 1933.
\newblock \doi{http://dx.doi.org/10.1037/h0071325}.

\bibitem[Huang and Pan(2020)]{huangpcariemann}
Long-Kai Huang and Sinno Pan.
\newblock Communication-efficient distributed {PCA} by {R}iemannian optimization.
\newblock In Hal~Daumé III and Aarti Singh, editors, \emph{Proceedings of the 37th International Conference on Machine Learning}, volume 119 of \emph{Proceedings of Machine Learning Research}, pages 4465--4474. PMLR, 13--18 Jul 2020.
\newblock URL \url{https://proceedings.mlr.press/v119/huang20e.html}.

\bibitem[J{\'e}gou and Chum(2012)]{pcaimage2}
Herv{\'e} J{\'e}gou and Ond{\v{r}}ej Chum.
\newblock Negative evidences and co-occurences in image retrieval: The benefit of pca and whitening.
\newblock In Andrew Fitzgibbon, Svetlana Lazebnik, Pietro Perona, Yoichi Sato, and Cordelia Schmid, editors, \emph{Computer Vision -- ECCV 2012}, pages 774--787, Berlin, Heidelberg, 2012. Springer Berlin Heidelberg.
\newblock ISBN 978-3-642-33709-3.

\bibitem[Kahan(2011)]{berkeleyproof}
William Kahan.
\newblock The nearest orthogonal or unitary matrix.
\newblock \emph{W. Kahan's Supplementary Notes for Math. 128}, 2011.

\bibitem[Kontar et~al.(2017)Kontar, Zhou, Sankavaram, Du, and Zhang]{kontar2017nonparametric}
Raed Kontar, Shiyu Zhou, Chaitanya Sankavaram, Xinyu Du, and Yilu Zhang.
\newblock Nonparametric-condition-based remaining useful life prediction incorporating external factors.
\newblock \emph{IEEE Transactions on Reliability}, 67\penalty0 (1):\penalty0 41--52, 2017.

\bibitem[Kontar et~al.(2018)Kontar, Zhou, Sankavaram, Du, and Zhang]{kontar2018nonparametric}
Raed Kontar, Shiyu Zhou, Chaitanya Sankavaram, Xinyu Du, and Yilu Zhang.
\newblock Nonparametric modeling and prognosis of condition monitoring signals using multivariate gaussian convolution processes.
\newblock \emph{Technometrics}, 60\penalty0 (4):\penalty0 484--496, 2018.

\bibitem[Kontar et~al.(2021)Kontar, Shi, Yue, Chung, Byon, Chowdhury, Jin, Kontar, Masoud, Nouiehed, et~al.]{ioft}
Raed Kontar, Naichen Shi, Xubo Yue, Seokhyun Chung, Eunshin Byon, Mosharaf Chowdhury, Jionghua Jin, Wissam Kontar, Neda Masoud, Maher Nouiehed, et~al.
\newblock The internet of federated things (ioft).
\newblock \emph{IEEE Access}, 9:\penalty0 156071--156113, 2021.

\bibitem[Krizhevsky et~al.(2009)Krizhevsky, Hinton, et~al.]{cifar10}
Alex Krizhevsky, Geoffrey Hinton, et~al.
\newblock Learning multiple layers of features from tiny images.
\newblock 2009.

\bibitem[{Kulkarni} et~al.(2020){Kulkarni}, {Kulkarni}, and {Pant}]{surveyonpersonalization}
V.~{Kulkarni}, M.~{Kulkarni}, and A.~{Pant}.
\newblock Survey of personalization techniques for federated learning.
\newblock In \emph{2020 Fourth World Conference on Smart Trends in Systems, Security and Sustainability (WorldS4)}, pages 794--797, 2020.
\newblock \doi{10.1109/WorldS450073.2020.9210355}.

\bibitem[Li et~al.(2018{\natexlab{a}})Li, Sahu, Zaheer, Sanjabi, and Ameet~Talwalkar]{fedprox}
Tian Li, Anit~Kumar Sahu, Manzil Zaheer, Maziar Sanjabi, and Virginia~Smith Ameet~Talwalkar.
\newblock Federated optimization in heterogeneous networks.
\newblock \emph{Proceedings of the 3rd MLSys Conference}, 2018{\natexlab{a}}.

\bibitem[Li et~al.(2018{\natexlab{b}})Li, Peng, and Wang]{pcacm2}
Wei Li, Minjun Peng, and Qingzhong Wang.
\newblock Fault detectability analysis in pca method during condition monitoring of sensors in a nuclear power plant.
\newblock \emph{Annals of Nuclear Energy}, 119:\penalty0 342--351, 2018{\natexlab{b}}.

\bibitem[Li et~al.(2020)Li, Huang, Yang, Wang, and Zhang]{convergencefedavg}
Xiang Li, Kaixuan Huang, Wenhao Yang, Shusen Wang, and Zhihua Zhang.
\newblock On the convergence of fedavg on non-iid data.
\newblock In \emph{International Conference on Learning Representations}, 2020.
\newblock URL \url{https://openreview.net/forum?id=HJxNAnVtDS}.

\bibitem[Liang et~al.(2014)Liang, Balcan, Kanchanapally, and Woodruff]{stacksingularvectorsvd}
Yingyu Liang, Maria-Florina~F Balcan, Vandana Kanchanapally, and David Woodruff.
\newblock Improved distributed principal component analysis.
\newblock In Z.~Ghahramani, M.~Welling, C.~Cortes, N.~Lawrence, and K.Q. Weinberger, editors, \emph{Advances in Neural Information Processing Systems}, volume~27. Curran Associates, Inc., 2014.
\newblock URL \url{https://proceedings.neurips.cc/paper/2014/file/52947e0ade57a09e4a1386d08f17b656-Paper.pdf}.

\bibitem[Liu et~al.(2019)Liu, So, and Wu]{qrisretraction}
Huikang Liu, Anthony Man-Cho So, and Weijie Wu.
\newblock Quadratic optimization with orthogonality constraint: Explicit lojasiewicz exponent and linear convergence of retraction-based line-search and stochastic variance-reduced gradient methods.
\newblock In \emph{Proceedings of the 33rd International Conference on Machine Learning}, 2019.

\bibitem[Lock et~al.(2013)Lock, Hoadley, Marron, and Nobel]{jive}
Eric~F Lock, Katherine~A Hoadley, James~Stephen Marron, and Andrew~B Nobel.
\newblock Joint and individual variation explained (jive) for integrated analysis of multiple data types.
\newblock \emph{The annals of applied statistics}, 7\penalty0 (1):\penalty0 523, 2013.

\bibitem[McMahan et~al.(2017)McMahan, Moore, Ramage, Hampson, and y~Arcas]{fedavg}
Brendan McMahan, Eider Moore, Daniel Ramage, Seth Hampson, and Blaise~Aguera y~Arcas.
\newblock Communication-efficient learning of deep networks from decentralized data.
\newblock In \emph{Artificial intelligence and statistics}, pages 1273--1282. PMLR, 2017.

\bibitem[Novembre and Stephens(2008)]{pcagenetics2}
John Novembre and Matthew Stephens.
\newblock Interpreting principal component analyses of spatial population genetic variation.
\newblock \emph{Nature genetics}, 40\penalty0 (5):\penalty0 646--649, 2008.

\bibitem[Oba et~al.(2007)Oba, Kawanabe, M\"{u}ller, and Ishii]{hca}
Shigeyuki Oba, Motoaki Kawanabe, Klaus-Robert M\"{u}ller, and Shin Ishii.
\newblock Heterogeneous component analysis.
\newblock In J.~Platt, D.~Koller, Y.~Singer, and S.~Roweis, editors, \emph{Advances in Neural Information Processing Systems}, volume~20. Curran Associates, Inc., 2007.
\newblock URL \url{https://proceedings.neurips.cc/paper/2007/file/a8abb4bb284b5b27aa7cb790dc20f80b-Paper.pdf}.

\bibitem[Pozo et~al.(2018)Pozo, Vidal, and Salgado]{pcacm1}
Francesc Pozo, Yolanda Vidal, and {\'O}scar Salgado.
\newblock Wind turbine condition monitoring strategy through multiway pca and multivariate inference.
\newblock \emph{Energies}, 11\penalty0 (4):\penalty0 749, 2018.

\bibitem[Qu et~al.(2002)Qu, Ostrouchov, Samatova, and Geist]{dispca2}
Yongming Qu, George Ostrouchov, Nagiza Samatova, and Al~Geist.
\newblock Principal component analysis for dimension reduction in massive distributed data sets.
\newblock In \emph{Knowledge and Information Systems - KAIS}, 04 2002.

\bibitem[Reich et~al.(2008)Reich, Price, and Patterson]{pcagenetics1}
David Reich, Alkes~L Price, and Nick Patterson.
\newblock Principal component analysis of genetic data.
\newblock \emph{Nature genetics}, 40\penalty0 (5):\penalty0 491--492, 2008.

\bibitem[Rinaldo(2019)]{covarianceconcentrationopnorm}
Alessandro Rinaldo.
\newblock Lecture notes in advanced statistical theory, Fall 2019.

\bibitem[Sattler et~al.(2019)Sattler, Müller, and Samek]{clusterfl}
Felix Sattler, Klaus-Robert Müller, and Wojciech Samek.
\newblock Clustered federated learning: Model-agnostic distributed multi-task optimization under privacy constraints.
\newblock \emph{arXiv preprint arXiv:1910.01991}, 2019.

\bibitem[Sun and Luo(2015)]{ruoyufactorization}
Ruoyu Sun and Ziquan Luo.
\newblock Guaranteed matrix completion via non-convex factorization.
\newblock \emph{FOCS}, 2015.

\bibitem[Tang(2019)]{proofonkpca}
Cheng Tang.
\newblock Exponentially convergent stochastic k-pca without variance reduction.
\newblock In \emph{Advances in Neural Information Processing Systems}, volume~32. Curran Associates, Inc., 2019.

\bibitem[Vacavant et~al.(2012)Vacavant, Chateau, Wilhelm, and Lequievre]{vacavant}
Antoine Vacavant, Thierry Chateau, Alexis Wilhelm, and Laurent Lequievre.
\newblock A benchmark dataset for outdoor foreground/background extraction.
\newblock In \emph{ACCV Workshops}, 2012.
\newblock URL \url{https://api.semanticscholar.org/CorpusID:10634625}.

\bibitem[Vu et~al.(2013)Vu, Cho, Lei, and Rohe]{fantope}
Vincent~Q Vu, Juhee Cho, Jing Lei, and Karl Rohe.
\newblock Fantope projection and selection: A near-optimal convex relaxation of sparse pca.
\newblock In \emph{Advances in Neural Information Processing Systems}, volume~26. Curran Associates, Inc., 2013.
\newblock URL \url{https://proceedings.neurips.cc/paper/2013/file/81e5f81db77c596492e6f1a5a792ed53-Paper.pdf}.

\bibitem[Wainwright(2019)]{wainwrightbook}
Martin~J. Wainwright.
\newblock \emph{High-Dimensional Statistics: A Non-Asymptotic Viewpoint}.
\newblock Cambridge University Press, 2019.
\newblock \doi{10.1017/9781108627771}.

\bibitem[Xu et~al.(2012)Xu, Caramanis, and Sanghavi]{robustpcaoutlier}
Huan Xu, Constantine Caramanis, and Sujay Sanghavi.
\newblock Robust pca via outlier pursuit.
\newblock \emph{IEEE Transactions on Information Theory}, 58\penalty0 (5):\penalty0 3047--3064, 2012.
\newblock \doi{10.1109/TIT.2011.2173156}.

\bibitem[Yang and Shahabi(2004)]{pcatimeseries1}
Kiyoung Yang and Cyrus Shahabi.
\newblock A pca-based similarity measure for multivariate time series.
\newblock In \emph{Proceedings of the 2nd ACM international workshop on Multimedia databases}, pages 65--74, 2004.

\bibitem[Zhou et~al.(2015)Zhou, Cichocki, Zhang, and Mandic]{cife}
Guoxu Zhou, Andrzej Cichocki, Yu~Zhang, and Danilo~P Mandic.
\newblock Group component analysis for multiblock data: Common and individual feature extraction.
\newblock \emph{IEEE transactions on neural networks and learning systems}, 27\penalty0 (11):\penalty0 2426--2439, 2015.

\bibitem[Zhuang et~al.(2020)Zhuang, Qi, Duan, Xi, Zhu, Zhu, Xiong, and He]{transferlearning}
Fuzhen Zhuang, Zhiyuan Qi, Keyu Duan, Dongbo Xi, Yongchun Zhu, Hengshu Zhu, Hui Xiong, and Qing He.
\newblock A comprehensive survey on transfer learning.
\newblock \emph{Proceedings of the IEEE}, 109\penalty0 (1):\penalty0 43--76, 2020.

\bibitem[Zhuo et~al.(2021)Zhuo, Kwon, Ho, and Caramanis]{overparametrize}
Jiacheng Zhuo, Jeongyeol Kwon, Nhat Ho, and Constantine Caramanis.
\newblock On the computational and statistical complexity of over-parameterized matrix sensing.
\newblock \emph{arXiv preprint arXiv:2102.02756}, 2021.

\bibitem[Zou et~al.(2006)Zou, Hastie, and Tibshirani]{sparsepca}
Hui Zou, Trevor Hastie, and Robert Tibshirani.
\newblock Sparse principal component analysis.
\newblock \emph{Journal of Computational and Graphical Statistics}, 15\penalty0 (2):\penalty0 265--286, 2006.
\newblock \doi{10.1198/106186006X113430}.

\end{thebibliography}
\end{document}